\documentclass[11pt,letterpaper]{article}

\usepackage[T1]{fontenc}
\usepackage[margin=1in,headheight=14pt]{geometry}
\usepackage{amsmath,amssymb,amsthm,mathtools}
\usepackage{newtxtext,newtxmath}
\usepackage{bm,microtype}
\usepackage{graphicx,booktabs,tabularx,array,multirow,multicol}
\usepackage{caption,subcaption}
\usepackage{algorithm,algpseudocode}
\usepackage[shortlabels]{enumitem}
\usepackage{dsfont,bbm,tikz,empheq,mdframed,tcolorbox,soul}
\usepackage[round,authoryear]{natbib}
\usepackage{xcolor,xurl,placeins,fancyhdr}
\usepackage{hyperref}

\definecolor{linkblue}{RGB}{0,51,153}
\hypersetup{
  pdftitle={Reinforcement Learning with Multi-Step Transition Predictions: A Bellman Theory},
  pdfauthor={Chenbei Lu, Zaiwei Chen, Tongxin Li, Chenye Wu, and Adam Wierman},
  colorlinks=true,
  linkcolor=linkblue,
  citecolor=linkblue,
  urlcolor=linkblue,
  bookmarksnumbered=true
}
\providecommand{\doi}[1]{doi: \begingroup\urlstyle{rm}\href{https://doi.org/#1}{\nolinkurl{#1}}\endgroup}
\allowdisplaybreaks[1]
\setlist{noitemsep,leftmargin=*}
\fancypagestyle{plain}{
  \fancyhf{}
  \fancyfoot[C]{\small\thepage}
  
}

\theoremstyle{plain}
\newtheorem{theorem}{Theorem}[section]
\newtheorem{lemma}{Lemma}[section]
\newtheorem{proposition}{Proposition}[section]
\newtheorem{corollary}{Corollary}[section]
\theoremstyle{definition}
\newtheorem{definition}{Definition}[section]
\newtheorem{assumption}{Assumption}[section]
\theoremstyle{remark}

\newcommand{\FH}{\mathrm{FH}}
\newcommand{\BF}{\mathrm{BF}}

\newcommand{\Bayes}{\mathcal{B}}
\newcommand{\E}{\mathbb{E}}

\title{\LARGE\bfseries Reinforcement Learning with Multi-Step Transition Predictions: A Bellman Theory}

\author{
  \normalsize Chenbei Lu\textsuperscript{1} \quad Zaiwei Chen\textsuperscript{2} \quad Tongxin Li\textsuperscript{3} \quad Chenye Wu\textsuperscript{4} \quad Adam Wierman\textsuperscript{5}\\[0.5em]
  \small\textsuperscript{1}Duffield College of Engineering, Cornell University\\
  \small\textsuperscript{2}Edwardson School of Industrial Engineering, Purdue University\\
  \small\textsuperscript{3}School of Data Science, The Chinese University of Hong Kong (Shenzhen)\\
  \small\textsuperscript{4}School of Science and Engineering, The Chinese University of Hong Kong (Shenzhen)\\
  \small\textsuperscript{5}Computing \& Mathematical Sciences, Caltech
}
\date{}

\begin{document}
\maketitle

\begin{abstract}
Multi-step predictions of future transition realizations can improve reinforcement learning, but their value remains difficult to quantify.
We study discounted finite MDPs under a \emph{block-feedback} protocol: at the start of each $K$-period block, the decision maker observes $K$-step probabilistic transition predictions for a subset of actions and adapts decisions to observed states within the block. We show that the optimal value satisfies a Bayesian Bellman equation on the original state space, without including the prediction in the state. We illustrate the Bellman--Jensen effect with exact one-step predictions, where observing the prediction moves maximization inside the expectation. For multi-step predictions, we quantify how horizon, prediction error, action coverage, and within-block feedback affect value. We propose block-feedback Bayesian offline learning with online adaptation (BF--BOLA), a model-based algorithm for learning near-optimal policies, and prove sample-complexity upper bounds with matching minimax lower bounds for sufficiently small target errors. Our results show that transition estimation is needed only for state-action pairs not covered by predictions and that learning a near-optimal policy can require substantially fewer prediction samples than estimating the Bayesian value function.
\end{abstract}


\section{Introduction}
Predictions of future system conditions are increasingly available in online decision making problems \citep{lim2021time}. For example, energy storage operators use renewable-generation and electricity-price predictions \citep{ahmed2019review}, portfolio managers rely on signals about future market conditions \citep{charpentier2021reinforcement}, and mobility and logistics systems use traffic, demand, and supply predictions \citep{haydari2020deep}. These predictions provide advance information about future uncertainty and can help decision makers adapt actions before the corresponding uncertainty is realized.

Markov decision processes (MDPs) model these problems through a one-step transition kernel, which gives the distribution of the next state given the current state and action. When the model is unknown, reinforcement learning (RL) uses data to learn a policy \citep{sutton2018reinforcement,puterman2014markov}. Multi-step transition predictions can change which action is best: an action that looks less attractive under the base transition model may have a higher expected return once the prediction is observed.

However, incorporating such predictions into decision making is nontrivial. A direct approach is to append the current \(K\)-step prediction to the state and solve the resulting augmented MDP. For probabilistic predictions, this can produce a high-dimensional, continuous state space. The number of discretized prediction states can grow exponentially with \(K\), which makes both policy representation and dynamic programming intractable. Moreover, optimal planning with multi-step transition lookahead can be NP-hard even when future transitions are revealed exactly \citep{pla2026hardness}. State-space growth may also increase the statistical burden of learning, since the sample complexity of tabular RL scales with the number of state-action pairs \citep{azar2012sample,gheshlaghi2013minimax,agarwal2020model,li2023breaking}. Beyond these computational and statistical challenges, it is also unclear how much predictive information actually improves decision making.


In this work, we study discounted finite MDPs with transition predictions under a decision protocol we call \emph{block feedback}. Time is divided into blocks of \(K\) periods. At the beginning of each block, the decision maker observes probabilistic predictions for a subset of transition rows over the next \(K\) periods. These predictions remain available throughout the block, while the current state is observed before each action is selected. Thus, actions can adapt to the realized state trajectory within the block while using the same multi-step prediction information. Within this framework, we study how multi-step transition predictions reshape the Bellman recursion and improve decision value, how to compute prediction-aware decisions efficiently, and how many samples are needed to learn a near-optimal policy.

This paper extends the conference version \citep{lu2025reinforcement} from open-loop block planning to block feedback. We sharpen the finite-horizon upper bound, prove matching lower bounds for the horizon and prediction-error terms, and establish minimax guarantees for policy learning and Bayesian value estimation.

\subsection{Main Contributions}
\label{subsec:main-contributions}

The main contributions of this work are summarized as follows.

\subsubsection{Low-Dimensional Bellman Representation}
We extend the Bayesian Bellman formulation to block-feedback policies, which adapt actions to states realized within a block. The optimal infinite-horizon value can still be represented using only \(|\mathcal S|\) state coordinates.
Given the observed prediction, the decision maker solves an exact \(K\)-stage dynamic program that accounts for state feedback. Averaging the resulting optimized value over the prediction distribution defines a Bayesian Bellman operator on the original state space. We prove that this operator is a \(\gamma^K\)-contraction under the infinity norm and that its unique fixed point is exactly the optimal block-feedback value. The same recursion yields an optimal prediction-dependent state-feedback policy through backward induction.

\subsubsection{Value of Predictive Information}
We derive sharp bounds on the value of multi-step transition predictions relative to a full-information offline oracle. The Bellman--Jensen mechanism explains how predictions improve decisions by changing the local order of maximization and expectation. We also establish a value ordering across no prediction, imperfect prediction, exact prediction with partial or full action coverage, and the full-information benchmark, all on a common underlying transition process.

We further decompose the gap to the offline oracle into three sources: finite prediction horizon, prediction error, and partial action coverage. Under exact, full-coverage predictions, the finite-horizon loss decays geometrically as \(O(\gamma^K/(1-\gamma))\). Under a uniform phase-wise prediction-error budget \(\epsilon\), the loss from imperfect predictions is \(O(\epsilon/(1-\gamma)^2)\) and does not accumulate linearly with \(K\). Matching worst-case lower bounds show that neither the geometric horizon rate nor the \((1-\gamma)^{-2}\) error sensitivity can be improved in general. The loss from partial coverage is bounded by a model-dependent one-block coverage gap amplified by \((1-\gamma^K)^{-1}\). We also quantify the value of within-block state feedback relative to open-loop block planning in Appendix~\ref{app:fh-value-info}.

\subsubsection{Minimax-Optimal Policy Learning}
We establish minimax sample-complexity guarantees for policy learning and Bayesian value estimation under block feedback. We propose block-feedback Bayesian offline learning with online adaptation (BF--BOLA), a model-based algorithm that estimates transition rows for actions outside the covered set \(\mathcal A^-\) from environment samples and learns the distribution of \(K\)-step predictions from independent prediction samples. Offline, BF--BOLA computes an empirical Bayesian value. Online, it combines this value with the currently observed predictions and adapts actions to the realized states. Suppressing logarithmic factors, we show that an \(\varepsilon_{\rm pol}\)-optimal block-feedback policy can be learned with
\(D_1
=
\widetilde O\!\left(
\frac{|\mathcal S|(|\mathcal A|-|\mathcal A^-|)}
{(1-\gamma)^3\varepsilon_{\rm pol}^2}
\right)\)
environment transition samples (plus one reward observation per state-action pair), and
\(D_2
=
\widetilde O\!\left(
1+
\frac{\gamma^{2K}}
{(1-\gamma)^2(1-\gamma^K)\varepsilon_{\rm pol}^2}
\right)\)
prediction samples.

We prove matching lower bounds for the environment-sample rate and for the nonconstant prediction-sample term when prediction uncertainty affects policy performance. These bounds establish minimax optimality up to logarithmic factors. The environment-sample rate shows that prediction coverage reduces the learning dimension from all state-action transition rows to only the uncovered rows. At equal target accuracies, the leading prediction-sample term for policy learning is \(\gamma^{2K}\) times that for Bayesian value estimation. Estimating the Bayesian value to accuracy \(\varepsilon_{\rm est}\) requires $\widetilde\Theta\!\left(
\frac{1}
{(1-\gamma)^2(1-\gamma^K)\varepsilon_{\rm est}^2}
\right)$ prediction samples.
Compared with the conference version \citep{lu2025reinforcement}, our Bayesian value-estimation bounds improve the environment-sample dependence on \((1-\gamma)^{-1}\) from quartic to cubic and the prediction-sample dependence on \((1-\gamma^K)^{-1}\) from quadratic to linear.

The policy bound uses a current-block cancellation. The learned and benchmark policies condition on the same current prediction during execution, so prediction-law estimation affects policy performance only through future-block Bayesian values. One-sided Bellman comparisons and variance-sensitive concentration then yield the \(\gamma^{2K}\) improvement in prediction-sample complexity.

\subsection{Related Literature}
\label{sec:related-literature}

\subsubsection{Model-Based RL}
\label{subsec:related-model-based}

Model-based reinforcement learning estimates the transition model and plans in the estimated MDP. Early work studied near-optimal planning under generative-model access \citep{kearns2002sparse}. {For discounted finite MDPs, \citet{azar2012sample} established minimax-optimal bounds for estimating \(Q^*\). \citet{gheshlaghi2013minimax} also obtained policy-learning guarantees for sufficiently small \(\varepsilon\).} Subsequent work developed variance-reduced policy-learning algorithms \citep{sidford2018near} and established the minimax optimality of planning in an empirical MDP \citep{agarwal2020model}. \citet{li2020breaking,li2023breaking} overcame the sample-size barrier and extended minimax policy guarantees to the full meaningful accuracy range. Up to logarithmic factors, the minimax rate is
\(\widetilde{\Theta}\bigl(|\mathcal S||\mathcal A|/((1-\gamma)^3\varepsilon^2)\bigr)\),
which characterizes the sample complexity of learning a near-optimal policy. Related work studies average-reward learning \citep{jin2021towards}. For episodic problems, efficient exploration has been studied in tabular MDPs \citep{azar2017minimax} and in linear MDPs \citep{jin2020provably}.

Our learning problem retains this model-based approach but changes what must be estimated. Uncovered transition rows are estimated from environment samples, while covered rows are supplied by predictions observed during operation. Accordingly, the environment-sample complexity scales with the number of uncovered rows,
\(\widetilde{\Theta}\bigl(|\mathcal S|(|\mathcal A|-|\mathcal A^-|)/((1-\gamma)^3\varepsilon^2)\bigr)\),
where \(\mathcal A^-\) is the set of actions covered by predictions. The learner also estimates the distribution of predictions from independent prediction samples.

Learning-augmented RL uses \(Q\)-value predictions to guide decisions \citep{golowich2022can,li2023beyond}. Other work studies how transition-matrix predictions affect sample complexity \citep{lyu2025efficiently}. Bayes-adaptive methods represent uncertainty about the model through a posterior \citep{duff2002optimal,ross2007bayes,guez2012efficient}. Our predictions instead provide additional information about upcoming transition realizations.

\subsubsection{Decision Making with Lookahead Information}
\label{subsec:related-lookahead}

Lookahead is often used as a planning tool. Multi-step greedy methods use deeper planning to improve policies \citep{efroni2018beyond,efroni2020online,winnicki2024role}. Adaptive lookahead adjusts planning depth to the current state \citep{rosenberg2023planning}, while lookahead-bounded \(Q\)-learning constructs value bounds from sampled information-relaxation problems \citep{el2020lookahead}. These methods use lookahead for computation, without observing future transition realizations during execution.

Model predictive control (MPC) uses forecasts by repeatedly solving a finite-horizon control problem \citep{garcia1989model}. In online optimization with switching costs, predictions of future cost functions guide decisions \citep{chen2015online,li2020online,li2020leveraging}, with competitive guarantees even for nonconvex costs \citep{lin2020online}. Regret analyses for linear-quadratic control study forecasts of disturbances \citep{yu2020power} and time-varying costs \citep{zhang2021regret}, while perturbation analyses address linear time-varying systems \citep{lin2021perturbation} and MPC with prediction errors and constraints \citep{lin2022bounded}. Work on nonstationary MDPs instead uses predictions of future transition kernels and rewards \citep{zhang2025predictive}.

More directly related are MDP models in which future information is revealed before action selection. Prior work studies the value of reward lookahead \citep{merlis2024reward} and develops Bellman equations and regret guarantees for one-step reward and transition lookahead \citep{merlis2024lookahead}. Extending transition lookahead to multiple steps raises a computational challenge. Rolling predictions enlarge the information state, and exact planning can be NP-hard even with perfect predictions \citep{pla2026hardness}. \citet{merlis2026adaptive} develops state-feedback policies with state-dependent block lengths under exact lookahead, together with block-boundary Bellman equations and online regret guarantees. \citet{pla2026nearoptimal} show that exact planning with rolling lookahead remains NP-hard for any fixed rational discount factor. They also develop near-optimal planning and regret guarantees, with extensions to noisy predictions. Their planner runs in polynomial time for fixed lookahead depth, while its computational and storage costs grow exponentially with that depth. We study probabilistic transition predictions under block feedback, where predictions are refreshed only at block boundaries and may leave residual transition uncertainty. The Bayesian value remains defined on the original state space. For any prediction horizon \(K\), within-block planning requires \(O(K|\mathcal S|^2|\mathcal A|)\) operations per observed prediction, polynomial in both \(K\) and the MDP size. We also establish minimax sample-complexity guarantees for policy learning and Bayesian value estimation.

Exogenous information provides another connection. Exogenous-information models study learning with irrelevant, action-independent observations \citep{efroni2022exogenous}, whereas exogenous-input MDPs use external inputs to drive the dynamics \citep{sinclair2023hindsight}. Discounted MDPs with i.i.d. contexts observed before action selection admit an averaged Bellman equation on the original state space and minimax learning bounds \citep{pla2026exogenous}. At block boundaries, our model has this structure, with the prediction as context and a within-block feedback rule as the action. When the one-step transition kernel and rewards are known and predictions have finite support, rescaling their context-sample upper bounds yields the same leading prediction-sample rates for sufficiently small target errors. The bounds do not depend on the number of within-block feedback rules, and backward induction evaluates the required maximization without enumerating these rules. With unknown dynamics, \citet{pla2026exogenous} allow additional transition queries at decision time. In our model, the uncovered transition rows are shared across predictions and phases, so BF--BOLA estimates them offline and uses the observed prediction without further transition queries. Our analysis allows a continuous prediction space without discretization, and we prove matching lower bounds within the calibrated prediction model.

\subsubsection{Information in Stochastic Decision Problems}
\label{subsec:related-information}

Our value analysis builds on the study of information in stochastic decision problems. Blackwell's comparison of statistical experiments orders information structures \citep{blackwell1953equivalent}. The stochastic-programming literature quantifies the value of observing uncertainty before decisions are made \citep{avriel1970value,huang1977sharp}. Information-relaxation methods obtain performance bounds by relaxing nonanticipativity constraints and adding penalties for anticipative decisions \citep{brown2010information,brown2017information,balseiro2019approximations}. We instead quantify the value of transition predictions through their effect on the Bellman equation, particularly the ordering of expectation and maximization.

\section{Model and Decision Protocols}
\label{sec:model}
In this section, we first present the models of Markov decision processes (MDPs) and the transition predictions. Then, we define the decision protocol of block feedback that we study, and introduce two reference protocols.

\subsection{Preliminaries on Markov Decision Processes}
\label{subsec:base-mdp}

We consider a standard discounted infinite-horizon MDP specified by the tuple  $\mathcal{M} = (\mathcal{S}, \mathcal{A}, P, r, \gamma)$ \citep{puterman2014markov}.
Here, $\mathcal{S}$ and $\mathcal{A}$ denote the finite state and action spaces, respectively. 
The transition kernel $P$ specifies the next-state distribution, with  $P(\cdot \mid s,a)$ denoting the distribution of the next state when action $a$ is taken in state $s$. 
The reward function \(r:\mathcal S\times\mathcal A\to[0,1]\) is deterministic and normalized to the unit interval, and \(\gamma\in(0,1)\) is the discount factor.
The reward normalization is standard and only fixes the scale of the value function. In the learning setting of Section~\ref{sec:learning}, the transition kernel and rewards are initially unknown.

\subsection{Transition Predictions}
\label{sec:prediction-model}

Time is divided into blocks of \(K\) periods. Block \(b\ge0\) consists of periods \(bK,\ldots,(b+1)K-1\). We use the same phase index \(j=0,\ldots,K-1\) for predictions and decisions, so phase \(j\) occurs at time \(bK+j\). We fix a set \(\mathcal A^-\subseteq\mathcal A\) of actions whose transition rows are covered by the prediction.

At the beginning of block \(b\), the decision maker observes \(\boldsymbol\Sigma_b=(\Sigma_{b,0},\ldots,\Sigma_{b,K-1})\). For a realized prediction \(\boldsymbol\sigma=(\sigma_0,\ldots,\sigma_{K-1})\), the row \(\sigma_j(\cdot\mid s,a)\in\Delta(\mathcal S)\) gives the predicted next-state distribution at phase \(j\), for each state \(s\) and covered action \(a\in\mathcal A^-\). Thus the prediction takes values in $\mathcal Q_K
:= \left(\Delta(\mathcal S)^{\mathcal S\times\mathcal A^-}\right)^K.$
We let \(\mu\) denote the probability law of the random prediction \(\boldsymbol\Sigma_b\) on \(\mathcal Q_K\).

Given a prediction \(\boldsymbol\sigma\), the conditional transition row at phase \(j\) is
\begin{align}
\bar P_j(s'\mid s,a,\boldsymbol\sigma)
=
\begin{cases}
\sigma_j(s'\mid s,a), & a\in\mathcal A^-,\\
P(s'\mid s,a), & a\notin\mathcal A^-.
\end{cases}
\label{eq:operational-transition}
\end{align}
Thus predictions provide additional information about upcoming transition outcomes for covered actions, while uncovered actions follow the base kernel \(P\). We impose the following assumptions on how predictions are generated and how they relate to the base MDP.

\begin{assumption}[Prediction process and consistency]
\label{assump:prediction-law}
\begin{enumerate}
\item \emph{Across-block sampling.}
At each block boundary \(\tau_b=bK\), a prediction \(\boldsymbol\Sigma_b\sim\mu\) is drawn independently of the past and independently across blocks. It is revealed at time \(\tau_b\) and is available for decisions at periods \(\tau_b,\ldots,\tau_b+K-1\).

\item \emph{Consistency.} For every covered action \(a\in\mathcal A^-\), the predicted rows average to the corresponding base transition row:
\[
\mathbb E_{\boldsymbol\Sigma_b\sim\mu}
\bigl[\Sigma_{b,j}(s'\mid s,a)\bigr]
=
P(s'\mid s,a),
\qquad
j=0,\ldots,K-1,\ s,s'\in\mathcal S.
\]

\item \emph{Fixed within-block prediction.} Conditional on the current prediction, earlier transitions observed within the same block do not change the transition rows used for later phases. Before choosing an action at phase \(j\), the next-state law for action \(a\) in state \(s\) is \(\bar P_j(\cdot\mid s,a,\boldsymbol\Sigma_b)\).

\item \emph{Base-MDP compatibility.} Any policy that ignores predictions has the same transition dynamics as in the base MDP. Formally, for every such policy \(\pi\), let \(\mathcal H_t^{\mathrm{phys}}:=\sigma(S_0,A_0,\ldots,A_{t-1},S_t)\). Then, for all \(t\ge0\) and \(s'\in\mathcal S\),
\[
\mathbb P^\pi(S_{t+1}=s'\mid\mathcal H_t^{\mathrm{phys}},A_t)
=P(s'\mid S_t,A_t)
\quad\text{almost surely}.
\]
\end{enumerate}
\end{assumption}

Even when predicted rows average to \(P\), past observations can carry information about the next state beyond what the current state and action reveal. The last condition rules this out for policies that ignore predictions.

\subsection{Decision Making under Block Feedback}
\label{subsec:bf-decision-model}

We now specify how the prediction is used for decisions. Let \(\tau_b:=bK\) be the beginning of block \(b\). At time \(\tau_b\), the decision maker observes the state \(S_{\tau_b}\) and the predictions \(\boldsymbol\Sigma_b\). The same prediction remains available throughout the block, and at each period, the decision maker observes the current state before choosing an action. Thus, predictions are updated only at block boundaries, while actions still use state feedback within the block. We refer to such decision rules as block-feedback policies.

A stationary Markov block-feedback policy consists of \(K\) phase-dependent decision rules,
\[
\pi^{\BF}
=
\left(\pi^{\BF}_0,\ldots,\pi^{\BF}_{K-1}\right),
\qquad
\pi^{\BF}_j:\mathcal S\times\mathcal Q_K\to\Delta(\mathcal A),
\]
where \(\pi_j^{\BF}\) is used at phase \(j\) of each block. {Let \(\Pi^{\BF}\) denote the class of all stationary Markov block-feedback policies.}
Under a policy \(\pi^{\BF}\in\Pi^{\BF}\), at phase \(j=0,\ldots,K-1\) of block \(b\), the action and next state satisfy
\begin{align*}
A_{\tau_b+j}
&\sim
\pi_j^{\BF}\!\left(
\cdot\mid
S_{\tau_b+j},
\boldsymbol\Sigma_b
\right),
\qquad
S_{\tau_b+j+1}
\sim
\bar P_j\!\left(
\cdot\mid
S_{\tau_b+j},
A_{\tau_b+j},
\boldsymbol\Sigma_b
\right).
\end{align*}
The phase-\(j\) action may depend on the entire prediction and the state realized at that phase, but it cannot depend on future states or future predictions.

\subsection{Reference Decision Protocols}
\label{subsec:reference-protocols}

We introduce two reference protocols, fixed-horizon planning and receding-horizon planning, to clarify what block feedback adds and what it avoids. Fixed-horizon planning uses the same blockwise prediction timing as block feedback but commits to an open-loop action sequence within each block \citep{lu2025reinforcement}. Receding-horizon planning instead uses a rolling prediction process: at every period, the current state and a newly updated \(K\)-step prediction are observed before the current action is selected. The three protocols are based on the same physical MDP and the same type of transition-prediction information, but they differ in how prediction information is timed and how actions may respond to realized states.

\paragraph{Fixed-horizon planning.} At the beginning of each block, the decision maker observes the current state and a \(K\)-step prediction and then selects an open-loop action sequence in \(\mathcal A^K\). A fixed-horizon policy is therefore a mapping $\pi^{\FH}:\mathcal S\times\mathcal Q_K\to\Delta(\mathcal A^K).$ Let \(\Pi^{\FH}\) denote this policy class.
Once selected, the action sequence is executed throughout the block and cannot depend on the intermediate states that are subsequently realized. Fixed-horizon planning thus uses the same predictions as block feedback but removes within-block state feedback. It admits a block-start value function on the original physical state space \(\mathcal S\), although its one-block optimization is over the exponentially large action space \(\mathcal A^K\).

\paragraph{Receding-horizon planning.} At every period \(t\), the decision maker observes the current state together with an updated \(K\)-step prediction and then selects the current action. Receding-horizon planning therefore uses more frequently updated prediction information than either fixed-horizon planning or block feedback. Because the prediction rolls forward after each transition, however, the current prediction generally becomes part of the Markov information state. Exact dynamic programming consequently requires a value function defined on an augmented space containing both the physical state and the rolling prediction.

For exact tabular implementations, the three protocols have different bottlenecks. Fixed-horizon planning requires optimization over open-loop action sequences and is NP-complete in the worst case, even with deterministic predictions and binary actions. For block feedback, a fixed observed prediction and terminal value can be handled by a \(K\)-stage dynamic program with dense tabular cost \(O(K|\mathcal S|^2|\mathcal A|)\). Receding-horizon planning avoids open-loop commitment but generally requires an augmented value representation. In a discretized setting where each single-phase prediction object has \(B\) possible values, a \(K\)-step rolling prediction has \(B^K\) possible prediction states, leading to exact-tabular storage of order \(|\mathcal S|B^K\). These distinctions are formalized in Appendix~\ref{app:reference-protocol-complexity} and summarized in Table~\ref{tab:complexity_comparison}.

\begin{table}[t]
\centering
\small
\renewcommand{\arraystretch}{1.15}
\setlength{\tabcolsep}{5pt}
\begin{tabularx}{\textwidth}{@{}l
>{\raggedright\arraybackslash}p{0.30\textwidth}
>{\raggedright\arraybackslash}p{0.22\textwidth}
>{\raggedright\arraybackslash}X@{}}
\toprule
Protocol & Information timing & Value representation & Planning task and cost \\
\midrule
FH
& One prediction per block; open-loop block actions
& Block-start value on \(\mathcal S\)
& One-block decision problem: NP-complete \\

BF
& One prediction per block; state-adaptive block actions
& Block-start value on \(\mathcal S\)
& \(O(K|\mathcal S|^2|\mathcal A|)\) per prediction, given the block-end value \\

RH
& Prediction refreshed every period
& Augmented value on \(\mathcal S\times\mathcal Q_K\)
& Exact rolling-horizon planning: NP-hard for \(K\ge2\) \citep{pla2026hardness} \\
\bottomrule
\end{tabularx}
\caption{Comparison of prediction protocols. The BF cost is for a single block, given its prediction and terminal Bayesian value.}
\label{tab:complexity_comparison}
\end{table}

\section{Prediction-Aware Bellman Theory}
\label{sec:prediction-aware-bellman}

A direct Markov formulation uses \((s,\boldsymbol\sigma,j)\) as the state, where \(\boldsymbol\sigma\in\mathcal Q_K\) is the current prediction and \(j\) is the within-block phase, adding \(K|\mathcal S||\mathcal A^-|(|\mathcal S|-1)\) continuous prediction coordinates. A discretization with \(M\) possible values per covered transition row has \(M^{K|\mathcal S||\mathcal A^-|}\) possible predictions. We instead formulate the Bellman equation at block boundaries. The new prediction is independent of the past, so the Bayesian value can be defined on the original state space.

\subsection{Bayesian Block-Feedback Value}
\label{subsec:bf-bayesian-value}

We evaluate a policy at the beginning of a block, before its prediction is observed. The policy can then use that prediction to choose actions throughout the block.

Throughout this section, we fix \(K\), \(\mathcal A^-\), and \(\mu\), and suppress these parameters in the value notation. For a block-feedback policy \(\pi\in\Pi^{\BF}\), define its Bayesian value by
\begin{align}
V^\pi(s)
:=
\mathbb E^{\pi,\mu}\!\left[
\sum_{t=0}^{\infty}\gamma^t r(S_t,A_t)
\,\middle|\,
S_0=s
\right],
\qquad s\in\mathcal S.
\label{def_Bayesian_bellman_bf}
\end{align}
The expectation is over predictions, state transitions, and policy randomization. The optimal Bayesian value is
\[
V^*(s)
:=
\sup\nolimits_{\pi\in\Pi^{\BF}}
V^\pi(s),
\qquad s\in\mathcal S.
\]
A policy is optimal if it attains this value for every initial state. We restore the prediction parameters when comparing information regimes in Section~\ref{sec:value-predictions}.

\subsection{Block-Feedback Bellman Equation}
\label{subsec:bf-bellman-equation}
We now construct the Bellman operator for the Bayesian block-boundary value. The key step is to separate the Bayesian value carried across blocks from the finite-horizon optimization inside the current block.

\begin{definition}[Block-feedback Bellman operator]
\label{def:bf-bellman-operator}
Let \(V:\mathcal S\to\mathbb R\) be a candidate value function for the next block boundary. For a realized prediction \(\boldsymbol\sigma=(\sigma_0,\ldots,\sigma_{K-1})\), define the within-block values \(J_j^{\BF}(s,\boldsymbol\sigma;V)\) by backward recursion:
\begin{align}
J^{\BF}_K(s,\boldsymbol\sigma;V)
&:=V(s),
\label{eq:bf-terminal}
\\
J^{\BF}_j(s,\boldsymbol\sigma;V)
&:=
\max_{a\in\mathcal A}
\left\{
r(s,a)
+
\gamma
\sum_{s'\in\mathcal S}
\bar P_j(s'\mid s,a,\boldsymbol\sigma)
J^{\BF}_{j+1}(s',\boldsymbol\sigma;V)
\right\},
\qquad
j=K-1,\ldots,0.
\label{eq:Jj-recursion}
\end{align}

The corresponding block-feedback Bayesian Bellman operator \(\mathcal T^{\BF}:\mathbb R^{|\mathcal S|}\to\mathbb R^{|\mathcal S|}\) is
\begin{align}
(\mathcal T^{\BF}V)(s)
:=
\mathbb E_{\boldsymbol\Sigma\sim\mu}
\left[
J^{\BF}_0(s,\boldsymbol\Sigma;V)
\right],
\qquad
s\in\mathcal S.
\label{eq:bf-operator}
\end{align}
\end{definition}

Once \(\boldsymbol\sigma\) is observed, Assumption~\ref{assump:prediction-law} gives the transition row \(\bar P_j\) for each state, action, and phase. Earlier observations do not change these rows. Ordinary finite-horizon dynamic programming therefore computes \(J_j^{\BF}\), optimizing over all decision rules that use the information available within the block. The operator \(\mathcal T^{\BF}\) first optimizes for each observed prediction and then averages over predictions.

\begin{theorem}[Bellman Optimality Equation for Block Feedback]
\label{thm:bf-bellman}
The optimal block-feedback Bayesian value function \(V^*\) is the unique fixed point of \(\mathcal T^{\BF}\):
\begin{align}
V^*(s)
=
\left(
\mathcal T^{\BF}
V^*
\right)(s),
\qquad
s\in\mathcal S.
\label{eq:bf-bellman}
\end{align}
Moreover, \(\mathcal T^{\BF}\) is a \(\gamma^K\)-contraction under the infinity norm.
\end{theorem}

See Appendix~\ref{app:bf-bellman-proof} for the proof.

The Bayesian value has only \(|\mathcal S|\) entries. For an observed prediction, the within-block optimization is a \(K\)-stage dynamic program.

\subsection{Optimal Block-Feedback Policies}
\label{subsec:bf-optimal-policy}

The fixed point in Theorem~\ref{thm:bf-bellman} gives the Bayesian value used at every block boundary. Once a prediction is observed, the decision maker can plug this Bayesian value into the within-block recursion and act greedily at each phase. Write \(J_j^*(s,\boldsymbol\sigma):=J_j^{\BF}(s,\boldsymbol\sigma;V^*)\) for the resulting within-block values.

\begin{corollary}[Optimal Block-Feedback Policy]
\label{coro:bf-opt-policy}
For every \(j=0,\ldots,K-1\), \(s\in\mathcal S\), and \(\boldsymbol\sigma\in\mathcal Q_K\), choose a maximizer according to a fixed rule:
\begin{align}
a_j^*(s,\boldsymbol\sigma)
\in
\arg\max\nolimits_{a\in\mathcal A}
\left\{
r(s,a)
+
\gamma
\sum_{s'\in\mathcal S}
\bar P_j(s'\mid s,a,\boldsymbol\sigma)
J^*_{j+1}(s',\boldsymbol\sigma)
\right\},
\label{eq:bf-greedy-policy}
\end{align}
The block-feedback policy that puts probability one on \(a_j^*(s,\boldsymbol\sigma)\) at phase \(j\) is optimal.
\end{corollary}

\section{The Value of Predictions}
\label{sec:value-predictions}
Building on the Bellman equation for the optimal policy, we now quantify the value created by transition predictions. We show that this value arises from a Bellman--Jensen gap induced by the recursive reordering of expectation and maximization in the Bellman expansion.

\subsection{Bellman--Jensen Effect and Value Ordering}
\label{subsec:bf-prediction-gaps}

{Let \(V^*_{\mathrm{MDP}}\) denote the optimal discounted value without predictions, and let \(\mu^0\) and \(\mu^\varepsilon\) denote the exact and imperfect prediction laws, respectively. When comparing values, we display the horizon, action coverage, and prediction law as subscripts.}

\paragraph{Bellman reordering.}
{
Consider \(K=1\) with exact predictions for all actions. Write \(\sigma^0_t\) for the exact one-step prediction observed at time \(t\). It specifies the next state for every state-action pair. The ordinary and prediction-aware Bellman equations have the recursive expansions
\begin{align}
V_{\mathrm{MDP}}^*(s_0)
&=
\max_{a_0}
\left\{
r(s_0,a_0)
+
\gamma
\mathbb E_{\sigma^0_0}
\left[
\max_{a_1}
\left\{
r(s_1,a_1)
+
\gamma\mathbb E_{\sigma^0_1}[\cdots]
\right\}
\right]
\right\},
\label{eq:mdp-nested-bellman}
\\
V^{\Bayes,\BF,*}_{1,\mathcal A,\mu^0}(s_0)
&=
\mathbb E_{\sigma^0_0}
\left[
\max_{a_0}
\left\{
r(s_0,a_0)
+
\gamma
\mathbb E_{\sigma^0_1}
\left[
\max_{a_1}
\left\{
r(s_1,a_1)
+
\gamma\cdots
\right\}
\right]
\right\}
\right].
\label{eq:prediction-nested-bellman}
\end{align}

Without predictions, each action is chosen before the transition information is observed. Exact one-step predictions reverse this order, allowing the action to depend on that information. Since
$\mathbb E\!\left[\max_a X_a\right]
\ge
\max_a\mathbb E[X_a]$, this reordering creates a local Jensen gain. The same reordering recurs at every step of the Bellman expansion, giving the \emph{Bellman--Jensen effect}.
}

\paragraph{Information ordering.} We next turn the local Bellman--Jensen intuition into a global comparison of information regimes. The comparison keeps the transition process fixed and varies only what information is revealed before actions are selected. Theorem~\ref{thm:value-ordering-standard-mdp} shows that adding prediction information cannot reduce value, exact prediction dominates imperfect prediction, full action coverage dominates partial coverage, and a full-information oracle provides an upper benchmark.

The full-information oracle observes all potential next states before the first decision, including the outcomes of actions it may never take. We draw the potential next states independently conditional on the prediction for each block. This construction preserves the trajectory law under block feedback, but the oracle value can depend on how the potential next states are jointly drawn. Let \(\Pi^{\mathrm{off}}\) denote the corresponding policy class, and define the full-information oracle value by
\begin{align}
V^{\Bayes,*}_{\mathrm{off}}(s)
:=
\sup_{\pi\in\Pi^{\mathrm{off}}}
\mathbb E^\pi
\left[
\sum_{t=0}^{\infty}
\gamma^t r(S_t,A_t)
\,\middle|\,
S_0=s
\right].
\label{eq:offline-value}
\end{align}
\begin{theorem}[Value Ordering across Information Regimes]
\label{thm:value-ordering-standard-mdp}
Under Assumption~\ref{assump:prediction-law}, {with exact and offline information defined on the common probability space in Appendix~\ref{app:prediction-model-proofs},} for every prediction horizon \(K\), predictable action set \(\mathcal A^-\), and state \(s\),
\[
V^*_{\mathrm{MDP}}(s)
=
V^{\Bayes,\BF,*}_{K,\varnothing}(s)
\le
V^{\Bayes,\BF,*}_{K,\mathcal A^-,\mu^\varepsilon}(s)
\le
V^{\Bayes,\BF,*}_{K,\mathcal A^-,\mu^0}(s)
\le
V^{\Bayes,\BF,*}_{K,\mathcal A,\mu^0}(s)
\le
V^{\Bayes,*}_{\mathrm{off}}(s).
\]
\end{theorem}

See Appendix~\ref{app:value-ordering-proof} for the proof.

The difference between the first and last values in Theorem~\ref{thm:value-ordering-standard-mdp} measures the value of full future information. If \(|\mathcal A|=1\), prediction cannot change any decision and all values in the theorem coincide. Likewise, if the imperfect prediction row equals \(P(\cdot\mid s,a)\) almost surely for every covered state-action pair, the signal reveals no additional information and its block-feedback value equals \(V^*_{\mathrm{MDP}}\).

\subsection{Quantitative Value-of-Information Bounds}
\label{subsec:value-bound}

We bound the gap between imperfect, partially covered block feedback and the full-information benchmark. The comparison passes through two intermediate regimes: exact prediction with full action coverage and exact prediction with the original partial coverage. This separates the losses caused by a finite prediction horizon, imperfect predictions, and incomplete action coverage.

To measure the uncertainty left by an imperfect prediction, we compare it with exact prediction of the same transition outcomes. Let \(\boldsymbol\Sigma_0^0\) and \(\boldsymbol\Sigma_0^\varepsilon\) denote these exact and imperfect predictions.

\begin{definition}[Phase-wise prediction-error profile]
\label{def:prediction-error-profile}
For an imperfect calibrated prediction law \(\mu^\varepsilon\), define the phase-wise error profile \(\boldsymbol\varepsilon=(\epsilon_0,\ldots,\epsilon_{K-1})\) by
\[
\epsilon_j
:=
\mathbb E\!\left[
\max_{(s,a)\in\mathcal S\times\mathcal A^-}
\left\|
\Sigma^0_{0,j}(\cdot\mid s,a)
-
\Sigma^\varepsilon_{0,j}(\cdot\mid s,a)
\right\|_1
\right],
\qquad
j=0,\ldots,K-1.
\]
If \(\mathcal A^-=\varnothing\), set \(\epsilon_j=0\) for every \(j\).
\end{definition}

For the fixed prediction law \(\mu^\varepsilon\), we use the shorthand
\[
V^{\Bayes,\BF,*}_{K,\mathcal A^-,\boldsymbol\varepsilon}
:=
V^{\Bayes,\BF,*}_{K,\mathcal A^-,\mu^\varepsilon}.
\]

\begin{definition}[One-block coverage gap]
\label{def:partial-coverage-width}
Let \(\mathcal T^{\BF}_{K,\mathcal A,\boldsymbol 0}\) denote the exact block-feedback operator when all action rows are revealed, and let \(\mathcal T^{\BF}_{K,\mathcal A^-,\boldsymbol 0}\) denote the corresponding operator when only rows in \(\mathcal A^-\) are revealed. The two operators are constructed from the same exact next-state array. They differ only in whether the rows associated with actions in \(\mathcal A\setminus\mathcal A^-\) are observed or averaged through their base transition kernels. Define
\[
\omega_{K,\mathrm{cov}}
:=
\sup_{0\le V\le (1-\gamma)^{-1}\mathbf 1}
\left\|
\mathcal T^{\BF}_{K,\mathcal A,\boldsymbol 0}V
-
\mathcal T^{\BF}_{K,\mathcal A^-,\boldsymbol 0}V
\right\|_\infty.
\]
\end{definition}

{The coverage gap \(\omega_{K,\mathrm{cov}}\) is zero when \(\mathcal A^-=\mathcal A\). It is the largest one-block value gained by revealing exact transition information for uncovered transition rows. {The coverage loss depends on whether uncovered transitions can change the optimal action, not merely on their number.}}

\begin{theorem}[Value-of-Information Bound for Block Feedback]
\label{thm:bf-value-information-bound}
Consider a discounted finite MDP with rewards in \([0,1]\), discount factor \(\gamma\in(0,1)\), prediction horizon \(K\ge1\), and predictable action set \(\mathcal A^-\subseteq\mathcal A\). Suppose Assumption~\ref{assump:prediction-law} holds, and let \(\mu^\varepsilon\) be an imperfect calibrated prediction law with error profile \(\boldsymbol\varepsilon=(\epsilon_0,\ldots,\epsilon_{K-1})\). Then
\[
\left\|
V^{\Bayes,*}_{\mathrm{off}}
-
V^{\Bayes,\BF,*}_{K,\mathcal A^-,\boldsymbol\varepsilon}
\right\|_\infty
\le
\min
\left\{
\frac{1}{1-\gamma},
A_1+A_2+A_3
\right\},
\]
where
\begin{align}
A_1
&:=
\frac{\gamma^K}{1-\gamma},
&&\text{finite-horizon loss},
\label{eq:finite-horizon-term}
\\
A_2
&:=
\sum_{j=0}^{K-1}
\frac{\gamma^{j+1}}{(1-\gamma)(1-\gamma^K)}
\epsilon_j,
&&\text{prediction-error loss},
\label{eq:prediction-error-term}
\\
A_3
&:=
\frac{\omega_{K,\mathrm{cov}}}{1-\gamma^K},
&&\text{partial-coverage loss}.
\label{eq:partial-coverage-term}
\end{align}
\end{theorem}

See Appendix~\ref{app:voi-bound-proof} for the proof.

{The horizon loss is bounded by the discounted reward after the first \(K\) periods. Prediction error and incomplete coverage, in contrast, affect every block, so their one-block effects are amplified by \((1-\gamma^K)^{-1}\).} {Errors at earlier phases receive larger weights because their effects are discounted less.} In particular, if \(\epsilon_j\le\epsilon\) for every \(j\), then
$A_2
\le
\frac{\gamma\epsilon}{(1-\gamma)^2}$,
so prediction error does not accumulate linearly with the prediction horizon.

\subsection{Matching Lower Bounds}
\label{subsec:value-bound-tightness}

We next show that the horizon and prediction-error dependences in Theorem~\ref{thm:bf-value-information-bound} cannot be improved in general.

\begin{theorem}[Finite-Horizon Lower Bound]
\label{thm:finite-horizon-lower-bound}
There exists an absolute constant \(c>0\) such that, for every prediction horizon \(K\ge1\) and every \(\gamma\in[1/2,1)\), there exists a finite MDP with exact full-coverage predictions for which
\[
\left\|
V^{\Bayes,*}_{\mathrm{off}}
-
V^{\Bayes,\BF,*}_{K,\mathcal A,\boldsymbol 0}
\right\|_\infty
\ge
\frac{c\gamma^K}{1-\gamma}.
\]
\end{theorem}

Theorem~\ref{thm:finite-horizon-lower-bound} establishes \(\gamma^K/(1-\gamma)\) as the exact worst-case order of the finite-horizon loss. The hard instance delays the transition information that determines the better branch until immediately after the current \(K\)-step prediction. The offline benchmark observes this information before choosing the branch, whereas a block-feedback policy with only the current \(K\)-step prediction does not. The discounted value of this unavailable information gives the lower bound. The proof is provided in Appendix~\ref{proof_finite_window_lower_bound}.

\begin{theorem}[Prediction-Error Lower Bound under a Uniform Budget]
\label{thm:prediction-error-lower-bound}
There exists an absolute constant \(c>0\) such that, for every prediction horizon \(K\ge1\), every \(\gamma\in[1/2,1)\), and every \(\epsilon\in(0,2]\), there exists a finite MDP with \(\mathcal A^-=\mathcal A\) and a calibrated imperfect prediction law whose error profile satisfies \(\epsilon_j\le\epsilon\) for all \(0\le j<K\), and for which
\[
\left\|
V^{\Bayes,\BF,*}_{K,\mathcal A,\boldsymbol 0}
-
V^{\Bayes,\BF,*}_{K,\mathcal A,\boldsymbol\varepsilon}
\right\|_\infty
\ge
c
\min\!\left\{
\frac{\epsilon}{(1-\gamma)^2},
\frac{1}{1-\gamma}
\right\}.
\]
\end{theorem}

Under the uniform budget \(\epsilon_j\le\epsilon\), the prediction-error term in Theorem~\ref{thm:bf-value-information-bound} satisfies
\[
A_2
\le
\frac{\gamma\epsilon}{(1-\gamma)^2}.
\]
Because \(\gamma\ge1/2\), Theorem~\ref{thm:prediction-error-lower-bound} matches this dependence up to absolute constants until the gap reaches the natural value range \(1/(1-\gamma)\). In the hard instance, exact prediction reveals whether a risky transition will succeed, whereas the imperfect calibrated signal leaves this outcome unresolved. The reward of a safe action is chosen so that its value is nearly tied with that of the risky action under the imperfect signal. Resolving the remaining uncertainty can therefore change the optimal action and produce a value gain of order \(\epsilon/(1-\gamma)^2\), with saturation at the total discounted reward range. The proof is provided in Appendix~\ref{proof_prediction_error_lower_bound}.

\section{Sample-Efficient Learning under Block Feedback}
\label{sec:learning}

We now develop a sample-efficient algorithm for learning near-optimal block-feedback policies. Block-feedback Bayesian offline learning with online adaptation (BF--BOLA) learns a Bayesian value function offline from transition and prediction samples, and then adapts online to each prediction through the block-feedback recursion. We provide minimax-optimal sample-complexity guarantees for the resulting policy.

We write \(V^\pi\) and \(V^*\) for the Bayesian values of a block-feedback policy and an optimal block-feedback policy. We omit the prediction parameters when they are not being compared.

\subsection{Sampling Model}
\label{subsec:statistical-resources}

We use two independent data sources. One supplies samples from uncovered transition rows. The other supplies samples from the prediction law \(\mu\).
\begin{itemize}
\item \emph{Generative-model samples.} A generative model can be queried at any state-action pair to observe its reward and draw an independent next state \citep{kearns2002sparse,kakade2003sample}. BF--BOLA observes the deterministic reward once per pair and uses next-state samples only for uncovered rows. Let $\mathcal U :=\mathcal S\times(\mathcal A\setminus\mathcal A^-)$ denote these rows. For every \((s,a)\in\mathcal U\), the learner draws \(N_1\) independent samples \(Y_i(s,a)\sim P(\cdot\mid s,a)\), \(i=1,\ldots,N_1\), independently across state-action pairs.
\item \emph{Prediction samples.} A prediction oracle returns independent draws from the same law \(\mu\) that generates the prediction observed online at a block boundary. The learner receives \(N_2\) samples \(\boldsymbol\sigma^{\,i}=(\sigma_0^i,\ldots,\sigma_{K-1}^i)\sim\mu\), \(i=1,\ldots,N_2\), with \(\boldsymbol\sigma^{\,i}\in\mathcal Q_K\). Each sample contains the complete \(K\)-phase prediction table for all covered state-action rows.
\end{itemize}
Each transition sample is one next-state observation, whereas each prediction sample contains a complete \(K\)-step table with \(K|\mathcal S||\mathcal A^-|(|\mathcal S|-1)\) free probability parameters. Accordingly, \(N_1\) counts observations per uncovered transition row and \(N_2\) counts complete predictions, with an additive \(|\mathcal S||\mathcal A|\) cost for reward observations. Proposition~\ref{prop:bfbola-complexity} and Appendix~\ref{proof_prop_bfbola_complexity} describe the computational and storage costs. In the wind-storage model of Section~\ref{subsec:wind-storage-case-study}, the calibrated signal model and the battery dynamics determine all covered transition rows from the \(K\) exogenous prediction signals. The decision maker therefore does not need a separate forecast for each state-action pair.

\subsection{BF--BOLA: Offline Bellman Learning and Online Policy Adaptation}
\label{subsec:bf-bola}

BF--BOLA turns the two data sources above into an empirical Bayesian Bellman operator, computes the empirical Bayesian value offline, and then adapts online to the prediction observed in each block.

\subsubsection{Model Estimation}

We estimate \(P\) and \(\mu\) from their respective samples:
\begin{align}
\widehat P(s'\mid s,a)
&:=
\frac{1}{N_1}
\sum_{i=1}^{N_1}
\mathbf 1\!\left\{Y_i(s,a)=s'\right\},
\qquad
(s,a)\in\mathcal U,
\label{est_psd}
\\
\widehat\mu_{N_2}
&:=
\frac{1}{N_2}
\sum_{i=1}^{N_2}
\delta_{\boldsymbol\sigma^{\,i}}.
\label{est_psigm}
\end{align}
The estimator \(\widehat P\) is used only for uncovered actions. The empirical law \(\widehat\mu_{N_2}\) approximates the block-boundary expectation over future predictions. For a realized prediction \(\boldsymbol\sigma\), define the empirical conditional transition row by
\[
\widehat{\bar P}_j(s'\mid s,a,\boldsymbol\sigma)
:=
\begin{cases}
\sigma_j(s'\mid s,a), & a\in\mathcal A^-,\\
\widehat P(s'\mid s,a), & a\notin\mathcal A^-.
\end{cases}
\]
Thus, covered actions use the prediction rows observed in the current prediction, while uncovered actions use the transition rows estimated from environment samples.

\subsubsection{Offline Bayesian Value Estimation}

For a terminal value \(V:\mathcal S\to\mathbb R\), define the empirical within-block recursion by
\begin{align}
\widehat J_K^{\BF}(s,\boldsymbol\sigma;V)
&:=V(s),
\\
\widehat J_j^{\BF}(s,\boldsymbol\sigma;V)
&:=
\max_{a\in\mathcal A}
\left\{
r(s,a)
+
\gamma
\sum_{s'\in\mathcal S}
\widehat{\bar P}_j(s'\mid s,a,\boldsymbol\sigma)
\widehat J_{j+1}^{\BF}(s',\boldsymbol\sigma;V)
\right\},
\qquad
j=K-1,\ldots,0.
\end{align}
The empirical block-feedback Bellman operator is
\begin{align}
\left(
\widehat{\mathcal T}^{\BF}V
\right)(s)
&:=
\int_{\mathcal Q_K}
\widehat J_0^{\BF}(s,\boldsymbol\sigma;V)
\,\widehat\mu_{N_2}(d\boldsymbol\sigma)
=
\frac{1}{N_2}
\sum_{i=1}^{N_2}
\widehat J_0^{\BF}
\left(
s,\boldsymbol\sigma^{\,i};V
\right).
\label{eq:empirical-bf-operator}
\end{align}

As in the true model, \(\widehat{\mathcal T}^{\BF}\) is a \(\gamma^K\)-contraction under the infinity norm. It therefore has a unique fixed point, denoted by $\widehat V
=
\widehat{\mathcal T}^{\BF}
\widehat V$.
BF--BOLA computes this empirical Bayesian value by value iteration offline. The resulting value remains a function on \(\mathcal S\). The \(N_2\) training predictions enter only through the empirical average in \(\widehat{\mathcal T}^{\BF}\).

\subsubsection{Online Policy Adaptation}

For a fresh prediction \(\boldsymbol\sigma\) observed during operation, define the empirical greedy action at each phase \(j\) and state \(s\), using a fixed tie-breaking rule, by
\begin{align}
\widehat a_j(s,\boldsymbol\sigma)
\in
\arg\max_{a\in\mathcal A}
\left\{
r(s,a)
+
\gamma
\sum_{s'\in\mathcal S}
\widehat{\bar P}_j(s'\mid s,a,\boldsymbol\sigma)
\widehat J_{j+1}^{\BF}
\left(
s',
\boldsymbol\sigma;
\widehat V
\right)
\right\}.
\label{eq:empirical-bf-greedy-action}
\end{align}
The learned block-feedback policy is $\widehat\pi_j^{\BF}
\left(
\cdot\mid s,\boldsymbol\sigma
\right)
:=
\delta_{\widehat a_j(s,\boldsymbol\sigma)}.$
At execution time, BF--BOLA can use a newly observed prediction directly. This prediction does not need to appear in the training samples. The training predictions are used to estimate the Bayesian value that enters the terminal term of the within-block recursion.

\begin{algorithm}[t]
\caption{BF--BOLA: Offline Bellman Learning with Online Policy Adaptation}
\label{alg:bf-bola}
\begin{algorithmic}[1]
\State Observe the deterministic reward \(r(s,a)\) once for each state-action pair.
\State Draw \(N_1\) next-state samples from every \((s,a)\in\mathcal U\).
\State Draw \(N_2\) independent predictions \(\{\boldsymbol\sigma^{\,i}\}_{i=1}^{N_2}\) with \(\boldsymbol\sigma^{\,i}\sim\mu\).
\State Construct \(\widehat P\), \(\widehat\mu_{N_2}\), and \(\widehat{\mathcal T}^{\BF}\).
\State Compute the empirical Bayesian value \(\widehat V\) as the fixed point of \(\widehat{\mathcal T}^{\BF}\).
\For{each block \(b=0,1,\ldots\)}
    \State Set \(\tau_b=bK\) and observe \(s_{\tau_b}\) and the fresh prediction \(\boldsymbol\sigma_b\).
    \State Compute \(\widehat J_j^{\BF}(\cdot,\boldsymbol\sigma_b;\widehat V)\) and \(\widehat a_j(\cdot,\boldsymbol\sigma_b)\) by backward induction.
    \For{\(j=0,\ldots,K-1\)}
        \State Select \(a_{\tau_b+j}=\widehat a_j(s_{\tau_b+j},\boldsymbol\sigma_b)\).
        \State Observe \(s_{\tau_b+j+1}\).
    \EndFor
\EndFor
\end{algorithmic}
\end{algorithm}

{We next bound the computational cost of offline value estimation and online policy adaptation.}

\begin{proposition}[{Computational Cost}]
\label{prop:bfbola-complexity}
Assume that transition rows and predictions are stored densely. Each offline Bellman update \(\widehat{\mathcal T}^{\BF}V\) requires \(O(N_2K|\mathcal S|^2|\mathcal A|)\) arithmetic operations. For each new online prediction, the within-block dynamic program requires \(O(K|\mathcal S|^2|\mathcal A|)\) arithmetic operations.
\end{proposition}

\subsection{Sample Complexity Guarantees}
\label{subsec:bf-sample-upper-bound}

We give finite-sample guarantees for learning a near-optimal block-feedback policy and for estimating the optimal Bayesian value.

\subsubsection{Policy Learning Guarantees}

Let \(D_1:=|\mathcal S|(|\mathcal A|-|\mathcal A^-|)N_1+|\mathcal S||\mathcal A|\) denote the total number of generative-model samples, including one reward observation from each state-action pair, and let \(D_2:=N_2\) denote the number of prediction samples. If \(\mathcal A=\mathcal A^-\), no transition samples are needed for uncovered rows.

\begin{theorem}[Near-Optimal Policy Learning with BF--BOLA]
\label{thm:bf-bola-policy-sample}
For any MDP satisfying Assumption~\ref{assump:prediction-law}, any confidence level \(\delta\in(0,1)\), policy accuracy \(\varepsilon_{\rm pol}\in(0,1]\), and prediction horizon \(K\ge1\), let \(D_1\) and \(D_2\) be the sample budgets used by BF--BOLA. If the sample budgets are chosen at the following orders, up to logarithmic factors,
\[
D_1
=
\widetilde O\!\left(
\frac{|\mathcal S|(|\mathcal A|-|\mathcal A^-|)}
{(1-\gamma)^3\varepsilon_{\rm pol}^2}
 \right)
+
|\mathcal S||\mathcal A|,
\qquad
D_2
=
\widetilde O\!\left(
1+
\frac{\gamma^{2K}}
{(1-\gamma)^2(1-\gamma^K)\varepsilon_{\rm pol}^2}
\right)
\]
then the deployed block-feedback policy \(\widehat\pi^{\BF}\) produced by Algorithm~\ref{alg:bf-bola} satisfies, with probability at least \(1-\delta\),
\[
\left\|
V^*
-
V^{\widehat\pi^{\BF}}
\right\|_\infty
\le
\varepsilon_{\rm pol}.
\]
\end{theorem}

The notation \(\widetilde O(\cdot)\) suppresses the logarithmic factors given in Appendix~\ref{proof_thm_bf_bola_policy_sample}.

\paragraph{Environment samples.}
The \(D_1\) requirement comes from direct interaction with the MDP through a generative model. Its leading term scales with \(|\mathcal S|(|\mathcal A|-|\mathcal A^-|)\), the number of transition rows not covered by predictions, rather than with the classical \(|\mathcal S||\mathcal A|\) transition-learning dimension for tabular model-based RL \citep{azar2012sample,gheshlaghi2013minimax,sidford2018near,agarwal2020model,li2023breaking}. Thus predictability reduces the environment-sampling burden by removing covered transition rows from the model-estimation problem. In the fully covered case \(\mathcal A^-=\mathcal A\), the transition-learning part of \(D_1\) disappears and only the additive reward-observation cost \(|\mathcal S||\mathcal A|\) remains.

\paragraph{Prediction samples.}
Prediction samples are used to estimate the law \(\mu\). For policy learning, prediction-law error affects future blocks, whose contribution is discounted by \(\gamma^K\). The prediction-sample requirement therefore decreases with \(K\) through the factor \(\gamma^{2K}/(1-\gamma^K)\). For \(K=1\), the leading term in the \(D_2\) bound has the same cubic dependence on \((1-\gamma)^{-1}\) as in classical discounted model-based RL. If \(\gamma^{2K}\le(1-\gamma)^2(1-\gamma^K)\varepsilon_{\rm pol}^2\), the bound in Theorem~\ref{thm:bf-bola-policy-sample} reduces to \(\widetilde O(1)\) prediction samples.

\subsubsection{Bayesian Value Estimation Guarantees}
\label{subsec:bf-bayesian-value-estimation}

We next bound the error in estimating the Bayesian value function, a stronger requirement than learning a near-optimal policy.

\begin{theorem}[Bayesian Value Estimation for BF--BOLA]
\label{thm:bf-bola-sample}
For any MDP satisfying Assumption~\ref{assump:prediction-law}, any confidence level \(\delta\in(0,1)\), estimation accuracy \(\varepsilon_{\rm est}\in(0,1]\), tradeoff parameter \(\alpha\in(0,1)\), and prediction horizon \(K\ge1\), let \(D_1\) and \(D_2\) be the sample budgets used by BF--BOLA. If the sample budgets are chosen at the following orders, up to logarithmic factors,
\[
D_1
=
\widetilde O\!\left(
\frac{|\mathcal S|(|\mathcal A|-|\mathcal A^-|)}
{(1-\gamma)^3(1-\alpha)^2\varepsilon_{\rm est}^2}
\right)
+
|\mathcal S||\mathcal A|,
\qquad
D_2
=
\widetilde O\!\left(
\frac{1}
{(1-\gamma)^2(1-\gamma^K)\alpha^2\varepsilon_{\rm est}^2}
\right)
\]
then the empirical Bayesian value computed by BF--BOLA satisfies, with probability at least \(1-\delta\),
\[
\left\|
\widehat V
-
V^*
\right\|_\infty
\le
\varepsilon_{\rm est}.
\label{eq:bf-bola-accuracy}
\]
\end{theorem}

The proofs of Theorems~\ref{thm:bf-bola-sample} and~\ref{thm:bf-bola-policy-sample} are given in Appendices~\ref{app:bf-bola-value-estimation-proofs} and~\ref{app:bf-bola-policy-learning-proofs}, respectively.

\subsection{Minimax Lower Bounds and Policy--Value Separation}
\label{subsec:learning-lower-bounds}

We establish lower bounds for both environment and prediction samples. These bounds also distinguish the prediction-sample requirements of policy learning and Bayesian value estimation.

\begin{theorem}[Minimax Lower Bounds for Block-Feedback Policy Learning]
\label{thm:bf-bola-policy-lower-bound}
There exist universal constants \(c,c_0,C_0>0\) such that, for every \(\gamma\in[1/2,1)\) under Assumption~\ref{assump:prediction-law}, the following lower bounds hold on corresponding hard subclasses for any algorithm that returns an \(\varepsilon_{\rm pol}\)-optimal block-feedback policy with probability at least \(2/3\):
\[
D_1
\ge
c\,
\frac{|\mathcal S|(|\mathcal A|-|\mathcal A^-|)}
{(1-\gamma)^3\varepsilon_{\rm pol}^2},
\qquad
D_2
\ge
c\,
\frac{\gamma^{2K}}
{(1-\gamma)^2(1-\gamma^K)\varepsilon_{\rm pol}^2}.
\]
The \(D_1\) bound holds for every \(K\ge1\) and \(\varepsilon_{\rm pol}\in(0,c_0]\) when \(|\mathcal S|\ge C_0\), \(|\mathcal A|\ge2\), and \(\mathcal A\setminus\mathcal A^-\ne\varnothing\), even if the prediction law is deterministic and known. The \(D_2\) bound applies for every \(K\ge1\) on a hard subclass with at most \(C_0K\) states and two covered actions, whenever \(0<\varepsilon_{\rm pol}\le c_0\min\{1,\gamma^K/(1-\gamma)\}\), even if the physical transition kernel and rewards are known and all action rows are covered by predictions. The two bounds are proved on separate hard subclasses.
\end{theorem}

The environment-sample lower bound matches BF--BOLA up to logarithmic factors. Within the stated accuracy range, the prediction-sample lower bound matches the \(D_2\) upper bound in Theorem~\ref{thm:bf-bola-policy-sample} apart from its additive constant. BF--BOLA is therefore minimax optimal in these ranges. The additive constant covers the case where uncertainty about future predictions is already below the target policy accuracy. The prediction-sample lower bound applies to model classes containing the \(O(K)\)-state hard family, for example when \(|\mathcal S|\ge C_0K\).

\begin{theorem}[Lower Bounds for Bayesian Value Estimation]
\label{thm:bayesian-value-estimation-lower-bounds}
\label{thm:prediction-oracle-lower-bound}
\label{thm:predictability-dimension-lower-bound}
There exist universal constants \(c,c_0,C_0>0\) such that, for every prediction horizon \(K\ge1\), every \(\gamma\in[1/2,1)\), and every \(\varepsilon_{\rm est}\in(0,c_0/(1-\gamma)]\), the following lower bounds hold on corresponding hard subclasses for any algorithm that returns an estimate \(\widetilde V\) satisfying \(\|\widetilde V-V^*\|_\infty\le\varepsilon_{\rm est}\) with probability at least \(2/3\):
\[
D_2
\ge
c\,
\frac{1}
{(1-\gamma)^2(1-\gamma^K)\varepsilon_{\rm est}^2},
\qquad
D_1
\ge
c\,
\frac{|\mathcal S|(|\mathcal A|-|\mathcal A^-|)}
{(1-\gamma)^3\varepsilon_{\rm est}^2}.
\]
The \(D_2\) bound holds on a hard subclass with at most \(C_0K\) states and two covered actions, even when the physical transition kernel and rewards are known. The \(D_1\) bound applies for \(|\mathcal S|\ge C_0\) and a nonempty set of uncovered actions, even when predictions on covered rows are exact and known. As above, the two bounds are proved on separate hard subclasses.
\end{theorem}

The two lower-bound constructions isolate the two statistical bottlenecks. Prediction-sample hardness comes from learning the informativeness of a calibrated prediction channel while the physical MDP is fixed. Environment-sample hardness comes from embedding a classical discounted-MDP hard instance only in the uncovered rows \citep{azar2012sample,gheshlaghi2013minimax,agarwal2020model,li2023breaking}. Prediction coverage removes covered rows from the environment-sampling burden, but it cannot make the remaining uncovered rows easier than ordinary MDP rows.

The proofs of both lower bounds are given in Appendix~\ref{app:learning-lower-bound-proofs}.

\begin{table}[t]
\centering
\small
\renewcommand{\arraystretch}{1.75}
\setlength{\tabcolsep}{6pt}
\begin{tabularx}{\textwidth}{@{}
>{\raggedright\arraybackslash}p{0.25\textwidth}
>{\centering\arraybackslash}X
>{\centering\arraybackslash}X
@{}}
\toprule
Target & Environment samples \(D_1\) & Prediction samples \(D_2\) \\
\midrule
Near-optimal policy
&
\(\displaystyle
\widetilde{\Theta}\!\left(
\frac{|\mathcal S|(|\mathcal A|-|\mathcal A^-|)}
{(1-\gamma)^3\varepsilon_{\rm pol}^2}
\right)\)
&
\(\displaystyle
\widetilde{\Theta}\!\left(
1+
\frac{\gamma^{2K}}
{(1-\gamma)^2(1-\gamma^K)\varepsilon_{\rm pol}^2}
\right)\)
\\
Bayesian value estimation
&
\(\displaystyle
\widetilde{\Theta}\!\left(
\frac{|\mathcal S|(|\mathcal A|-|\mathcal A^-|)}
{(1-\gamma)^3\varepsilon_{\rm est}^2}
\right)\)
&
\(\displaystyle
\widetilde{\Theta}\!\left(
\frac{1}
{(1-\gamma)^2(1-\gamma^K)\varepsilon_{\rm est}^2}
\right)\)
\\
\bottomrule
\end{tabularx}
\caption{Statistical rates under block feedback, up to logarithmic factors, where the upper and lower bounds in Sections~\ref{subsec:bf-sample-upper-bound} and~\ref{subsec:learning-lower-bounds} match. For policy learning, the \(D_2\) bounds match for \(0<\varepsilon_{\rm pol}\le c_0\min\{1,\gamma^K/(1-\gamma)\}\) under the conditions of Theorem~\ref{thm:bf-bola-policy-lower-bound}. The \(D_1\) rates omit the additive \(|\mathcal S||\mathcal A|\) reward observations.}
\label{tab:bf-learning-rates}
\end{table}

Table~\ref{tab:bf-learning-rates} summarizes the statistical rates. The environment-sample budget \(D_1\) scales with the number of uncovered transition rows, reducing the classical tabular dimension from \(|\mathcal S||\mathcal A|\) to \(|\mathcal S|(|\mathcal A|-|\mathcal A^-|)\), while preserving the usual per-row dependence \(1/((1-\gamma)^3\varepsilon^2)\). For policy learning, the prediction observed in the current block is used directly at deployment, so the learner only needs to estimate how prediction uncertainty affects future Bayesian values. The contribution of future blocks is discounted by \(\gamma^K\), yielding the \(\gamma^{2K}\) factor in the policy-learning sample complexity. By contrast, Bayesian value estimation targets the prediction-averaged value before the current prediction is observed and must also estimate the current-block prediction average, yielding the larger \(D_2\) rate.

\section{Numerical Experiments}
\label{sec:exp}

In a finite production MDP, we compare the value of predictions across decision protocols in Figure~\ref{fig:synthetic-value-information} and vary the two sample budgets separately in Figure~\ref{fig:synthetic-learning-rates}. Figure~\ref{fig:storage-control-case-study} compares the protocols in a wind-storage MDP calibrated to historical data.

\subsection{Experimental Design}
\label{subsec:experimental-design}

Figures~\ref{fig:synthetic-value-information} and \ref{fig:synthetic-learning-rates} use a four-state production MDP with three actions and discount factor \(\gamma=0.92\). The states are unprepared, prepared for type 0, prepared for type 1, and completed output. Two operating actions prepare or execute their respective types, and the third action waits or ships completed output. Switching types requires reconfiguration, and failed execution preserves the existing preparation. Setup and execution costs are sampled uniformly from \([0.015,0.025]\) and \([0.045,0.055]\), respectively, and shipping revenue from \([0.90,0.93]\). Adding the same constant \(0.06\) to every net reward places rewards in \([0,1]\) without changing policy comparisons.

Each type has an independent, equally likely good or bad execution mode in every period. The corresponding success probabilities are sampled uniformly from \([0.83,0.89]\) and \([0.04,0.10]\). At channel-noise level \(\eta\in[0,1]\), each observed mode bit is flipped with probability \(\eta/2\), and the prediction reports the resulting posterior transition rows. The rows average to the same base kernel for every noise level and coverage set. Independence across periods ensures that ignoring predictions recovers the base MDP. Values and policies are computed by solving the corresponding Bellman equations. Figure~\ref{fig:synthetic-value-information} uses 24 MDP instances, and Figure~\ref{fig:synthetic-learning-rates} uses the first 10 of these instances with 12 independent training runs at each budget.

\subsection{Value of Information}
\label{subsec:experiment-value-information}

Figure~\ref{fig:synthetic-value-information} separates the value of advance information from the value of state feedback. All values start from the unprepared state. In (a), FH, BF, and RH use the same prediction horizon and quality, while one-step BF observes only the current period's prediction. Advance information changes which type is prepared, and state feedback determines whether to execute, reconfigure, wait, or ship. RH updates its \(K\)-step prediction every period by retaining the overlapping \(K-1\) components and appending the prediction for one additional future period. Its value is computed on the augmented state space.

In (b), BF and RH both gain over one-step prediction. For BF, changing \(K\) changes both the horizon and the refresh times. For example, a new prediction arrives at time 3 when \(K=3\), but not when \(K=4\). A longer horizon therefore need not provide more information at every decision time, so BF values need not be monotone in \(K\). RH instead retains the shorter-horizon information as \(K\) grows, so its optimal value is nondecreasing in this experiment. Panels (c) and (d) show how noisier signals and reduced coverage lower value. FH's negative gain at zero coverage reflects the cost of open-loop commitment.

\begin{figure}[!t]
  \centering
  \begin{subfigure}[t]{0.48\linewidth}
    \centering
    \includegraphics[width=0.97\linewidth]{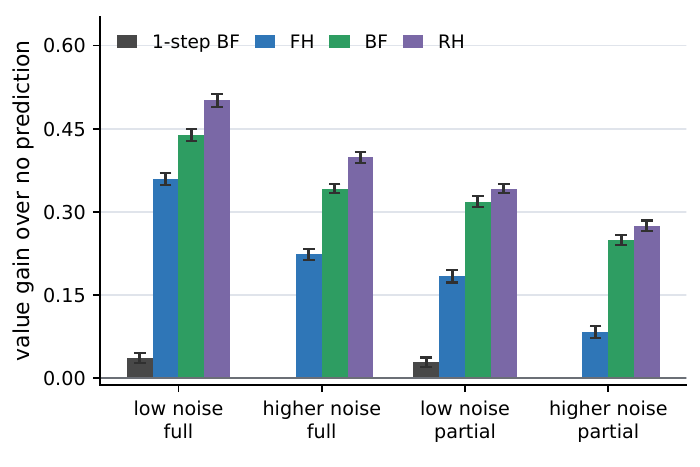}
    \caption{Protocol comparison}
  \end{subfigure}
  \hfill
  \begin{subfigure}[t]{0.48\linewidth}
    \centering
    \includegraphics[width=0.97\linewidth]{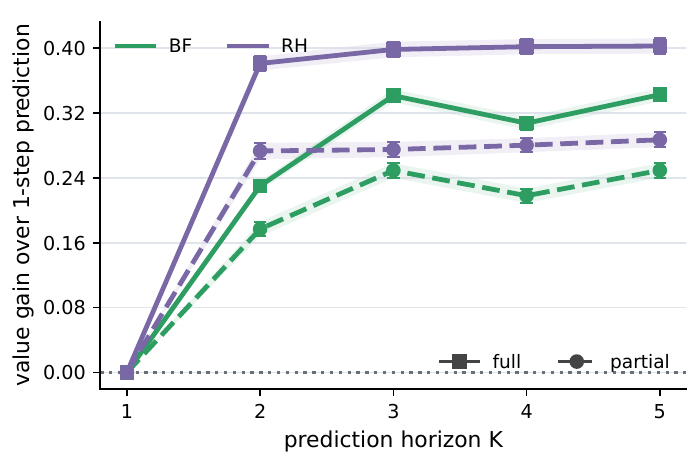}
    \caption{Value of advance prediction}
  \end{subfigure}

  \vspace{0.6em}
  \begin{subfigure}[t]{0.48\linewidth}
    \centering
    \includegraphics[width=0.97\linewidth]{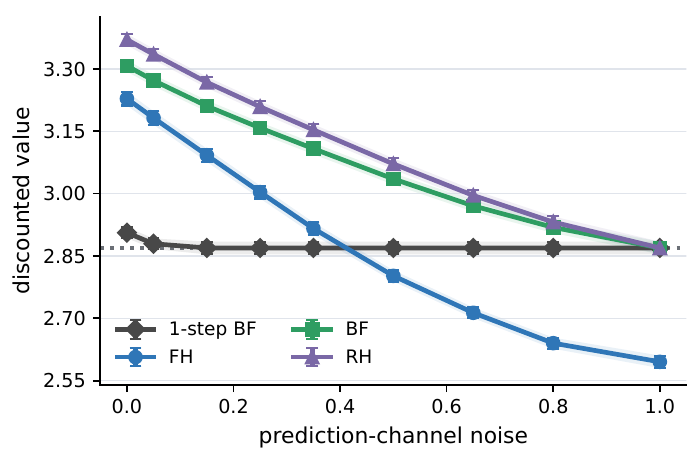}
    \caption{Prediction-channel noise}
  \end{subfigure}
  \hfill
  \begin{subfigure}[t]{0.48\linewidth}
    \centering
    \includegraphics[width=0.97\linewidth]{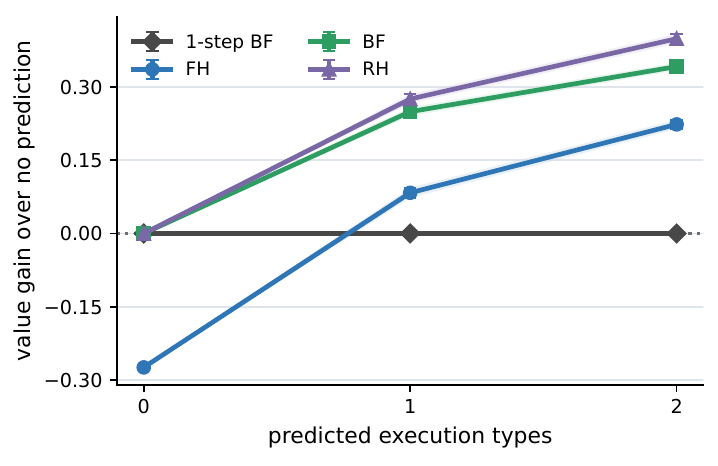}
    \caption{Action coverage}
  \end{subfigure}
  \caption{Value of transition predictions in the production MDP. The two noise levels in (a) are \(0\) and \(0.15\). In (b), solid curves have full coverage and dashed curves predict type 0 only. We use \(K=3\) in (a), (c), and (d), and channel noise \(0.15\) in (b) and (d). Each gain is measured against its baseline in the same MDP. Error bars and bands are \(1.96\) standard errors across 24 MDP instances.}
  \label{fig:synthetic-value-information}
\end{figure}
\subsection{Learning Efficiency}
\label{subsec:experiment-learning-efficiency}

Figure~\ref{fig:synthetic-learning-rates} studies BF--BOLA on the same production MDPs with \(K=3\) and channel noise \(0.15\). The top row uses full coverage and varies the prediction-sample budget \(D_2\), isolating prediction-law estimation. The bottom row provides the true prediction law and varies the number of samples per uncovered transition row, comparing type-0 coverage with no prediction in the same MDPs. Value error and policy loss are measured in the maximum norm over states, and each learned policy is evaluated under the true transition model and prediction law. The value errors decrease roughly as the inverse square root of the sample size, as shown by the reference curves. At \(D_2=2048\), the empirical success rate for \(\varepsilon_{\rm pol}=10^{-3}\) is \(95\%\), while the mean Bayesian value error remains approximately \(1.2\times10^{-2}\).

\begin{figure}[!t]
  \centering
  \begin{subfigure}[t]{0.48\linewidth}
    \centering
    \includegraphics[width=\linewidth]{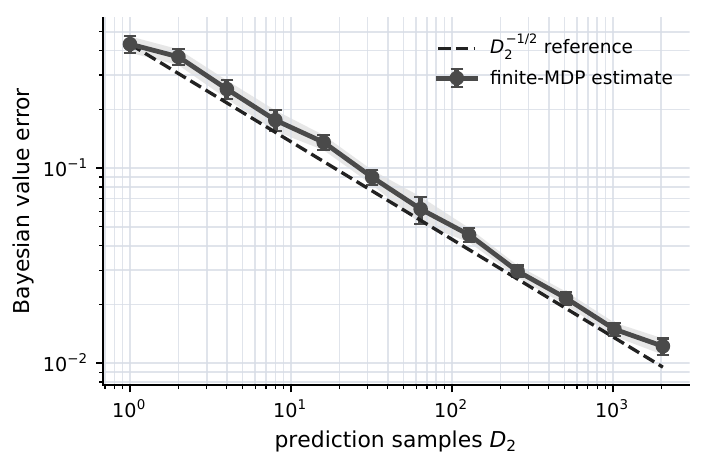}
    \caption{\(D_2\): value error}
  \end{subfigure}
  \hfill
  \begin{subfigure}[t]{0.48\linewidth}
    \centering
    \includegraphics[width=\linewidth]{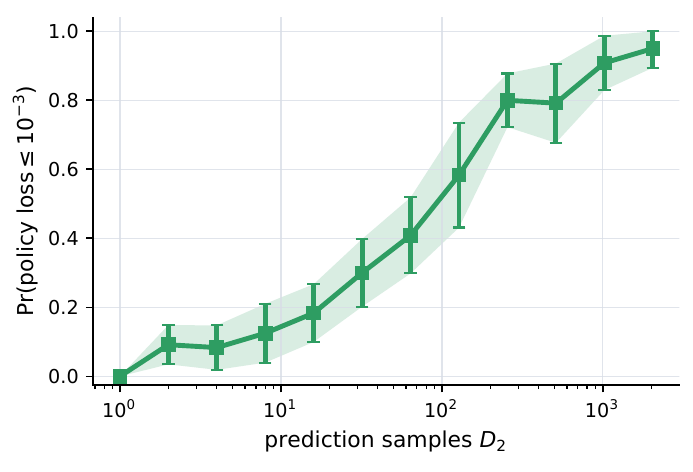}
    \caption{\(D_2\): policy success}
  \end{subfigure}

  \vspace{0.6em}
  \begin{subfigure}[t]{0.48\linewidth}
    \centering
    \includegraphics[width=\linewidth]{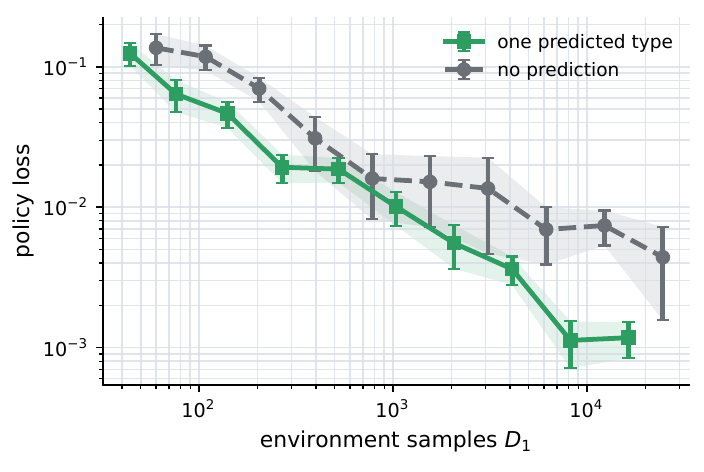}
    \caption{\(D_1\): policy loss}
  \end{subfigure}
  \hfill
  \begin{subfigure}[t]{0.48\linewidth}
    \centering
    \includegraphics[width=\linewidth]{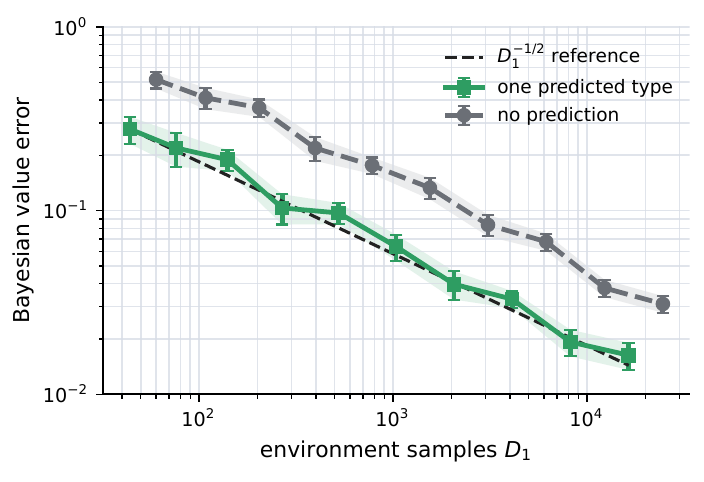}
    \caption{\(D_1\): value error}
  \end{subfigure}
  \caption{Two-source learning in the production MDP. The top row varies \(D_2\) under full coverage. The bottom row varies \(D_1\) under type-0 coverage or no prediction, with the prediction law provided. The budget \(D_1\) includes 12 reward observations. Policy success means loss at most \(10^{-3}\). Each budget uses 12 training runs on each of 10 MDP instances. Error bars and bands are \(1.96\) standard errors across instance means.}
  \label{fig:synthetic-learning-rates}
\end{figure}
\subsection{Wind-Storage Control Case Study}
\label{subsec:wind-storage-case-study}

Finally, we compare optimal FH, BF, and RH policies in a finite wind-storage MDP calibrated to historical data. The operator charges or discharges a battery to reduce wind-delivery penalties, following the storage setting of \citet{lu2024self}. Price and wind mismatch are exogenous, while state of charge is controlled.

\begin{figure}[!t]
  \centering
  \begin{subfigure}[t]{0.48\linewidth}
    \centering
    \includegraphics[width=\linewidth]{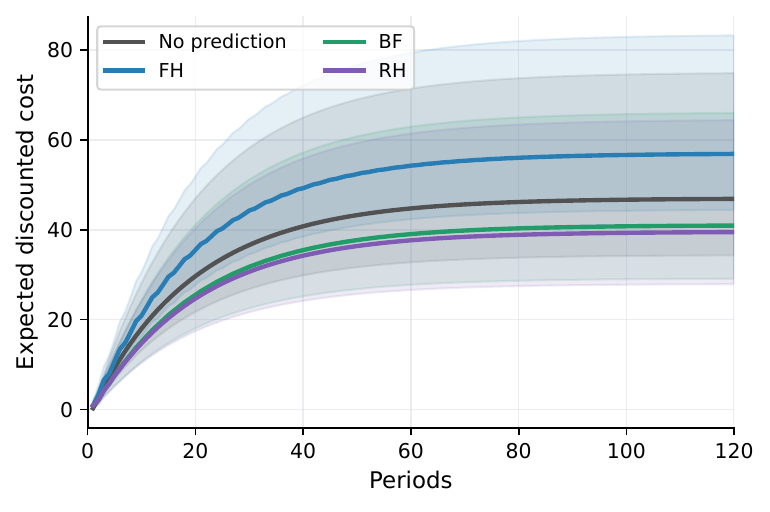}
    \caption{Discounted operating cost}
  \end{subfigure}
  \hfill
  \begin{subfigure}[t]{0.48\linewidth}
    \centering
    \includegraphics[width=\linewidth]{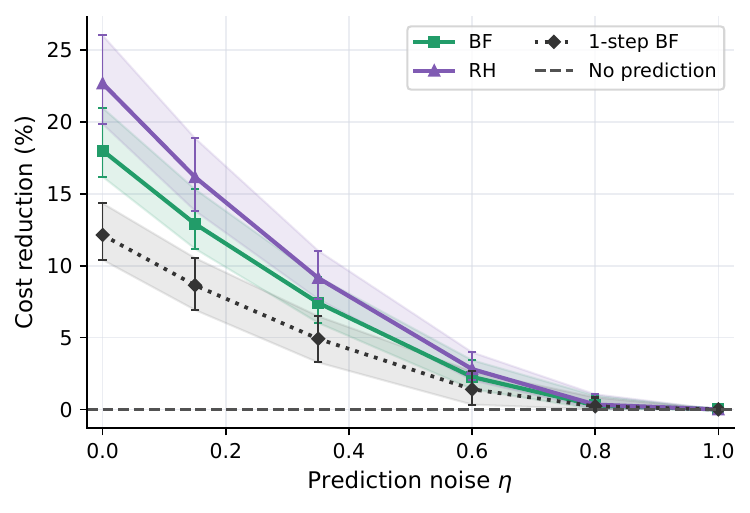}
    \caption{Prediction gain under noise}
  \end{subfigure}

  \vspace{0.6em}
  \begin{subfigure}[t]{0.48\linewidth}
    \centering
    \includegraphics[width=\linewidth]{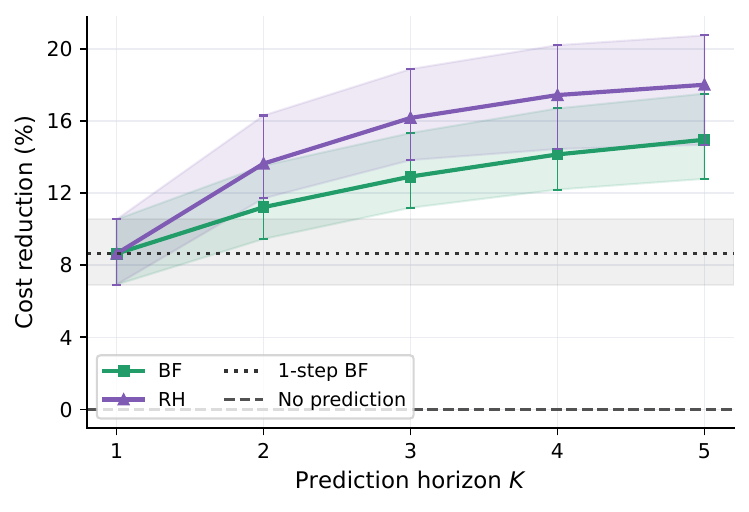}
    \caption{Prediction horizon}
  \end{subfigure}
  \hfill
  \begin{subfigure}[t]{0.48\linewidth}
    \centering
    \includegraphics[width=\linewidth]{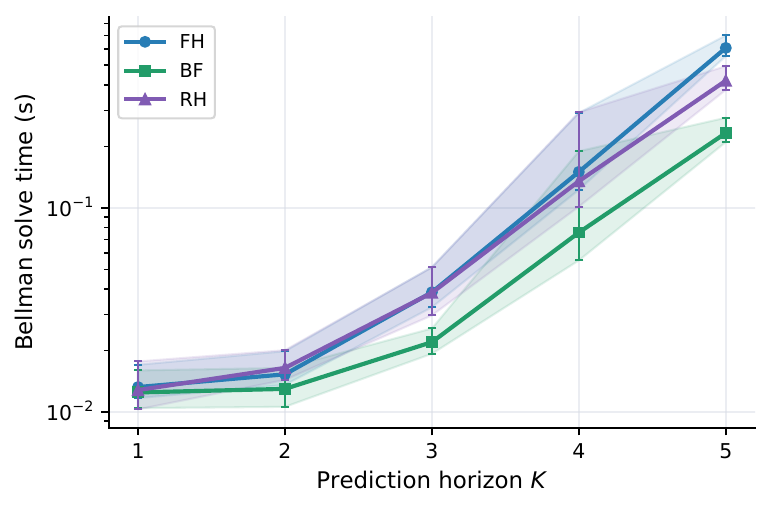}
    \caption{Bellman solution time}
  \end{subfigure}
  \caption{Optimal control in the data-calibrated wind-storage MDP, with capacity 4, per-period charge/discharge limit 1.5, and \(\gamma=0.95\). Panel (a) shows exact expected discounted partial sums under the infinite-horizon optimal policies. Panels (b,c) show paired percentage savings over the no-prediction policy for feedback protocols, including a one-step BF reference. Panels (a,b) use \(K=3\), except for that reference. Panels (a,c,d) use \(\eta=0.15\). Panel (d) times the model-specific Bellman solvers at the same value tolerance. Curves show means across 12 MDPs fitted to day-level bootstrap resamples. Bands and error bars show their 2.5th--97.5th percentiles.}
  \label{fig:storage-control-case-study}
\end{figure}

We calibrate the exogenous process using 52,404 timestamp-aligned wind and CAISO-price observations from January--June 2020 \citep{caiso}. We clip prices to \([0.1,100]\). Wind mismatches are divided by 100. Discretizing prices at the median (high/low) and wind mismatches by sign (surplus/deficit) yields four joint exogenous states. Their empirical frequencies define \(q\), and their conditional means specify prices \(p_e\) and mismatches \(\Delta_e\). The evaluation MDP draws \(e_t\) independently from \(q\) each period. We use an illustrative capacity-limited regime with state \((i_t,e_t)\), where \(i_t\in\{0,0.5,\ldots,4\}\) is stored energy. The per-period charge/discharge limit is 1.5, so rebuilding a depleted reserve requires multiple periods. Requested actions \(\{-1.5,-1,\ldots,1.5\}\) are clipped to \([-i_t,4-i_t]\) to obtain the realized action \(a_t\). Positive values of \(a_t\) charge the battery, so \(i_{t+1}=i_t+a_t\) and the cost is \(c_t=p_{e_t}(a_t-\Delta_{e_t})_+\). Energy and costs use the corresponding scaled units. All policies minimize \(\mathbb E[\sum_{t\ge0}0.95^t c_t]\), starting from \(i_0=2\) and \(e_0\sim q\). This is equivalent to maximizing bounded rewards \(r=1-c/\max c\).

Predictions reveal a noisy signal of each upcoming exogenous state. Given \(e'\sim q\), the signal is \(z=e'\) with probability \(1-\eta\) and is otherwise drawn independently from \(q\). Thus \(z\sim q\), and its posterior is \(\rho_z(e')=(1-\eta)\mathbf 1\{e'=z\}+\eta q(e')\). The predicted transition row combines the deterministic storage update with \(\rho_z\). It averages to the base transition row, and independent period signals satisfy Assumption~\ref{assump:prediction-law}. All actions are covered. FH and BF observe the next \(K\) signals at block boundaries. RH retains the overlapping \(K-1\) signals and observes one additional signal each period. We solve the FH and BF block Bellman equations and the augmented RH equation to value error at most \(10^{-7}\), without a terminal-value approximation.

Figure~\ref{fig:storage-control-case-study}(a) compares expected discounted operating costs. At \(K=3\) and \(\eta=0.15\), mean paired cost reductions relative to the no-prediction policy are \(12.9\%\) for BF and \(16.2\%\) for RH across 12 MDPs fitted to day-level bootstrap resamples. One-step BF saves \(8.6\%\), so multi-step BF provides an additional 4.3 percentage points of savings. The no-prediction policy retains state feedback, whereas FH commits to open-loop actions and cannot react to residual wind and price uncertainty within a block. The optimal cost under BF never exceeds that under FH in any tested setting. In (b), the prediction benefit vanishes as \(\eta\) reaches 1. In (c), increasing \(K\) from 1 to 5 raises the mean savings to \(15.0\%\) for BF and \(18.0\%\) for RH. The three protocols agree at \(K=1\), and FH and BF agree under exact prediction. In (d), BF has the lowest mean solve time for \(K=2,\ldots,5\). At \(K=5\), it takes 0.23 seconds, compared with 0.61 for FH and 0.42 for RH.

To check sensitivity to discretization, we keep the battery parameters fixed and split each of the four exogenous groups into finer wind-mismatch bins. The predictions still provide information only about the original four groups. With eight and sixteen exogenous states, three-step BF saves an additional 0.82 and 0.66 percentage points over one-step BF, respectively.

\FloatBarrier
\section{Conclusion}
This paper develops a Bellman framework for reinforcement learning with multi-step transition predictions. Under block feedback, predictions are used as temporary information within each block, while the long-run Bayesian value remains defined on the original state space. This yields a low-dimensional Bellman equation, a value-of-information analysis based on Bellman-Jensen and within-block feedback gains, and the BF--BOLA learning algorithm. The sample-complexity results show that prediction coverage reduces the environment-sampling dimension and that near-optimal policy learning can require fewer prediction samples than estimating the corresponding Bayesian value function.

Extending the Bellman framework to temporally correlated predictions is a natural next step. Another question is whether receding-horizon protocols admit low-dimensional representations that retain the value of more frequent prediction updates.

\clearpage
\bibliographystyle{plainnat}
\bibliography{references}

\clearpage
\appendix

\section*{Appendix Roadmap}
Appendix~\ref{app:bf-bellman-stack} proves Bellman optimality and the information ordering. Appendix~\ref{app:value-prediction-bounds} gives the value-of-information proofs and matching examples. Appendices~\ref{app:bf-bola-value-estimation-proofs} and~\ref{app:bf-bola-policy-learning-proofs} prove the learning upper bounds and supply their concentration and variance estimates. Appendix~\ref{app:learning-lower-bound-proofs} proves the learning lower bounds and verifies the hard constructions. Appendix~\ref{app:reference-protocol-baseline} contains the reference-protocol results.

\section{Bellman Theory and Information Ordering}
\label{app:bf-bellman-stack}
\label{app:prediction-model-proofs}
We construct the prediction process and transition outcomes on the same probability space to compare their information values.

\subsection{Proof of Theorem~\ref{thm:bf-bellman}}
\label{app:bf-bellman-proof}
\begin{proof}
Backward induction in \eqref{eq:bf-terminal}--\eqref{eq:Jj-recursion} gives, for any terminal values \(V,W\),
\[
\sup_{s,\boldsymbol\sigma}
\left|J_j^{\BF}(s,\boldsymbol\sigma;V)-J_j^{\BF}(s,\boldsymbol\sigma;W)\right|
\le \gamma^{K-j}\|V-W\|_\infty,
\qquad j=0,\ldots,K.
\]
Averaging at \(j=0\) shows that \(\mathcal T^{\BF}\) is a \(\gamma^K\)-contraction. The Banach fixed-point theorem therefore gives a unique fixed point \(U\in\mathbb R^{|\mathcal S|}\) \citep{banach1922operations}.

To identify \(U\) with the optimal Bayesian value, fix \(\pi\in\Pi^{\BF}\) and define its block-evaluation operator by
\[
(\mathcal T^\pi V)(s)
:=
\mathbb E^{\pi,\mu}\!\left[
\sum_{j=0}^{K-1}\gamma^j r(S_j,A_j)+\gamma^K V(S_K)
\,\middle|\,S_0=s
\right].
\]
At the next block boundary \(t=K\), the prediction \(\boldsymbol\Sigma_1\) for block \(1\) is sampled independently of the completed block, including \(S_K\). Since \(\pi\) uses the same phase-dependent rules in every block, its expected remaining return conditional on \(S_K\) is \(V^\pi(S_K)\). Thus \(V^\pi=\mathcal T^\pi V^\pi\), and the value carried across blocks depends only on the state, not on the previous prediction.

The within-block dynamic program gives \(\mathcal T^\pi U\le\mathcal T^{\BF}U=U\). The operator \(\mathcal T^\pi\) is monotone and a \(\gamma^K\)-contraction, so iterating it from \(U\) yields \(V^\pi\le U\). Conversely, choosing a maximizing action in \eqref{eq:Jj-recursion} with terminal value \(U\) gives a stationary block-feedback policy \(\pi^*\). The maximizers are measurable when selected according to a fixed ordering of the finite action set \citep{puterman2014markov}. This policy satisfies \(\mathcal T^{\pi^*}U=\mathcal T^{\BF}U=U\), so uniqueness of its evaluation fixed point gives \(V^{\pi^*}=U\). Taking the supremum over \(\pi\) proves \(V^*=U\).
\end{proof}

\subsection{Prediction Model and Coupling}

We realize the operational model on a common probability space as follows. In each block, first draw \(\boldsymbol\Sigma_b\sim\mu\). Conditional on this prediction, draw the potential next states \(\{Z_{b,j}(s,a)\}_{j,s,a}\) independently, with
\[
\mathbb P\!\left(Z_{b,j}(s,a)=s'\mid\boldsymbol\Sigma_b\right)
=\bar P_j(s'\mid s,a,\boldsymbol\Sigma_b).
\]
Repeat this construction independently across blocks, and set \(S_{bK+j+1}=Z_{b,j}(S_{bK+j},A_{bK+j})\). Consistency gives each potential next state marginal law \(P(\cdot\mid s,a)\). Conditional independence is imposed only after observing the complete prediction, so the prediction rows themselves may be dependent. For every nonanticipative policy, this construction has the same joint law of predictions and realized trajectories as the operational model. Base-MDP compatibility in Assumption~\ref{assump:prediction-law} therefore also holds when predictions are ignored.

This construction gives the conditional transition rows in \eqref{eq:operational-transition}. Given the prediction, a current potential next state is independent of earlier draws and policy randomization, so selecting a state-action pair does not change its reported row. The same independence ensures that earlier observations do not update later rows. For a policy that ignores predictions, Base-MDP compatibility in Assumption~\ref{assump:prediction-law} gives the base transition law at every step, and hence the same trajectory distribution and return as in the standard MDP.

For the regime comparisons, write \(\boldsymbol\Sigma_b^\varepsilon=\boldsymbol\Sigma_b\) and \(\mu^\varepsilon=\mu\). Define exact prediction by \(\Sigma^0_{b,j}(\cdot\mid s,a)=\delta_{Z_{b,j}(s,a)}\) on covered rows, and let \(\mu^0\) be its induced law. Full-coverage exact prediction reveals all potential next states in the block, and offline information reveals them in every block. All comparisons use this same coupling, rather than independently chosen exact-prediction laws. The enriched exact-information regime observes both \(\boldsymbol\Sigma_b^\varepsilon\) and \(\boldsymbol\Sigma_b^0\).

\begin{proof}[Proof of the enriched exact-information reduction]
We show that the imperfect signal gives no further value once the exact prediction is known. The enriched value is at least the exact-prediction value because the decision maker can ignore \(\boldsymbol\Sigma_b^\varepsilon\). For the reverse inequality, fix a terminal value \(V\) and realizations \((\xi,\zeta)\) of \((\boldsymbol\Sigma_b^\varepsilon,\boldsymbol\Sigma_b^0)\). The one-hot rows in \(\zeta\) already determine every covered transition. Uncovered potential next states are independent of both signals and of all covered potential next states, so their conditional rows remain \(P\).

We can now use backward induction to remove the imperfect signal from the value recursion. Starting from \(J_K(s,\xi,\zeta;V)=V(s)\), every \(J_j(s,\xi,\zeta;V)\) is independent of \(\xi\), because neither the reward nor any transition row depends on it. Averaging over the coupled predictions gives the ordinary exact-prediction BF operator. The two regimes therefore have the same fixed point and the same value.
\end{proof}

\subsection{Proof of Theorem~\ref{thm:value-ordering-standard-mdp}}
\label{app:value-ordering-proof}
\begin{proof}
Base-MDP compatibility in Assumption~\ref{assump:prediction-law} gives the first equality, and ignoring predictions gives the first inequality. The coupling argument in Appendix~\ref{app:prediction-model-proofs} shows that exact prediction dominates imperfect prediction. Under full-coverage exact prediction, a policy can ignore the additional rows. Finally, the offline oracle can follow any BF policy that uses full-coverage exact predictions.
\end{proof}

\section{Value-of-Prediction Bounds}
\label{app:value-prediction-bounds}
This appendix proves the value-of-information bounds and gives the reference comparison with open-loop block planning.

\subsection{Proof of Theorem~\ref{thm:bf-value-information-bound}}
\label{app:voi-bound-proof}
\begin{proof}
Write \(V_{\rm full}\), \(V_{\rm part}\), and \(V_{\rm imp}\) for the optimal BF values under exact predictions on all actions, exact predictions on \(\mathcal A^-\), and imperfect predictions on \(\mathcal A^-\), respectively. We bound the three terms in
\[
\begin{aligned}
\|V^{\Bayes,*}_{\mathrm{off}}-V_{\rm imp}\|_\infty
\le{}&\|V^{\Bayes,*}_{\mathrm{off}}-V_{\rm full}\|_\infty
+\|V_{\rm full}-V_{\rm part}\|_\infty
+\|V_{\rm part}-V_{\rm imp}\|_\infty.
\end{aligned}
\]

For the horizon term, Lemma~\ref{lem:post-prediction-markovization} gives a BF rule that matches any offline policy's expected rewards over the first \(K\) periods. Repeating this rule in later blocks gives an admissible stationary BF policy with nonnegative subsequent rewards. The offline policy's remaining reward is at most \(\gamma^K/(1-\gamma)\). Taking suprema over offline policies and using the information ordering gives \(\|V^{\Bayes,*}_{\mathrm{off}}-V_{\rm full}\|_\infty\le A_1\).

We next bound the loss from partial coverage. We apply Definition~\ref{def:partial-coverage-width} at \(V_{\rm part}\) and compare the two fixed-point equations. Using the \(\gamma^K\)-contraction, we obtain
\[
\|V_{\rm full}-V_{\rm part}\|_\infty
\le \gamma^K\|V_{\rm full}-V_{\rm part}\|_\infty+\omega_{K,\mathrm{cov}},
\]
so \(\|V_{\rm full}-V_{\rm part}\|_\infty\le A_3\).

We now bound the effect of prediction error within a single block. Fix exact and imperfect predictions of the same transition outcomes, and let \(d_j\) be their largest \(\ell_1\) row distance over covered state-action pairs at phase \(j\), with \(d_j=0\) if \(\mathcal A^-=\varnothing\). We use the same terminal Bayesian value \(0\le V\le(1-\gamma)^{-1}\mathbf1\) in both recursions and denote the infinity-norm difference between their phase-\(j\) value vectors by \(\Delta_j\).

The terminal values agree, so \(\Delta_K=0\). At each earlier phase, we separate the propagated value error from the transition-row error. The next-phase values are bounded by \((1-\gamma)^{-1}\), and taking the maximum over actions does not increase the largest action-value difference. The backward recursion therefore gives
\[
\Delta_j
\le\gamma\Delta_{j+1}+\frac{\gamma d_j}{1-\gamma},
\qquad j=0,\ldots,K-1.
\]
We unroll this recursion and then average over predictions. Since \(\mathbb E d_j=\epsilon_j\), the Bellman-operator error can be bounded by
\[
\begin{aligned}
\left\|
\mathcal T^{\BF}_{K,\mathcal A^-,\boldsymbol0}V
-\mathcal T^{\BF}_{K,\mathcal A^-,\boldsymbol\varepsilon}V
\right\|_\infty
&\le\mathbb E\Delta_0\le\sum_{j=0}^{K-1}\frac{\gamma^{j+1}\epsilon_j}{1-\gamma}.
\end{aligned}
\]
We then use the \(\gamma^K\)-contraction to convert this one-block operator bound into the Bayesian value bound \(\|V_{\rm part}-V_{\rm imp}\|_\infty\le A_2\). Combining the three comparisons gives the \(A_1+A_2+A_3\) bound. Since rewards lie in \([0,1]\), the value gap is also at most \((1-\gamma)^{-1}\).
\end{proof}

\subsection{Reference Comparison with Open-Loop Block Planning}
\label{app:fh-value-info}

We compare block feedback with open-loop planning to quantify the value of within-block state feedback. For a Bayesian value $V$ and prediction $\boldsymbol{\sigma}$, let $J^{\BF}_0(s,\boldsymbol{\sigma};V)$ be defined by \eqref{eq:bf-terminal}--\eqref{eq:Jj-recursion}. The corresponding open-loop block value is
\begin{align}
J^{\FH}_0(s,\boldsymbol{\sigma};V)
:=
\max_{\boldsymbol{a}\in\mathcal A^K}
\left\{
\sum_{t=0}^{K-1}\gamma^t
\sum_{s_t}P(s_t\mid s,\boldsymbol{a}_{0:t-1},\boldsymbol{\sigma}_{0:t-1})
r(s_t,a_t)
+\gamma^K
\sum_{s_K}P(s_K\mid s,\boldsymbol{a},\boldsymbol{\sigma})V(s_K)
\right\}.
\label{eq:fh-block-value}
\end{align}
We use \(V^{\Bayes,\FH,*}_{K,\mathcal A^-,\boldsymbol{\varepsilon}}\) for the open-loop Bayesian value defined in Appendix~\ref{app:fh-reference-results}. For any terminal value $V$, define the one-block feedback gain by
\begin{align}
  g_K(V)(s)
  &:=
  \mathbb{E}_{\boldsymbol{\sigma}}
  \left[
    J^{\BF}_0(s,\boldsymbol{\sigma};V)
    -
    J^{\FH}_0(s,\boldsymbol{\sigma};V)
  \right],
  \qquad \forall s\in\mathcal S.
  \label{eq:bf-local-gap}
\end{align}
This gain is nonnegative because BF can use any open-loop action sequence. Its largest value over states is \(\|g_K(V)\|_\infty\).

{
\begin{theorem}[Block Feedback Dominates Open-Loop Block Planning]
\label{thm:bf-dominates-fh}
Fix a prediction horizon $K\ge 1$, predictable action set $\mathcal A^-\subseteq\mathcal A$, and prediction error profile $\boldsymbol{\varepsilon}$. Let \(V_{\BF}\) and \(V_{\FH}\) be the optimal Bayesian values under BF and FH with these same predictions. Then, for every \(s\in\mathcal S\),
\begin{align}
0
\le
g_K(V_{\FH})(s)
\le
V_{\BF}(s)-V_{\FH}(s)
\le
\frac{\|g_K(V_{\BF})\|_\infty}{1-\gamma^K}.
\label{eq:bf-ge-fh}
\end{align}
\end{theorem}
}
The proof is provided in Appendix~\ref{proof_thm_bf_dominates_fh}. This comparison justifies fixed-horizon planning as the open-loop reference: block feedback can mimic any open-loop block sequence and can only gain from within-block state feedback.

\subsection{Proof of Theorem \ref{thm:bf-dominates-fh}}
\label{proof_thm_bf_dominates_fh}
{
We compare the optimal BF and FH values by comparing their Bellman operators. Define \((\mathcal T^{\FH}U)(s):=\mathbb E_{\boldsymbol\Sigma\sim\mu}[J^{\FH}_0(s,\boldsymbol\Sigma;U)]\). By Theorems~\ref{thm:bf-bellman} and~\ref{thm:fh-bayesian-bellman}, the two values are fixed points of monotone \(\gamma^K\)-contractions. We compare these fixed points using the standard discounted Bellman argument \citep{puterman2014markov,bertsekas1996neuro}.

For a given block-start state and prediction, the BF recursion can use any fixed open-loop sequence, so \(\mathcal T^{\BF}U\ge\mathcal T^{\FH}U\) and \(g_K(U)\ge0\) for every terminal value \(U\). Iterating \(\mathcal T^{\BF}\) from \(V_{\FH}\) gives an increasing sequence converging to \(V_{\BF}\). Consequently,
\[
  0\le g_K(V_{\FH})
  =\mathcal T^{\BF}V_{\FH}-V_{\FH}
  \le V_{\BF}-V_{\FH}
\]
pointwise, which proves the lower bound.

For the upper bound, we write \(D:=\|V_{\BF}-V_{\FH}\|_\infty\). We separate the one-block feedback gain from the difference in future values. Using the fixed-point equations and the definition of \(g_K\), we obtain
\[
\begin{aligned}
 D
 &=\left\|
   g_K(V_{\BF})
   +\mathcal T^{\FH}V_{\BF}-\mathcal T^{\FH}V_{\FH}
   \right\|_\infty\\
 &\le \|g_K(V_{\BF})\|_\infty+\gamma^K D.
\end{aligned}
\]
Rearranging and combining with the pointwise lower bound proves \eqref{eq:bf-ge-fh}.
}

\subsection{Current-Prediction Policy Matching}
\label{app:current-prediction-policy-matching}

\begin{lemma}[Current-prediction policy matching]
\label{lem:post-prediction-markovization}
Fix a block-start state \(s_0\) and condition on the exact next-state table revealed by the current prediction for all actions over the next \(K\) periods. For any offline policy that may depend on the full future next-state array, there exists a randomized block-feedback rule for the current block whose conditional state-action marginal at each phase \(j=0,\ldots,K-1\) matches that of the offline policy. Consequently, the two rules obtain the same expected discounted reward during the first \(K\) periods.
\end{lemma}

\begin{proof}
We construct a BF rule that chooses actions with the same conditional probabilities as the offline policy. Let \(\boldsymbol\Sigma^0_0\) denote the exact prediction for the initial block. For every phase \(j\) and every state \(s\) with positive conditional probability under the offline policy, we use the action distribution \(\mathbb P^{\rm off}(A_j=\cdot\mid S_j=s,S_0=s_0,\boldsymbol\Sigma^0_0)\). On null states, we choose any action distribution.

We verify the matching property by induction on \(j\). The state distributions agree initially. If they agree at phase \(j\), the definition of our rule gives the same joint distribution of state and action. The exact prediction then determines the next state for each state-action pair, so the next-state distributions also agree. Future next-state variables do not affect this transition. Thus the two policies have the same state-action distribution at every phase. Since rewards depend only on the current state and action, they receive the same expected discounted reward over the first \(K\) periods.
\end{proof}

\subsection{Proof of Theorem \ref{thm:finite-horizon-lower-bound}}
\label{proof_finite_window_lower_bound}

Fix \(K\ge1\) and \(\gamma\in[1/2,1)\). We construct two branches whose final transitions are not revealed until after the initial prediction horizon. The decision maker must choose a branch before seeing this information, whereas the offline oracle can use it to choose the better branch. Predictions are exact and all actions are covered, so the value gap comes only from the finite prediction horizon.

The state space consists of a start state \(s_0\), two deterministic branches \(\{x_{i,\ell}:i\in\{1,2\},\ell=1,\dots,K\}\), a good absorbing state \(g\), and a bad absorbing state \(b\). The action set contains two branch-selection actions \(a_1,a_2\) and an advance action \(c\). The rewards are
\[
  r(g,a)=1,\qquad
  r(s,a)=0\quad\text{for all }s\ne g,\ a\in\mathcal A .
\]
At \(s_0\), action \(a_i\) moves deterministically to \(x_{i,1}\), while action \(c\) moves to \(b\). Along each branch, action \(c\) moves deterministically from \(x_{i,\ell}\) to \(x_{i,\ell+1}\) for \(\ell<K\), while actions \(a_1,a_2\) move to \(b\). At \(x_{i,K}\), only action \(c\) is useful: it moves to \(g\) with probability \(1/2\) and to \(b\) otherwise. Actions \(a_1,a_2\) at \(x_{i,K}\) move deterministically to \(b\). Both \(g\) and \(b\) are absorbing under all actions.

The exact prediction process is the realization-level version of this transition kernel. Deterministic rows are predicted deterministically. The only nontrivial rows are the terminal branch rows \((x_{i,K},c)\). For each future time and branch \(i\), independently draw a Bernoulli random variable that determines whether this row maps to \(g\) or to \(b\). Thus the predictions are row-wise independent and consistent with the transition kernel. We set \(\mathcal A^-=\mathcal A\) and \(\boldsymbol\varepsilon=\boldsymbol 0\). In the calculation below, let \(U_i\) denote the Bernoulli variable governing action \(c\) from \(x_{i,K}\) at time \(K\).

In this instance, block feedback has no additional within-block decision to exploit after the initial branch is chosen: deviating from action \(c\) sends the process to \(b\), and the terminal random row is outside the initial \(K\)-step prediction. Hence the block-feedback \(K\)-step-prediction value coincides with the open-loop \(K\)-step-prediction value at \(s_0\).

At time \(0\), the actually reachable part of a \(K\)-step prediction reveals only the deterministic motion from \(s_0\) to the two states \(x_{1,K}\) and \(x_{2,K}\). The prediction may contain rows for \(x_{i,K}\) at earlier phases, but those rows correspond to hypothetical transitions at times before \(x_{i,K}\) is reached and are independent of the Bernoulli variables governing the actual transitions from \(x_{i,K}\) at time \(K\). Thus a policy with only the current \(K\)-step prediction has no information about which branch will lead to \(g\) after the branch has been selected. Since the two branches are symmetric, the \(K\)-step-prediction Bayesian policy obtains, from either branch,
\[
  V^{\Bayes,\BF,*}_{K,\mathcal A,\boldsymbol 0}(s_0)
  =
  \gamma^K\cdot
  \frac12\cdot
  \frac{\gamma}{1-\gamma}
  =
  \frac{\gamma^{K+1}}{2(1-\gamma)}.
\]
Here the factor \(\gamma^K\) is the discount to state \(x_{i,K}\), and the additional factor \(\gamma\) appears because the reward \(1\) starts only after the transition from \(x_{i,K}\) to \(g\).

In contrast, the offline oracle observes the future transition realizations on both branches before choosing \(a_1\) or \(a_2\). It selects a branch that leads to \(g\) whenever at least one branch is good. Since the two Bernoulli variables are independent, \(\mathbb P(\max\{U_1,U_2\}=1)=3/4\). Therefore,
\[
  V^{\Bayes,*}_{\mathrm{off}}(s_0)
  =
  \frac34\cdot
  \frac{\gamma^{K+1}}{1-\gamma}.
\]
Consequently, \(V^{\Bayes,*}_{\mathrm{off}}(s_0)-V^{\Bayes,\BF,*}_{K,\mathcal A,\boldsymbol 0}(s_0)=\gamma^{K+1}/[4(1-\gamma)]\ge \gamma^K/[8(1-\gamma)]\), where the last inequality uses \(\gamma\ge1/2\).
This proves the theorem with \(c=1/8\).

\subsection{Proof of Theorem \ref{thm:prediction-error-lower-bound}}
\label{proof_prediction_error_lower_bound}

We construct a risky action and a safe action that have the same value under an imperfect prediction. An exact prediction reveals when the risky action succeeds, allowing the decision maker to choose it only on success. Both prediction laws satisfy Assumption~\ref{assump:prediction-law}, so the value gap comes from the remaining uncertainty rather than a failure of calibration.

Fix \(K\ge1\), \(\gamma\in[1/2,1)\), and an error budget \(\epsilon\in(0,2]\). Set \(q:=\min\{\epsilon/8,(1-\gamma)/2\}\). Consider a finite MDP with a decision state \(x\) and a unit-reward absorbing state \(g\). The risky predictable action \(a\) has reward zero at \(x\) and moves to \(g\) with the fixed known probability \(q\), otherwise returning to \(x\). The safe action \(c\) stays at \(x\) deterministically and gives reward \(\rho:=\gamma q/(1-\gamma+\gamma q)\). State \(g\) gives reward one forever.

The exact prediction row reveals the Bernoulli next-state variable of the risky action. The imperfect prediction row is the uninformative calibrated posterior \(q\), so the physical transition kernel is the same under exact and imperfect predictions. The natural row error on the risky row is \(\bar\epsilon=\mathbb E\|\delta_Z-(q,1-q)\|_1=4q(1-q)\le\epsilon\).
The same construction at every phase gives \(\epsilon_j=\bar\epsilon\le\epsilon\) for all \(0\le j<K\), satisfying the prescribed uniform error budget.

Under the imperfect prediction, the risky action and the safe action are tied at the fixed point: the risky action has value \(\gamma[q(1-\gamma)^{-1}+(1-q)V_I(x)]\), while the safe action has value \(\rho+\gamma V_I(x)\), and the choice of \(\rho\) makes these two quantities equal. Since \(\rho<\gamma\), every trajectory starting from \(x\) has discounted return at most \(\gamma/(1-\gamma)\). Thus a revealed success favors the risky action, and a revealed failure favors the safe action, regardless of later predictions. This gives
\[
  V_E(x)
  =
  \frac{q\gamma(1-\gamma)^{-1}+(1-q)\rho}
  {1-\gamma(1-q)},
  \qquad
  V_I(x)=\frac{\rho}{1-\gamma}.
\]
A direct calculation using \(\rho=\gamma q/(1-\gamma+\gamma q)\) yields \(V_E(x)-V_I(x)=\gamma q(1-q)/(1-\gamma+\gamma q)^2\).
If \(\epsilon\le 4(1-\gamma)\), then \(q=\epsilon/8\) and the last display is at least \(c\epsilon/(1-\gamma)^2\). If \(\epsilon>4(1-\gamma)\), then \(q=(1-\gamma)/2\), and the same display is at least \(c/(1-\gamma)\). Combining the two cases gives
\[
  V^{\Bayes,\BF,*}_{K,\mathcal A,\boldsymbol 0}(x)
  -
  V^{\Bayes,\BF,*}_{K,\mathcal A,\boldsymbol\varepsilon}(x)
  \ge
  c
  \min\!\left\{
  \frac{\epsilon}{(1-\gamma)^2},
  \frac{1}{1-\gamma}
  \right\}.
\]
The construction uses independent next-state variables at every phase and state-action row, and its imperfect rows are calibrated posteriors of those variables. It therefore satisfies Assumption~\ref{assump:prediction-law} and the stated error budget.

\section{BF--BOLA Bayesian Value Estimation}
\label{app:bf-bola-value-estimation-proofs}
The proof in Appendix~\ref{app:value-learning-proof} separates transition error from prediction-law error. We bound each using fixed-target concentration, then transfer variance bounds to the empirical model to control the accumulated error. We use \(V^*_{Q,\nu}\) for the optimal Bayesian value under \((Q,\nu)\), as in that section. The implementation-cost proof follows at the end.

\subsection{Proof of Theorem~\ref{thm:bf-bola-sample}}
\label{app:value-learning-proof}
Let \(V^*_{Q,\nu}\) denote the optimal Bayesian value under transition model \(Q\) and prediction law \(\nu\), and write \(\widehat\mu=\widehat\mu_{N_2}\). Thus \(V^*=V^*_{P,\mu}\) is the true optimal value and \(\widehat V=V^*_{\widehat P,\widehat\mu}\) is the empirical value. All transition-error terms below are zero when \(\mathcal A^-=\mathcal A\).

\begin{proof}[Proof of Theorem~\ref{thm:bf-bola-sample}]
We change one part of the empirical model at a time. We first replace \(\widehat P\) by \(P\) while keeping \(\widehat\mu\), and then replace \(\widehat\mu\) by \(\mu\). This gives
\begin{align}
\|\widehat V-V^*\|_\infty
\le
\|\widehat V-V^*_{P,\widehat\mu}\|_\infty
+\|V^*_{P,\widehat\mu}-V^*\|_\infty.
\label{eq:bf-tight-main-decomposition}
\end{align}
The first term is the error from estimating uncovered transition rows. The second is the error from estimating the prediction law. We bound the first term using Lemma~\ref{lem:bf-transition-perturbation-tight} conditional on \(\widehat\mu\), and the second using Lemma~\ref{lem:bf-prediction-perturbation-tight} with \(Q=P\). We apply each lemma with failure probability \(\delta/2\). Since the two data sources are independent, a union bound gives both estimates with probability at least \(1-\delta\):
\[
\begin{aligned}
\|\widehat V-V^*_{P,\widehat\mu}\|_\infty
&\le C\sqrt{\frac{L_{\rm tr}^{\BF}}{N_1(1-\gamma)^3}}
+\frac{C L_{\rm tr}^{\BF}}{N_1(1-\gamma)^2},\\
\|V^*_{P,\widehat\mu}-V^*\|_\infty
&\le\frac{C}{1-\gamma}
\sqrt{\frac{\log(8|\mathcal S|/\delta)}{N_2(1-\gamma^K)}}
+\frac{C\log(8|\mathcal S|/\delta)}{N_2(1-\gamma)(1-\gamma^K)},
\end{aligned}
\]
where \(L_{\rm tr}^{\BF}\) is the logarithmic factor in the transition lemma evaluated at failure probability \(\delta/2\). We choose the universal constants in the sample sizes large enough that these two bounds are at most \((1-\alpha)\varepsilon_{\rm est}\) and \(\alpha\varepsilon_{\rm est}\), respectively. These choices also satisfy the conditions of both lemmas and bound the linear terms, since \(\varepsilon_{\rm est}\le1\) and \(1-\gamma^K\ge1-\gamma\). We then substitute the two bounds into \eqref{eq:bf-tight-main-decomposition} to complete the proof.
\end{proof}

\subsection{Supporting Bounds for Theorem~\ref{thm:bf-bola-sample}}
\label{proof_thm_bf_bola_sample}
The following lemma gives a standard resolvent bound used in discounted variance analyses \citep{agarwal2020model,zurek2024span}.

\begin{lemma}[Discounted variance bound]
\label{lem:resolvent-variance-inequality}
Let \(M\) be a stochastic matrix, \(\beta\in(0,1)\), and \(W\ge0\). Then
\[
\begin{aligned}
  (I-\beta M)^{-1}(I-\beta^2M)
  &=
  I+\beta(1-\beta)(I-\beta M)^{-1}M,
  \\
  \left\|
  (I-\beta M)^{-1}(I-\beta^2M)W
  \right\|_\infty
  &\le
  (1+\beta)\|W\|_\infty
  \le
  2\|W\|_\infty .
\end{aligned}
\]
\end{lemma}

\begin{proof}
The identity follows from \(I-\beta^2M=(I-\beta M)+\beta(1-\beta)M\). The Neumann series gives \(\|(I-\beta M)^{-1}MW\|_\infty\le\|W\|_\infty/(1-\beta)\), which proves the inequality.
\end{proof}

{
The same proof applies when \(M\) is a Markov transition kernel on a general state space. We will use the lemma for discounted-return variances. Write \(H=(1-\gamma)^{-1}\). For a Markov reward process with discount \(\beta\), let \(W\) be the variance of its return and \(q\) the variance of its one-step Bellman target. The standard variance identity is \(W=q+\beta^2MW\) \citep{sobel1982variance}. If \(0\le W\le H^2\mathbf1\) and \(0\le v\le q\), the lemma and Cauchy--Schwarz give
\begin{equation}
\left\|(I-\beta M)^{-1}v\right\|_\infty\le2H^2,
\qquad
\left\|(I-\beta M)^{-1}\sqrt v\right\|_\infty
\le\frac{\sqrt2 H}{\sqrt{1-\beta}}.
\label{eq:discounted-return-variance}
\end{equation}
In the applications below, conditional variance lets us compare the variance of a prediction-dependent target with the full Bellman-target variance \(q\).
}

The two terms in \eqref{eq:bf-tight-main-decomposition} require separate comparisons: changing the transition model at a fixed prediction law, and changing the prediction law at a fixed transition model. The following lemmas give these bounds, with proofs in Appendices~\ref{proof-bf-transition-perturbation-tight} and \ref{proof-bf-prediction-perturbation-tight}.

\begin{lemma}[BF transition perturbation]
\label{lem:bf-transition-perturbation-tight}
Condition on an empirical prediction law \(\widehat\mu=N_2^{-1}\sum_{i=1}^{N_2} \delta_{\boldsymbol\sigma^i}\). Let
\[
  L_{\mathrm{tr}}^{\BF}
  :=
  \log\!\left(
    \frac{
    1+cKN_2
    \bigl(|\mathcal S|(|\mathcal A|-|\mathcal A^-|)+1\bigr)
    |\mathcal A|}
    {\delta}
  \right).
\]
There is a universal constant \(C\) such that, with probability at least \(1-\delta\) over the transition samples, if \(N_1\ge C L_{\mathrm{tr}}^{\BF}/(1-\gamma)^2\), then
\begin{align}
  \|V^*_{\widehat P,\widehat\mu}
    -V^*_{P,\widehat\mu}\|_\infty
  \le
  C\sqrt{
    \frac{L_{\mathrm{tr}}^{\BF}}
    {N_1(1-\gamma)^3}
  }
  +
  C\frac{L_{\mathrm{tr}}^{\BF}}
  {N_1(1-\gamma)^2}.
  \label{eq:bf-transition-perturbation-tight}
\end{align}
\end{lemma}

\begin{lemma}[BF prediction-oracle perturbation]
\label{lem:bf-prediction-perturbation-tight}
Fix a one-step model \(Q\). Let \(V=V^*_{Q,\mu}\) and \(\widehat V=V^*_{Q,\widehat\mu}\). There is a universal constant \(C\) such that, with probability at least \(1-\delta\) over the prediction samples, if \(N_2\ge C\log(4|\mathcal S|/\delta)/(1-\gamma^K)^2\), then
\begin{align}
  \|\widehat V-V\|_\infty
  \le
  \frac{C}{1-\gamma}
  \sqrt{\frac{\log(4|\mathcal S|/\delta)}{N_2(1-\gamma^K)}}
  +
  \frac{C\log(4|\mathcal S|/\delta)}{N_2(1-\gamma)(1-\gamma^K)}.
  \label{eq:bf-prediction-perturbation-tight}
\end{align}
\end{lemma}

\subsubsection{Proof of Lemma \ref{lem:bf-transition-perturbation-tight}}
\label{proof-bf-transition-perturbation-tight}
{
We prove the bound in four steps. We first express the problem as an augmented MDP and bound its one-step Bellman errors. We then use the variance of the total return to bound the accumulated error, following \citet{sidford2018near,agarwal2020model}. Finally, we combine upper and lower bounds on the value difference. The same transition samples are used for different predictions, so our concentration event must hold for all of them at once. Throughout this proof, we write \(H=(1-\gamma)^{-1}\) and \(\lambda=L_{\mathrm{tr}}^{\BF}/N_1\).

\paragraph{Step 1. Reduce to an augmented MDP.}
We condition on \(\widehat\mu\) so that the sampled predictions are fixed. We can then include these finitely many predictions and the current phase in an augmented MDP. Its states are \(x=(s,i,j)\in\mathcal S\times[N_2]\times\{0,\ldots,K-1\}\), where \(i\) identifies the prediction and \(j\) is the phase. A transition leads to \((s',i,j+1)\) when \(j<K-1\). At the final phase, the next state is \((s',i',0)\), with \(i'\) drawn uniformly from \([N_2]\). Covered actions use the rows of \(\boldsymbol\sigma^i\), and uncovered actions use \(P(\cdot\mid s,a)\) or \(\widehat P(\cdot\mid s,a)\). Note that all augmented states using the same uncovered row \((s,a)\) share the same \(N_1\) transition samples.

Let \(V_{\mathrm{aug}}^*\) and \(\widehat V_{\mathrm{aug}}\) be the true and empirical optimal augmented values. By averaging them over \(i\) at phase zero, we recover \(V^*_{P,\widehat\mu}\) and \(V^*_{\widehat P,\widehat\mu}\), respectively. We therefore only need to bound \(E=\|\widehat V_{\mathrm{aug}}-V_{\mathrm{aug}}^*\|_\infty\).

\paragraph{Step 2. Establish a fixed-target concentration event.}
We apply concentration to the true optimal value because it does not depend on the transition samples. For each context \(x=(s,i,j)\), its next-state target can be written as
\[
  U_x(y):=
  \begin{cases}
    V_{\mathrm{aug}}^*(y,i,j+1), & j<K-1,\\
    N_2^{-1}\sum_{i'=1}^{N_2}V_{\mathrm{aug}}^*(y,i',0), & j=K-1.
  \end{cases}
\]
We need to bound the errors in both the mean and the second moment, since we will later compare true and empirical variances. For an uncovered action \(a\), we write \(z=(s,a)\), \(v_{x,a}=\operatorname{Var}_{P_z}(U_x)\), \(\rho_{x,a}=(\widehat P_z-P_z)U_x\), and \(\xi_{x,a}=(\widehat P_z-P_z)U_x^2\). We set these quantities to zero for covered actions. Conditional on \(\widehat\mu\), every \(U_x\in[0,H]^{\mathcal S}\) is independent of the transition samples. We can therefore apply Bernstein's inequality \citep{boucheron2013concentration} to \(U_x\) and \(U_x^2\) to obtain
\begin{align}
  |\rho_{x,a}|&\le C\sqrt{\lambda v_{x,a}}+CH\lambda,
  &
  |\xi_{x,a}|&\le CH\sqrt{\lambda v_{x,a}}+CH^2\lambda
  \label{eq:bf-tr-bernstein-event}
\end{align}
simultaneously with probability at least \(1-\delta\). We use \(\operatorname{Var}(U_x^2)\le4H^2\operatorname{Var}(U_x)\) for the second inequality. The logarithmic factor \(L_{\mathrm{tr}}^{\BF}\) comes from a union bound over uncovered rows, predictions, and phases. These events need not be independent. Since the inequalities hold for every action, we can also use them for the action chosen by the empirical greedy policy.

\paragraph{Step 3. Transfer variance through the empirical resolvent.}
We first bound how much the empirical value can exceed the true value. Let \(\widehat M\) be the empirical augmented transition matrix of an empirical greedy policy, and write \(\mathcal R=(I-\gamma\widehat M)^{-1}\). In \eqref{eq:bf-tr-bernstein-event}, we select the greedy action at each state and denote the corresponding residual and variance by a subscript \(x\). Comparing the Bellman equations and summing the discounted residuals gives
\begin{align}
  E_+:=\max_x(\widehat V_{\mathrm{aug}}(x)-V_{\mathrm{aug}}^*(x))_+
  \le \gamma\|\mathcal R|\rho|\|_\infty .
  \label{eq:bf-tr-upper-resolvent}
\end{align}
To use this inequality, we must bound the accumulated variance. There are two differences to account for: concentration uses the true target and the true transition row, but the return-variance identity uses their empirical versions. We define \(\widehat U_x\) by replacing \(V_{\mathrm{aug}}^*\) with \(\widehat V_{\mathrm{aug}}\) in \(U_x\). For an uncovered selected row, let \(\widehat v_x^*=\operatorname{Var}_{\widehat P_z}(U_x)\) and \(\widehat v_x=\operatorname{Var}_{\widehat P_z}(\widehat U_x)\). We set both to zero for covered actions. We can compare these variances using \(v_x=\widehat v_x^*+2(P_zU_x)\rho_x+\rho_x^2-\xi_x\) and \(\widehat v_x^*\le2\widehat v_x+2E^2\). Substituting the bounds from \eqref{eq:bf-tr-bernstein-event} gives
\begin{align}
  S:=\|\mathcal R v\|_\infty
  &\le
  2\|\mathcal R\widehat v\|_\infty+2HE^2
  +CH\sqrt{\lambda}\,R
  +CH^3(\lambda+\lambda^2)+C\lambda S,
  \qquad
  R:=\|\mathcal R\sqrt v\|_\infty .
  \label{eq:bf-tr-S-expanded}
\end{align}

We bound the first term using \eqref{eq:discounted-return-variance}. Let \(\widetilde v\) be the full one-step variance of \(\widehat V_{\mathrm{aug}}\) under the empirical rule. Conditional variance gives \(\widehat v\le\widetilde v\), and the Bellman target has variance \(\gamma^2\widetilde v\). The common bound therefore gives \(\|\mathcal R\widehat v\|_\infty\le CH^2\) for \(\gamma\ge1/2\). For \(\gamma<1/2\), the same conclusion follows from \(\widetilde v\le H^2\mathbf1\) and \(H\le2\).

We now return to \eqref{eq:bf-tr-S-expanded}. By Cauchy's inequality, \(R^2\le HS\). Under the sample-size condition \(\lambda\le c/H^2\), we can use Young's inequality and move the resulting terms proportional to the left-hand side to obtain \(S\le CH^2+CHE^2\) and \(R\le CH^{3/2}+CHE\). Substituting these bounds into \eqref{eq:bf-tr-upper-resolvent}, we obtain
\begin{align}
  E_+\le C\sqrt{H^3\lambda}+CH^2\lambda+CH\sqrt{\lambda}\,E .
  \label{eq:bf-tr-upper}
\end{align}

\paragraph{Step 4. Bound the lower deviation and conclude.}
We next bound how much the true value can exceed the empirical value. For this, we evaluate a true optimal policy in the empirical model. Let \(\pi^*\) be a true optimal augmented policy and let \(V_{\widehat P}^{\pi^*,\mathrm{aug}}\) be its empirical evaluation value. By empirical optimality, \((V_{\mathrm{aug}}^*-\widehat V_{\mathrm{aug}})_+\le|V_{\mathrm{aug}}^*-V_{\widehat P}^{\pi^*,\mathrm{aug}}|\). We write \(E_*=\|V_{\mathrm{aug}}^*-V_{\widehat P}^{\pi^*,\mathrm{aug}}\|_\infty\). We can use the same concentration event because the true value of \(\pi^*\) is \(V_{\mathrm{aug}}^*\), so the residual still uses the fixed targets \(U_x\). We apply the argument in Step 3 to the empirical transition matrix and evaluation value of \(\pi^*\), taking absolute values in the Bellman comparison. This gives
\[
  E_*\le C\sqrt{H^3\lambda}+CH^2\lambda+CH\sqrt{\lambda}\,E_* .
\]
The condition \(\lambda\le c/H^2\) makes the coefficient of the last term small enough to move it to the left-hand side. We combine this bound with \eqref{eq:bf-tr-upper} and do the same for \(CH\sqrt{\lambda}\,E\), obtaining \(E\le C\sqrt{H^3\lambda}+CH^2\lambda\). Finally, we average the augmented values at phase zero to obtain \eqref{eq:bf-transition-perturbation-tight}.
}

\subsubsection{Proof of Lemma \ref{lem:bf-prediction-perturbation-tight}}
\label{proof-bf-prediction-perturbation-tight}
{
We prove the prediction-law bound by following the same four steps at the block level. We first bound the empirical error for a target built from the true value. We then compare its variance with the variance under empirical policies and use this comparison to bound both directions of the value error.

We write \(H=(1-\gamma)^{-1}\), \(\beta=\gamma^K\), \(h=1-\beta\), and \(\lambda=\log(4|\mathcal S|/\delta)/N_2\). For a terminal value \(u\in[0,H]^{\mathcal S}\), let \(\Psi_u(s,\boldsymbol\sigma)=J^{\BF}_{Q,0}(s,\boldsymbol\sigma;u)\). The within-block recursion gives \(0\le\Psi_u\le H\) and \(|\Psi_u-\Psi_{u'}|\le\beta\|u-u'\|_\infty\). To apply concentration, we use the true-value target \(G_s=\Psi_V(s,\cdot)\), with empirical error \(\rho_s=(\widehat\mu-\mu)G_s\) and variance \(v_s=\operatorname{Var}_\mu(G_s)\). This target is independent of the prediction samples, and \(\mu G_s=V(s)\).

\paragraph{Step 1. Concentrate the fixed targets.}
We first bound the residual and compare the true and empirical variances of the fixed target. We apply Bernstein's inequality \citep{boucheron2013concentration} and the sample-variance bound of \citet{maurer2009empirical}, rescaled from \([0,1]\) to \([0,H]\), to obtain
\begin{align}
  |\rho_s|&\le C\sqrt{\lambda v_s}+CH\lambda,
  &
  v_s&\le C\operatorname{Var}_{\widehat\mu}(G_s)+CH^2\lambda,
  \qquad s\in\mathcal S,
  \label{eq:bf-pred-rho-fixed}
\end{align}
on a joint event of probability at least \(1-\delta\). The logarithm covers the union over states and both inequalities. The variance here uses divisor \(N_2\), which differs from the unbiased sample variance in the cited result by a factor at most two.

\paragraph{Step 2. Transfer variance to a sample-dependent target.}
We next compare the variance of the fixed target with that of a target obtained from an empirical policy. We will use this comparison for both directions of the value error. Let \(M\) be any empirical stochastic matrix and let \(A_s(\boldsymbol\sigma^i)\in[0,H]\) satisfy \(|A_s(\boldsymbol\sigma^i)-G_s(\boldsymbol\sigma^i)|\le\beta\Delta\). We write \(\bar v_s=\operatorname{Var}_{\widehat\mu}(A_s)\) and \(\mathcal R_M=(I-\beta M)^{-1}\). Suppose that \(\|\mathcal R_M\bar v\|_\infty\le CH^2\). We can then use \(\operatorname{Var}_{\widehat\mu}(G_s)\le2\bar v_s+2\beta^2\Delta^2\) and \eqref{eq:bf-pred-rho-fixed} to obtain
\begin{align}
  S_M:=\|\mathcal R_M v\|_\infty
  &\le CH^2+\frac{C\Delta^2}{h}+\frac{CH^2\lambda}{h},
  &
  R_M:=\|\mathcal R_M\sqrt v\|_\infty
  &\le \frac{CH}{\sqrt h}+\frac{C\Delta}{h}.
  \label{eq:bf-pred-variance-transfer-clean}
\end{align}
For the second bound, we use Cauchy's inequality \(R_M^2\le S_M/h\) and the sample-size condition \(\lambda\le ch^2\). Once the event in Step 1 holds, this comparison is deterministic. We may therefore use it even when \(M\) and \(A_s\) depend on the samples.

\paragraph{Step 3. Bound the upper deviation.}
We first bound the upper deviation by following an empirical greedy rule. For a fixed block rule \(d\), let \(\Psi_u^d\) denote its within-block value. Let \(\widehat d\) be greedy for \(\widehat V\), and let \(\widehat M\) be its block-end transition matrix under \(\widehat\mu\). By greediness, \(\Psi_{\widehat V}^{\widehat d}=\Psi_{\widehat V}\) and \(\Psi_V^{\widehat d}\le G\). These relations allow us to use the fixed-target residual in the value comparison. They give \(\widehat V-V\le\beta\widehat M(\widehat V-V)+\rho\), from which we obtain
\begin{align}
  E_+:=\max_s(\widehat V(s)-V(s))_+
  \le \|(I-\beta\widehat M)^{-1}|\rho|\|_\infty .
  \label{eq:bf-pred-upper-resolvent-clean}
\end{align}
We apply Step 2 with \(A_s=\Psi_{\widehat V}^{\widehat d}(s,\cdot)=\Psi_{\widehat V}(s,\cdot)\) and \(E=\|\widehat V-V\|_\infty\). The Lipschitz bound gives the required comparison \(|A_s-G_s|\le\beta E\).

It remains to verify the variance bound assumed in Step 2. Under the empirical rule, let \(\widetilde v\) be the variance of the full one-block Bellman target. Since \(A_s\) is its conditional mean given the prediction, conditional variance gives \(\bar v\le\widetilde v\). We apply \eqref{eq:discounted-return-variance} to the empirical block process and obtain \(\|(I-\beta\widehat M)^{-1}\bar v\|_\infty\le2H^2\). Combining this with Steps 1--2 and \eqref{eq:bf-pred-upper-resolvent-clean}, we obtain
\begin{align}
  E_+\le CH\sqrt{\lambda/h}+\frac{CH\lambda}{h}
  +\frac{C\sqrt{\lambda}}{h}E .
  \label{eq:bf-pred-upper}
\end{align}

\paragraph{Step 4. Bound the lower deviation and conclude.}
For the lower deviation, we evaluate a true optimal rule under the empirical prediction law. This lets us keep the same fixed target \(G\). Let \(d^*\) be greedy for \(V\), and let \(\widetilde V=V^{d^*}_{Q,\widehat\mu}\) be its empirical evaluation value. By empirical optimality, \((V-\widehat V)_+\le|V-\widetilde V|\). We write \(E_*=\|V-\widetilde V\|_\infty\), and let \(M_*\) be the empirical block-end transition matrix of \(d^*\). Comparing the fixed-rule evaluation equations gives \(|V-\widetilde V|\le|\rho|+\beta M_*|V-\widetilde V|\). We can now apply Steps 2--3 with \(A_s=\Psi_{\widetilde V}^{d^*}(s,\cdot)\) and \(\Delta=E_*\), since \(\Psi_V^{d^*}=G\) and \(|\Psi_{\widetilde V}^{d^*}-G|\le\beta E_*\). The same return-variance argument gives \(\|\mathcal R_{M_*}\bar v\|_\infty\le2H^2\), so
\[
  E_*\le CH\sqrt{\lambda/h}+\frac{CH\lambda}{h}
  +\frac{C\sqrt{\lambda}}{h}E_* .
\]
Under \(\lambda\le ch^2\), we can move the last term to the left-hand side. We then combine this lower-deviation bound with \eqref{eq:bf-pred-upper} and move the term proportional to \(E\) to the left-hand side as well. This proves \eqref{eq:bf-prediction-perturbation-tight}.
}

\subsection{Proof of Proposition \ref{prop:bfbola-complexity}}
\label{proof_prop_bfbola_complexity}
{
For a fixed empirical prediction, the within-block backward recursion has \(K|\mathcal S||\mathcal A|\) state--phase--action updates, and each dense transition update sums over \(|\mathcal S|\) successors. Averaging the empirical BF operator over \(N_2\) stored predictions therefore costs \(O(N_2K|\mathcal S|^2|\mathcal A|)\) arithmetic operations per sweep. For a new online prediction, we run the same recursion once, at a cost of \(O(K|\mathcal S|^2|\mathcal A|)\) arithmetic operations.

The empirical BF operator's \(\gamma^K\)-contraction gives the iteration count. After \(m\) sweeps from any value in \([0,(1-\gamma)^{-1}]^{|\mathcal S|}\), the error is at most \((\gamma^K)^m/(1-\gamma)\). To reach infinity-norm accuracy \(\eta_{\rm opt}\), we therefore need \(m=O(\log((1-\gamma)^{-1}/\eta_{\rm opt})/[-\log(\gamma^K)])\) sweeps.

For storage, a dense prediction contains \(K|\mathcal S||\mathcal A^-|(|\mathcal S|-1)\) free scalar probabilities. Storing all prediction samples requires \(O(N_2K|\mathcal S||\mathcal A^-|(|\mathcal S|-1))\) scalars, and storing the uncovered empirical transition rows requires \(O(|\mathcal S|^2(|\mathcal A|-|\mathcal A^-|))\) scalars. The backward recursion uses \(O(|\mathcal S|)\) working memory if only adjacent phase-value vectors are retained.
}

\section{BF--BOLA Policy Learning}
\label{app:bf-bola-policy-learning-proofs}
This appendix bounds the residuals in the policy-learning proof in Appendix~\ref{app:policy-learning-proof}. The Bellman targets depend on the training data, so we first establish stability-based concentration. We then control transition and prediction errors along the learned policy, using block-level variance bounds when the block discount is close to one.

\subsection{Proof of Theorem~\ref{thm:bf-bola-policy-sample}}
\label{app:policy-learning-proof}
\paragraph{Current-block cancellation.}
Unlike value estimation, policy learning benefits from observing the current prediction before selecting actions. Comparing policies conditional on the same realized prediction cancels the current-block contribution of prediction-law estimation error. To formalize this, we introduce the auxiliary optimal value \(\overline V:=V^*_{\widehat P,\mu}\) under the comparison model \((\widehat P,\mu)\), which retains the empirical transition kernel and uses the true prediction law. Let \(W:=V^{\widehat\pi^{\BF}}_{P,\mu}\) denote the true value of the learned policy. We first bound \(\overline V-W\) and subsequently control \(V^*-\overline V\).

For any transition model \(Q\), terminal value \(v\), and block rule \(\pi\), let \(\Psi_{Q,v}^{\pi}(s,\boldsymbol\sigma)\) be the conditional expectation of the discounted block reward plus \(\gamma^K v(S_K)\), starting from state \(s\) given prediction \(\boldsymbol\sigma\) and following rule \(\pi\). Covered rows come from \(\boldsymbol\sigma\), and uncovered rows come from \(Q\). Evaluating the learned policy \(\widehat\pi^{\BF}\) with terminal value \(\widehat V\) yields two one-block residuals isolating prediction-law and transition-model errors, respectively:
\begin{align}
b_{\rm pred}(s)
&=(\widehat\mu-\mu)\Psi_{\widehat P,\widehat V}^{\widehat\pi^{\BF}}(s,\cdot),
\label{eq:policy-prediction-residual}\\
b_{\rm tr}(s)
&=\mathbb E_{\boldsymbol\sigma\sim\mu}\!\left[
\Psi_{\widehat P,\widehat V}^{\widehat\pi^{\BF}}(s,\boldsymbol\sigma)
-\Psi_{P,\widehat V}^{\widehat\pi^{\BF}}(s,\boldsymbol\sigma)
\right].
\label{eq:policy-transition-residual}
\end{align}
Let \(T_{Q,\nu}^{\pi}v:=\mathbb E_{\boldsymbol\Sigma\sim\nu}[\Psi_{Q,v}^{\pi}(\cdot,\boldsymbol\Sigma)]\) be the block-evaluation operator with block-end transition matrix \(M_{Q,\nu}^{\pi}\), and abbreviate \(M:=M_{P,\mu}^{\widehat\pi^{\BF}}\). Subtracting the fixed-point identities \(\widehat V=T_{\widehat P,\widehat\mu}^{\widehat\pi^{\BF}}\widehat V\) and \(W=T_{P,\mu}^{\widehat\pi^{\BF}}W\) gives the accumulated error over successive blocks:
\begin{align}
\widehat V-W=(I-\gamma^K M)^{-1}(b_{\rm tr}+b_{\rm pred}).
\label{eq:bf-policy-eval-identity}
\end{align}

Because \(\widehat\pi^{\BF}\) is greedy for \(\widehat V\) under each realized prediction, comparing \eqref{eq:bf-policy-eval-identity} with the optimality equation for \(\overline V\) eliminates the current-block contribution of \(b_{\rm pred}\).

\begin{lemma}[Current-block cancellation in the policy gap]
\label{lem:bf-policy-main-cancellation}
With the notation above, the following inequality holds componentwise:
\[
\overline V-W
\le
\gamma^K\|\widehat V-\overline V\|_\infty\mathbf 1
+(I-\gamma^K M)^{-1}|b_{\rm tr}|
+\bigl[(I-\gamma^K M)^{-1}-I\bigr]|b_{\rm pred}|.
\]
\end{lemma}

\begin{proof}
Write \(\mathcal T^*:=\mathcal T_{\widehat P,\mu}^{\BF}\) and \(\mathcal R:=(I-\gamma^K M)^{-1}\). Since \(\widehat\pi^{\BF}\) maximizes the within-block recursion for every realized prediction with terminal value \(\widehat V\), changing the distribution over current predictions does not alter its greedy rule. Hence \(\mathcal T^*\widehat V=T_{\widehat P,\mu}^{\widehat\pi^{\BF}}\widehat V=\widehat V-b_{\rm pred}\). Combining this identity with \(\mathcal T^*\overline V=\overline V\) and \eqref{eq:bf-policy-eval-identity} yields
\[
\overline V-W
=\mathcal T^*\overline V-\mathcal T^*\widehat V
+\mathcal R b_{\rm tr}+(\mathcal R-I)b_{\rm pred}.
\]
By the \(\gamma^K\)-contraction of \(\mathcal T^*\), the first difference is bounded componentwise by \(\gamma^K\|\widehat V-\overline V\|_\infty\mathbf 1\). Taking absolute values on the remaining terms and using the entrywise nonnegativity of \(\mathcal R\) and \(\mathcal R-I=\gamma^K M\mathcal R\) completes the proof.
\end{proof}

Thus prediction-law error contributes only through future blocks, with the additional discount in \(\gamma^K M\mathcal R\).

\begin{proof}[Proof of Theorem~\ref{thm:bf-bola-policy-sample}]
We use the cancellation lemma to prove the policy bound. When \(\gamma^K<1/2\), errors from later blocks decay quickly. When \(\gamma^K\ge1/2\), we control their accumulated variance instead. Let \(N_1,N_2\) satisfy \eqref{eq:bf-policy-sample-sizes}, with logarithmic factors \(\overline L_{\rm tr},\overline L_{\rm pred}\) defined there. Throughout the proof, we work on the joint concentration event, which has probability at least \(1-\delta\). We write \(E_\pi:=\|\widehat V-W\|_\infty\) for the evaluation error and \(\eta:=\|\widehat V-\overline V\|_\infty\) for the prediction-law value error.

We first consider long blocks, \(\gamma^K<1/2\). Since \(\|(I-\gamma^K M)^{-1}\|_\infty=(1-\gamma^K)^{-1}\le2\), we can bound the evaluation error directly. We combine \eqref{eq:bf-policy-eval-identity} with the transition-residual bound \eqref{eq:bf-policy-transition-final} and move the term proportional to \(E_\pi\) to the left-hand side. This gives \(E_\pi\le C(\varepsilon_{\rm pol}+\|b_{\rm pred}\|_\infty)\). Substituting this bound into Lemma~\ref{lem:bf-policy-main-cancellation}, we obtain
\begin{align}
\|(\overline V-W)_+\|_\infty
\le C\Biggl[\varepsilon_{\rm pol}+\gamma^K\eta
+\left(
\frac{1}{1-\gamma}\sqrt{\frac{\gamma^K\overline L_{\rm tr}}{N_1}}
+\frac{\gamma^K}{1-\gamma^K}
\right)\|b_{\rm pred}\|_\infty\Biggr].
\label{eq:bf-policy-long-master}
\end{align}
We now bound the prediction terms in this inequality. If \(\gamma^K/((1-\gamma)(1-\gamma^K))\le\varepsilon_{\rm pol}\), the discount is strong enough that the deterministic bounds \(\eta,\|b_{\rm pred}\|_\infty\le(1-\gamma)^{-1}\) suffice. The condition on the transition sample size bounds the mixed term by \(C\varepsilon_{\rm pol}\sqrt{\gamma^K/(1-\gamma)}\le C\varepsilon_{\rm pol}\). The remaining prediction terms are also \(O(\varepsilon_{\rm pol})\).

Otherwise, when \(\gamma^K/((1-\gamma)(1-\gamma^K))>\varepsilon_{\rm pol}\), we invoke concentration to sharpen the bound on prediction-law error. Equations~\eqref{eq:bf-policy-eta-bound}--\eqref{eq:bf-policy-bpred-sup} and the \(N_2\) condition give \(\gamma^K\eta+\gamma^K\|b_{\rm pred}\|_\infty/(1-\gamma^K)\le C\varepsilon_{\rm pol}\), including the linear terms. Under the condition defining this case, the mixed term is at most \(C\varepsilon_{\rm pol}^2\sqrt{(1-\gamma)/\gamma^K}\le C\varepsilon_{\rm pol}\). We therefore obtain \(\|(\overline V-W)_+\|_\infty\le C\varepsilon_{\rm pol}\) in both cases.

It remains to compare the auxiliary value with the true optimal value. We apply Lemma~\ref{lem:true-law-transition-bridge} with \(\nu=\mu\) to obtain \(\|V^*-\overline V\|_\infty\le C\varepsilon_{\rm pol}\). Combining the two comparisons gives \(\|V^*-W\|_\infty\le C\varepsilon_{\rm pol}\).

We next consider short blocks, \(\gamma^K\ge1/2\). Errors may persist over many blocks, so we use \eqref{eq:bf-policy-short-master}, which controls their sum through the variances of the block returns. This bound separates the transition and prediction errors from terms proportional to \(E_\pi\). The sample-size conditions bound the first two groups by \(C\varepsilon_{\rm pol}\), using \(1-\gamma^K\ge1-\gamma\) for the linear prediction term. The remaining coefficients come from transition and prediction estimation, respectively:
\[
\begin{aligned}
E_\pi
&\le C\varepsilon_{\rm pol}
+C\left[
\sqrt{\frac{\overline L_{\rm tr}}{N_1(1-\gamma)^2}}
+\frac{1}{1-\gamma^K}\sqrt{\frac{\overline L_{\rm pred}}{N_2}}
\right]E_\pi\\
&\le C\varepsilon_{\rm pol}+\frac12E_\pi.
\end{aligned}
\]
Moving \(E_\pi/2\) to the left-hand side gives \(E_\pi\le C\varepsilon_{\rm pol}\).

To finish this case, we also need to compare the empirical value with the true optimum. Since \(\gamma^{2K}\ge1/4\), the policy and value bounds require the same sample order. In particular, these sample sizes satisfy Theorem~\ref{thm:bf-bola-sample} with \(\alpha=1/2\) and \(\varepsilon_{\rm est}=\varepsilon_{\rm pol}\). We obtain \(\|V^*-\widehat V\|_\infty\le\varepsilon_{\rm pol}\), and the triangle inequality gives \(\|V^*-W\|_\infty\le C\varepsilon_{\rm pol}\).

Both cases hold on the same joint event. We increase the universal sample-size constant to obtain the prescribed accuracy \(\varepsilon_{\rm pol}\). Finally, we count the samples as \(D_1=|\mathcal S|(|\mathcal A|-|\mathcal A^-|)N_1+|\mathcal S||\mathcal A|\) and \(D_2=N_2\), which gives the stated rates.
\end{proof}

\subsection{Supporting Bounds for Theorem~\ref{thm:bf-bola-policy-sample}}
\label{proof_thm_bf_bola_policy_sample}
We use the following sample-size conditions to prove Theorem~\ref{thm:bf-bola-policy-sample}.

Let \(\widehat\mu:=\widehat\mu_{N_2}\), \(\beta:=\gamma^K\), and \(H:=(1-\gamma)^{-1}\). Let \(N_1\) be the number of next-state samples per uncovered transition row and \(N_2\) be the number of prediction samples. The total budgets are \(D_1=|\mathcal U|N_1+|\mathcal S||\mathcal A|\) and \(D_2=N_2\), including one reward observation per state-action pair. If \(\mathcal U=\emptyset\), the transition part below is identically zero.

Collect logarithmic factors into
\[
  \overline L_{\rm tr}
  :=
  \log\!\left(\frac{cK(|\mathcal U|+1)|\mathcal A|(1+N_2)}{\delta}\right)
  \log(2N_1),
  \qquad
  \overline L_{\rm pred}
  :=
  \log\!\left(\frac{c(|\mathcal S|+1)}{\delta}\right)\log(2N_2),
\]
where \(c>1\) is a universal constant.

{
For \(\delta\in(0,1)\) and \(\varepsilon_{\rm pol}\in(0,1]\), we take the sample sizes to satisfy
\begin{equation}
N_1
\ge
\frac{C\overline L_{\rm tr}}{(1-\gamma)^3\varepsilon_{\rm pol}^2},
\qquad
N_2
\ge
C\overline L_{\rm pred}
\left(
1+
\frac{\gamma^{2K}}
{(1-\gamma)^2(1-\gamma^K)\varepsilon_{\rm pol}^2}
\right),
\label{eq:bf-policy-sample-sizes}
\end{equation}
where \(C>0\) is a sufficiently large universal constant. If \(\mathcal U=\emptyset\), we omit the condition on \(N_1\).
}

We use the values \(V^*,\overline V,\widehat V,W\) defined in Appendix~\ref{app:policy-learning-proof}.

To obtain a common concentration event, we apply each bound below with failure probability equal to a sufficiently small fixed fraction of \(\delta\). We include both directions of every two-sided comparison and, for short blocks, the value-estimation bound in Theorem~\ref{thm:bf-bola-sample}. The conditional failure bounds hold uniformly over the conditioning variables. We can therefore average these bounds and apply a union bound to obtain a joint event of probability at least \(1-\delta\). This allocation of failure probabilities changes only the constants in the logarithmic factors.

\subsubsection{Policy-Evaluation Identities}

Equation~\eqref{eq:bf-policy-eval-identity} turns residual bounds into an evaluation-error bound. Lemma~\ref{lem:bf-policy-main-cancellation} gives the sharper policy comparison by cancelling the current-block prediction residual. Both identities hold for every training dataset. We use the one-block payoffs, operators, and residuals defined in Appendix~\ref{app:policy-learning-proof}, and recall \(\eta=\|\widehat V-\overline V\|_\infty\).

\subsubsection{Stability-Based Bernstein Perturbation}

The empirical Bellman target depends on the same samples used to compute its average. We therefore cannot apply Bernstein's inequality as if the target were fixed. Instead, we first bound how much the target changes when one sample is replaced. We use this stability bound to obtain a concentration inequality that still depends on the target's variance.

The weak-interaction inequality of \citet{bousquet2020sharper} supplies the required moment bound. If \(Z_1,\ldots,Z_n\) are independent and \(g_i(Z_1,\ldots,Z_n)\) satisfy \(|\E[g_i\mid Z_i]|\le M\), \(\E[g_i\mid Z_{-i}]=0\), and a bounded-difference condition with constant \(\beta_{\rm wi}\) in all coordinates except \(i\), then for every \(p\ge2\),
\[
  \left\|\sum_{i=1}^n g_i\right\|_{L_p}
  \le
  Cpn\beta_{\rm wi}\log(2n)+CM\sqrt{pn}.
\]
We will decompose the empirical error into an independent average and a remainder. The next lemma uses Bernstein's inequality for the average and the weak-interaction bound for the remainder.

\begin{lemma}[Scalar stable Bernstein]
\label{lem:bf-policy-stable-bernstein}
Let \(X_1,\ldots,X_n\) be i.i.d. from \(P\), let \(D=(X_1,\ldots,X_n)\), and let \(f_D:\mathcal X\to[-B,B]\). Write \(P_n:=n^{-1}\sum_{i=1}^n\delta_{X_i}\) for the empirical distribution, and let \(X\sim P\) be independent of \(D\). Suppose that replacing one observation in \(D\) changes \(f_D\) uniformly by at most \(\kappa\). Then, for every \(t\ge1\),
\[
  \mathbb P\!\left(
  |(P_n-P)f_D|
  >
  C\left[
    \sqrt{\frac{\operatorname{Var}_P(f_D(X))t}{n}}
    +
    \left(\frac{B}{n}+\kappa\log(2n)\right)t
  \right]\right)
  \le 4e^{-t}.
\]
The same statement holds conditionally on external randomness independent of \(X_1,\ldots,X_n\), provided the stability bound holds after conditioning.
\end{lemma}

\begin{proof}
We prove the lemma in three steps. We first separate an independent average from the empirical error. We then bound the remainder using stability and compare the average variance with the variance of the observed target. Let \(D^\circ\) be an independent copy of \(D\) and define \(\bar g(x):=\E_{D^\circ}[f_{D^\circ}(x)-Pf_{D^\circ}]\). We first show that the remainder satisfies
\[
  \left\|(P_n-P)f_D-\frac1n\sum_{i=1}^n\bar g(X_i)\right\|_{L_p}
  \le C\kappa p\log(2n),\qquad p\ge2.
\]
To make a target independent of the sample where we evaluate it, we replace that sample in the training dataset. Let \(X_i'\) be an independent copy of \(X_i\), and let \(D^{(i)}\) be obtained from \(D\) by replacing \(X_i\) by \(X_i'\). We define
\[
  g_i(D):=\E_{X_i'}\!\left[f_{D^{(i)}}(X_i)-Pf_{D^{(i)}}\right].
\]
We can bound the error from this replacement using stability. Both the evaluation term and its mean under \(P\) change by at most \(\kappa\) when \(D\) is replaced by \(D^{(i)}\). Summing over samples gives \(|n(P_n-P)f_D-\sum_i g_i(D)|\le2n\kappa\).

We next separate each replacement term into its conditional mean and a remainder. We write \(m_i(x):=\E[g_i(D)\mid X_i=x]\) and \(h_i(D):=g_i(D)-m_i(X_i)\). To apply the weak-interaction bound, we check its conditions. By construction, \(\E[h_i(D)\mid X_i]=0\). We also have \(\E[h_i(D)\mid X_{-i}]=0\), because conditional on \(X_{-i}\), the target \(f_{D^{(i)}}\) is independent of \(X_i\). If \(j\ne i\), replacing \(X_j\) changes \(f_{D^{(i)}}\) in one coordinate, so both \(g_i\) and \(h_i\) change by at most \(2\kappa\). We can therefore apply the inequality of \citet{bousquet2020sharper} to obtain \(\|\sum_i h_i(D)\|_{L_p}\le Cn\kappa p\log(2n)\).

The remaining terms form the independent average we need. Indeed, \(m_i(x)=\bar g(x)\), because conditional on \(X_i=x\), \(D^{(i)}\) has the same law as \(D^\circ\) and is independent of \(x\). Combining the replacement error and the remainder bound, then dividing by \(n\), proves the first displayed bound.

We now apply Bernstein's moment inequality to the independent average. The summands are centered, with \(|\bar g|\le2B\) and \(P\bar g^2\le \E\operatorname{Var}_P(f_D(X))\). Combining this estimate with the remainder bound gives, for \(p\ge2\),
\[
  \|(P_n-P)f_D\|_{L_p}
  \le
  C\sqrt{\frac{p\,\E\operatorname{Var}_P(f_D(X))}{n}}
  +C\frac{Bp}{n}
  +C\kappa p\log(2n).
\]
The variance in this bound is averaged over training datasets, whereas the lemma uses the variance of the target fitted to the observed dataset. We compare the two using stability of the standard deviation. Let \(S(D):=\sqrt{\operatorname{Var}_P(f_D(X))}\). If \(D,D'\) differ in one observation, then \(|S(D)-S(D')|\le \|f_D-f_{D'}\|_{L_2(P)}\le\|f_D-f_{D'}\|_\infty\le\kappa\). We apply Efron--Stein to obtain \(\operatorname{Var}(S(D))\le n\kappa^2/2\). McDiarmid's inequality also gives an event \(\mathcal E_t\) with probability at least \(1-e^{-t}\) on which \(\E S(D)\le S(D)+\kappa\sqrt{nt/2}\). Combining these two bounds, we obtain on the same event
\[
  \sqrt{\E S(D)^2}
  =
  \sqrt{(\E S(D))^2+\operatorname{Var}(S(D))}
  \le
  S(D)+C\kappa\sqrt{nt},
\]
where we used \(t\ge1\). Since \(\E S(D)^2=\E\operatorname{Var}_P(f_D(X))\), we can now compare the variance in the moment bound with the observed variance \(S(D)^2\).

Finally, we turn the moment bound into a tail bound. We take \(p=t+2\). On \(\mathcal E_t\), the variance comparison shows that the right-hand side of the moment bound is at most
\[
  C\left[
  S(D)\sqrt{\frac{t}{n}}
  +
  \left(\frac{B}{n}+\kappa\log(2n)\right)t
  \right],
\]
after increasing the universal constant. We can therefore apply Markov's inequality to obtain
\[
\begin{aligned}
&\mathbb P\!\left(
|(P_n-P)f_D|
>
C\left[
S(D)\sqrt{\frac{t}{n}}
+
\left(\frac{B}{n}+\kappa\log(2n)\right)t
\right],
\ \mathcal E_t
\right) \\
&\qquad\le
\mathbb P\!\left(|(P_n-P)f_D|>e\|(P_n-P)f_D\|_{L_p}\right)
\le e^{-p}\le e^{-t}.
\end{aligned}
\]
Combining this bound with \(\mathbb P(\mathcal E_t^c)\le e^{-t}\) proves the tail inequality, after enlarging constants. For the conditional version, we first condition on the external randomness. The variables \(X_1,\ldots,X_n\) remain independent, and the same range and stability bounds hold. We can therefore repeat the proof conditionally and then average over the external randomness.
\end{proof}

We next need to average the residual bound over fresh predictions, whose support may be continuous. For each prediction, we assign a score that specifies the tail level at which its concentration bound holds. By bounding the average of these scores, we can average the residual bound without discretizing the prediction space.

\begin{lemma}[Bernstein bound averaged over predictions]
\label{lem:bf-policy-integrated-bernstein}
Let \((\Omega,\mathcal F,\nu)\) be a probability space and let \(\mathcal I\) be finite. For each \(i\in\mathcal I\), let \(Y_{i,1},\ldots,Y_{i,n}\) be i.i.d. from \(P_i\), independent of external randomness \(\mathcal G_i\). Write \(\widehat P_i:=n^{-1}\sum_{\ell=1}^n\delta_{Y_{i,\ell}}\) for their empirical distribution. Suppose \(f_{i,D}(\omega,y)\in[-B,B]\) is jointly measurable and replacing one observation in the sample vector used by index \(i\) changes \(f_{i,D}\) uniformly in \((\omega,y)\) by at most \(\kappa\). Define \(\xi_i(\omega):=(\widehat P_i-P_i)f_{i,D}(\omega,\cdot)\) and \(v_i(\omega):=\operatorname{Var}_{P_i}(f_{i,D}(\omega,Y))\). Then, with probability at least \(1-\delta\), there exist nonnegative measurable scores \(\Lambda_i(\omega)\), \(i\in\mathcal I\), such that, pointwise in \((i,\omega)\),
\[
  |\xi_i(\omega)|
  \le
  C\sqrt{\frac{v_i(\omega)\Lambda_i(\omega)}{n}}
  +
  C\left(\frac{B}{n}+\kappa\log(2n)\right)\Lambda_i(\omega),
\]
and, with \(\Lambda_*(\omega):=\max_i\Lambda_i(\omega)\) and \(L:=\log(C|\mathcal I|/\delta)\), \(\int\Lambda_*\,d\nu\le CL\) and \(\int\Lambda_*^2\,d\nu\le CL^2\).
\end{lemma}

\begin{proof}
We first fix \(i\) and \(\omega\) and apply Lemma~\ref{lem:bf-policy-stable-bernstein} conditional on \(\mathcal G_i\). We define the score as the smallest integer tail level for which the resulting error bound holds:
\[
  \Lambda_i(\omega)
  :=
  \inf\{m\in\mathbb N:\ m\ge1\ \text{and the displayed bound holds with }m
  \text{ in place of }\Lambda_i(\omega)\}.
\]
We can choose these scores measurably. For each integer \(m\), the condition inside the infimum is measurable because \(f_{i,D}\) is jointly measurable. The infimum is over a countable set. The conditional tail bound gives \(\mathbb P(\Lambda_i(\omega)>m\mid\mathcal G_i)\le4e^{-m}\), so \(\E[\exp(c_0\Lambda_i(\omega))\mid\mathcal G_i]\le C\) for a universal \(c_0>0\).

We now integrate this exponential bound over predictions and sum over the finite index set. Fubini's theorem and Markov's inequality give, with probability at least \(1-\delta\),
\[
  \int \sum_{i\in\mathcal I}\exp(c_0\Lambda_i(\omega))\,\nu(d\omega)
  \le \frac{C|\mathcal I|}{\delta}.
\]
Since \(\exp(c_0\Lambda_*)\le\sum_i\exp(c_0\Lambda_i)\), we can use Jensen's inequality to obtain \(\int\Lambda_*\,d\nu\le CL\). To bound the second moment, we first use the same exponential bound to get \(\nu(\Lambda_*>u)\le\min\{1,Ce^{L-c_0u}\}\). Integrating this tail gives \(\int\Lambda_*^2\,d\nu\le CL^2\).
Only the finite index set \(\mathcal I\) enters the logarithmic factor.
\end{proof}

To use the concentration lemmas, we must bound how much the fitted Bellman targets change when one sample is replaced. We first bound the change in the empirical operator and then use contraction to bound the change in its fixed point.

We write \(\widehat J_j(s,\boldsymbol\sigma):=\widehat J_j^{\BF}(s,\boldsymbol\sigma;\widehat V)\), so \(\widehat J_K=\widehat V\). Dataset superscripts indicate which training data are used.

\begin{lemma}[Deterministic replace-one stability]
\label{lem:bf-policy-replace-one-stability}
Fix all prediction samples. Let \(D\) and \(D'\) be two environment-transition datasets that differ in exactly one next-state observation from one uncovered transition row \(z\in\mathcal S\times(\mathcal A\setminus\mathcal A^-)\). Let \(\widehat{\mathcal T}_D\) and \(\widehat{\mathcal T}_{D'}\) be the corresponding empirical BF optimality operators, and let \(\widehat V_D\) and \(\widehat V_{D'}\) be their fixed points. Then \(\|\widehat P_{D,z}-\widehat P_{D',z}\|_1\le2/N_1\), \(\|\widehat V_D-\widehat V_{D'}\|_\infty\le CH^2/N_1\), and every empirical next-state target satisfies
\[
  \sup_{j,\boldsymbol\sigma,y}
  \left|
    \widehat J^D_{j+1}(y,\boldsymbol\sigma)
    -
    \widehat J^{D'}_{j+1}(y,\boldsymbol\sigma)
  \right|
  \le
  \frac{CH^2}{N_1}.
\]
If two prediction datasets differ in exactly one prediction while all environment samples are fixed, then the empirical fixed points differ by at most \(CH/[N_2(1-\gamma^K)]\), and every one-block target used in the prediction residual changes uniformly by at most \(C\gamma^K H/[N_2(1-\gamma^K)]\).
\end{lemma}

\begin{proof}
We first replace one transition sample. This moves probability mass \(1/N_1\) from one next state to another, giving the row bound. Fix any terminal value \(0\le V\le H\mathbf 1\). For a given prediction, the change in a one-step Bayesian-value term is at most \(2\gamma H/N_1\). We add these changes over the \(K\)-stage backward recursion to obtain
\[
  \|\widehat{\mathcal T}_D V-\widehat{\mathcal T}_{D'}V\|_\infty
  \le
  \frac{2\gamma H}{N_1}\sum_{j=0}^{K-1}\gamma^j
  \le
  \frac{2H^2(1-\gamma^K)}{N_1}.
\]
We divide by the contraction gap \(1-\gamma^K\) to obtain the fixed-point bound. We then use this bound at the end of the block and apply backward induction to bound the next-state targets. The maximizing action may change, but this does not increase the bound because maximization is nonexpansive.

We next replace one prediction sample. This changes one term in the empirical prediction average, and every one-block target is bounded by \(H\). Hence \(\|\widehat{\mathcal T}_D V-\widehat{\mathcal T}_{D'}V\|_\infty\le H/N_2\). Dividing by the contraction gap gives the fixed-point bound \(CH/[N_2(1-\gamma^K)]\). Since a one-block target is \(\gamma^K\)-Lipschitz in its terminal value, we obtain the stated target-stability bound.
\end{proof}

\subsubsection{Transition-Error Propagation}

We now bound the accumulated transition error along the learned policy. We first bound each one-step error and then add these errors over time. We keep track of the discount on the terminal-value error because it is needed for the long-block policy bound.

Condition on the prediction samples. For each row \(z=(s,a)\in\mathcal U\) and phase \(j\in\{0,\ldots,K-1\}\), consider the empirical next-state target \(\widehat J_{j+1}(\cdot,\boldsymbol\sigma)\) at a fresh prediction \(\boldsymbol\sigma\).

These targets depend on the transition data. Lemma~\ref{lem:bf-policy-replace-one-stability} gives replace-one stability constant \(CH^2/N_1\) with respect to the sample vector from the selected uncovered transition row. Lemma~\ref{lem:bf-policy-integrated-bernstein}, applied with \(\mathcal I=\mathcal U\times\{0,\ldots,K-1\}\), \(\nu=\mu\), \(B=H\), and \(\kappa=CH^2/N_1\), therefore gives an event of failure probability at most a fixed fraction of \(\delta\) on which
\begin{align}
  \left|(\widehat P_z-P_z)\widehat J_{j+1}(\cdot,\boldsymbol\sigma)\right|
  \le
  C\sqrt{\frac{v_{zj}(\boldsymbol\sigma)\Lambda_{zj}(\boldsymbol\sigma)}{N_1}}
  +
  C\frac{H^2\log(2N_1)}{N_1}\Lambda_{zj}(\boldsymbol\sigma),
  \label{eq:bf-policy-transition-score}
\end{align}
where \(v_{zj}(\boldsymbol\sigma):=\operatorname{Var}_{P_z}(\widehat J_{j+1}(Y,\boldsymbol\sigma))\), and the score maximum \(\Lambda_*:=\max_{z,j}\Lambda_{zj}\) satisfies \(\int\Lambda_*\,d\mu\le C\log(CK|\mathcal U|/\delta)\) and \(\int\Lambda_*^2\,d\mu\le C\log^2(CK|\mathcal U|/\delta)\).

To see how the row errors accumulate, we run \(\widehat\pi^{\BF}\) under the true model and prediction law. At time \(t=bK+j\), let \(\boldsymbol\Sigma_b\sim\mu\) be the fresh prediction for block \(b\). If the chosen action lies outside \(\mathcal A^-\), we define its local error and score by
\[
  d_t
  :=
  \gamma(\widehat P_{Z_t}-P_{Z_t})
  \widehat J_{j+1}(\cdot,\boldsymbol\Sigma_b),
  \qquad
  \lambda_t:=\Lambda_{Z_tj}(\boldsymbol\Sigma_b),
\]
where \(Z_t=(S_t,A_t)\). If the chosen action lies in \(\mathcal A^-\), set \(d_t=\lambda_t=0\). The selected score is bounded by the maximum row--phase score at the current prediction. Since each block's prediction has law \(\mu\) independently of the block-start state by Assumption~\ref{assump:prediction-law}, for \(q=1,2\),
\begin{align}
  \sup_{s_0}\sum_{t\ge0}\gamma^t\E_{s_0}\lambda_t^q
  \le
  \frac{1}{1-\gamma}\int\Lambda_*^q\,d\mu .
  \label{eq:bf-policy-score-budget}
\end{align}
The pointwise maximum thus controls every row choice, even when that choice depends on the training data and current prediction.

We will need a mean-square bound on these residuals to compare true and empirical variances. We measure their size using the discounted \(L_2\) norm \(\|d\|_{\rm occ,2}:=\sup_{s_0}((1-\gamma)\sum_{t\ge0}\gamma^t\E_{s_0}d_t^2)^{1/2}\). Combining \eqref{eq:bf-policy-transition-score} with the score budget and \(v_{zj}\le H^2\), we obtain
\begin{align}
  \|d\|_{\rm occ,2}
  \le
  C H\sqrt{\frac{\overline L_{\rm tr}}{N_1}}
  +
  C\frac{H^2\overline L_{\rm tr}}{N_1}.
  \label{eq:bf-policy-d-rms}
\end{align}

We next compare the empirical and true targets for the same learned rule. Their difference has two sources: local transition errors within the block and the error in the terminal Bayesian value. Let \(U_j(\cdot,\boldsymbol\sigma)\) be the true-model within-block value of \(\widehat\pi^{\BF}\), with \(U_K=W\). For a fixed prediction \(\boldsymbol\sigma\), we write \(P_j^{\widehat\pi^{\BF}}(s'\mid s,\boldsymbol\sigma):=\bar P_j(s'\mid s,\widehat a_j(s,\boldsymbol\sigma),\boldsymbol\sigma)\) for its phase-\(j\) transition matrix. Let \(d_j(s,\boldsymbol\sigma)\) be the local residual \(d_t\) at state \(s\), phase \(j\), and prediction \(\boldsymbol\sigma\).

We subtract the two fixed-policy recursions and write \(e_j:=\widehat J_j-U_j\). The difference satisfies \(e_j=d_j+\gamma P_j^{\widehat\pi^{\BF}}e_{j+1}\) for \(j<K-1\) and \(e_{K-1}=d_{K-1}+\gamma P_{K-1}^{\widehat\pi^{\BF}}(\widehat V-W)\). At time \(t=bK+j\), we write \(U_{t,+}:=U_{j+1}(\cdot,\boldsymbol\Sigma_b)\) and \(\widehat U_{t,+}:=\widehat J_{j+1}(\cdot,\boldsymbol\Sigma_b)\). We denote the terminal-value error by \(E_\pi:=\|\widehat V-W\|_\infty\).

We can now unroll the recursion to separate these two sources of error. At phase \(j\), the terminal part is bounded by \(\gamma^{K-j}E_\pi\). In the discounted \(L_2\) norm, its contribution is at most \(\sqrt{\gamma^K}E_\pi\). To verify this, we use conditional Jensen's inequality and sum the squared discounted terminal bounds:

\[
  (1-\gamma)\sum_{b\ge0}\sum_{j=0}^{K-1}
  \gamma^{bK+j}\gamma^{2(K-j)}E_\pi^2
  \le \gamma^K E_\pi^2 .
\]
We combine this terminal bound with the sum of the local residuals using Minkowski's inequality. This gives
\begin{align}
  \|e\|_{\rm occ,2}
  \le
  \sqrt{\gamma^K}E_\pi+\frac{\|d\|_{\rm occ,2}}{1-\sqrt\gamma}
  \le
  \sqrt{\gamma^K}E_\pi+2H\|d\|_{\rm occ,2}.
  \label{eq:bf-policy-simulation-error}
\end{align}

For the Markov reward process at individual time steps induced by \(\widehat\pi^{\BF}\), the total-variance identity controls the accumulated true-target variance:
\begin{align}
  \sup_{s_0}\sum_{t\ge0}\gamma^t
  \E_{s_0}\!\left[\mathbf 1\{A_t\notin\mathcal A^-\}
  \gamma^2\operatorname{Var}_{P_{Z_t}}\!\left(U_{t,+}(Y)\right)\right]
  \le
  C H^2.
  \label{eq:bf-policy-true-var-budget}
\end{align}
The indicator restricts the sum to uncovered rows, where the conditional transition law is \(P_{Z_t}\). By conditional variance, the displayed discounted row variance is at most the full one-step Bellman-target variance. We apply \eqref{eq:discounted-return-variance} to the Markov reward process on states \(x=(s,\boldsymbol\sigma,j)\), with discount \(\gamma\), to obtain the stated bound.

We use \eqref{eq:bf-policy-simulation-error} to compare the empirical targets with the true targets. Their variance difference is bounded by the local residuals and the terminal-value error, giving
\begin{align}
  \sup_{s_0}\sum_{t\ge0}\gamma^t
  \E_{s_0}\!\left[\mathbf 1\{A_t\notin\mathcal A^-\}
  \gamma^2\operatorname{Var}_{P_{Z_t}}\!\left(\widehat U_{t,+}(Y)\right)\right]
  \le
  C H^2+C H\gamma^K E_\pi^2+C H^3\|d\|_{\rm occ,2}^2.
  \label{eq:bf-policy-emp-var-budget}
\end{align}
It remains to bound the accumulated transition residual. With terminal value \(\widehat V\), we can write each one-block transition error as a sum of local errors. This gives
\[
  \left\|(I-\gamma^K M_{P,\mu}^{\widehat\pi^{\BF}})^{-1}|b_{\rm tr}|\right\|_\infty
  \le
  \sup_{s_0}\sum_{t\ge0}\gamma^t\E_{s_0}|d_t|.
\]
We bound this sum by applying Cauchy--Schwarz to \eqref{eq:bf-policy-transition-score}. We then use the score budget and the variance bound \eqref{eq:bf-policy-emp-var-budget}, and substitute the residual bound \eqref{eq:bf-policy-d-rms}. We obtain
\begin{align}
  \left\|(I-\gamma^K M_{P,\mu}^{\widehat\pi^{\BF}})^{-1}|b_{\rm tr}|\right\|_\infty
  \le
  C\sqrt{\frac{H^3\overline L_{\rm tr}}{N_1}}
  +
  C\frac{H^3\overline L_{\rm tr}}{N_1}
  +
  C H\sqrt{\frac{\gamma^K\overline L_{\rm tr}}{N_1}}E_\pi .
  \label{eq:bf-policy-transition-final}
\end{align}
The leading term has order \(H^{3/2}/\sqrt{N_1}\) on uncovered rows. The final term contains the discounted terminal error. We move this term to the left-hand side when combining the bounds in the main proof.

\subsubsection{Transition Perturbation under a Fixed Prediction Law}
\label{subsec:bf-policy-true-law-comparison}

The policy bound compares the learned policy with \(\overline V\), rather than directly with \(V^*\). We must therefore also bound the difference between these two optimal values. They use the same true prediction law \(\mu\), but different transition models. We first extend the residual bound to any fixed prediction law, including laws with continuous support. We then compare the optimal values by evaluating each model's optimal rule under the other model.

\begin{lemma}[Transition-residual bound under a fixed prediction law]
\label{lem:fixed-law-transition-template}
Fix a prediction law \(\nu\) on \(\mathcal Q_K\) that is independent of the environment transition samples. Let \(\pi_D\) be either a block rule that is greedy for the empirical value at every prediction or a block rule fixed independently of the environment row samples. Let \(\widehat W_\nu\) and \(W_\nu\) be the values of this same rule under \((\widehat P,\nu)\) and \((P,\nu)\), respectively, and set \(E_\nu:=\|\widehat W_\nu-W_\nu\|_\infty\). In this lemma, \(\widehat J_j\) denotes the empirical within-block value of \(\pi_D\) with terminal value \(\widehat W_\nu\). Define
\[
  b_\nu(s):=\int
  \left[
    \Psi_{\widehat P,\widehat W_\nu}^{\pi_D}(s,\boldsymbol\sigma)
    -
    \Psi_{P,\widehat W_\nu}^{\pi_D}(s,\boldsymbol\sigma)
  \right]\nu(d\boldsymbol\sigma),
  \qquad
  M_\nu:=M_{P,\nu}^{\pi_D}.
\]
If replacing one observation in one uncovered transition row changes every \(\widehat J_{j+1}\) uniformly by at most \(CH^2/N_1\), then, with probability at least \(1-\delta/16\),
\begin{align}
  \left\|(I-\gamma^K M_\nu)^{-1}|b_\nu|\right\|_\infty
  \le
  C\sqrt{\frac{H^3 L_\nu}{N_1}}
  +
  C\frac{H^3 L_\nu}{N_1}
  +
  C H\sqrt{\frac{\gamma^K L_\nu}{N_1}}\,E_\nu,
  \label{eq:fixed-law-template}
\end{align}
where \(L_\nu:=\log(cK(|\mathcal U|+1)|\mathcal A|/\delta)\log(2N_1)\). The same statement is trivial when \(\mathcal U=\emptyset\), with \(b_\nu\equiv0\).
\end{lemma}

\begin{proof}
We apply the transition-error calculation above with the fixed law \(\nu\). Lemma~\ref{lem:bf-policy-integrated-bernstein} applies with index set \(\mathcal U\times\{0,\ldots,K-1\}\), range \(H\), and stability constant \(CH^2/N_1\). For an empirical optimal rule, the stability condition follows from Lemma~\ref{lem:bf-policy-replace-one-stability} with \(\nu\) in place of \(\widehat\mu\). For a rule independent of the row samples, the same fixed-policy contraction argument proves it. Thus the row-error and score bounds in \eqref{eq:bf-policy-transition-score}--\eqref{eq:bf-policy-d-rms} hold with predictions averaged under \(\nu\) and logarithmic factor \(L_\nu\).

We compute the discounted norm under \(\pi_D\) and \(\nu\). The terminal values are now \(\widehat W_\nu\) and \(W_\nu\), so their difference in the within-block comparison \eqref{eq:bf-policy-simulation-error} is \(E_\nu\). The return-variance bound \eqref{eq:bf-policy-true-var-budget} applies to this same rule, and \eqref{eq:bf-policy-emp-var-budget} then bounds the empirical-target variance. We can therefore repeat the Cauchy--Schwarz calculation leading to \eqref{eq:bf-policy-transition-final}, with \(E_\nu\) in place of \(E_\pi\), to obtain \eqref{eq:fixed-law-template}. These steps only average over the fixed law and do not require its support to be finite.
\end{proof}

{
\begin{lemma}[Optimal-value perturbation under a fixed prediction law]
\label{lem:true-law-transition-bridge}
Fix any prediction law \(\nu\) on \(\mathcal Q_K\) that is independent of the environment transition samples. The law may have arbitrary support. Let \(V_\nu\) and \(\widehat V_\nu\) be the optimal BF values under \((P,\nu)\) and \((\widehat P,\nu)\), respectively, and use \(L_\nu\) from Lemma~\ref{lem:fixed-law-transition-template}. If \(N_1\ge C H^2L_\nu\), then, with probability at least \(1-\delta/8\),
\begin{align}
  \|\widehat V_\nu-V_\nu\|_\infty
  \le
  C\sqrt{\frac{H^3L_\nu}{N_1}}
  +
  C\frac{H^3L_\nu}{N_1}.
  \label{eq:true-law-transition-bridge}
\end{align}
\end{lemma}
}

\begin{proof}
We prove the bound by comparing the value of the same rule under both models. For the upper deviation, we use an empirical optimal rule. Let \(\widehat\pi_\nu\) be a measurable pointwise-greedy rule for \(\widehat V_\nu\) under \((\widehat P,\nu)\), and let \(W_+:=V^{\widehat\pi_\nu}_{P,\nu}\) be its true value. The fixed-policy evaluation identity gives \(\widehat V_\nu-W_+=(I-\gamma^K M_{P,\nu}^{\widehat\pi_\nu})^{-1}b_+\), where \(b_+\) is the transition residual in Lemma~\ref{lem:fixed-law-transition-template} with terminal value \(\widehat V_\nu\). We apply that lemma to obtain \(\|\widehat V_\nu-W_+\|_\infty\le A_{\rm law}+c_{\rm law}\|\widehat V_\nu-W_+\|_\infty\), where \(A_{\rm law}:=C\sqrt{H^3L_\nu/N_1}+CH^3L_\nu/N_1\) and \(c_{\rm law}:=CH\sqrt{\gamma^K L_\nu/N_1}\).

We can now solve this inequality for the value difference. The condition \(N_1\ge CH^2L_\nu\) ensures \(c_{\rm law}\le1/4\), so \(\|\widehat V_\nu-W_+\|_\infty\le2A_{\rm law}\). Since \(V_\nu\ge W_+\), this also bounds \(\|(\widehat V_\nu-V_\nu)_+\|_\infty\).

For the reverse deviation, we evaluate a true optimal rule in the empirical model. Let \(\pi^*_\nu\) be an optimal pointwise rule under \((P,\nu)\), and let \(W_-:=V^{\pi^*_\nu}_{\widehat P,\nu}\). Since \(\widehat V_\nu\ge W_-\), it is enough to bound \(\|V_\nu-W_-\|_\infty\). The rule \(\pi^*_\nu\) is independent of the empirical row samples. Its empirical evaluation value has the same \(CH^2/N_1\) replace-one stability, by contraction with the rule held fixed. We therefore apply Lemma~\ref{lem:fixed-law-transition-template} and solve the resulting inequality as above to obtain \(\|W_--V_\nu\|_\infty\le2A_{\rm law}\). Combining the upper and lower deviations proves \eqref{eq:true-law-transition-bridge}.
\end{proof}

\subsubsection{Prediction Terms After Cancellation}

After cancellation, we still need to bound two prediction terms: the terminal-value error and the prediction residual. We can bound the first using a target independent of the prediction samples. The second uses the learned value, so we need the stability bound.

We first condition on all transition samples and compare \(\widehat V\) with \(\overline V\). For each \(s\), the Bellman target \(\boldsymbol\sigma\mapsto J^{\BF}_{\widehat P,0}(s,\boldsymbol\sigma;\overline V)\) is independent of the prediction samples and bounded by \(H\). We apply Bernstein's inequality and take a union bound over \(s\in\mathcal S\). We then use the \(\gamma^K\)-contraction of the BF operator to obtain
\begin{align}
  \eta
  \le
  \min\left\{
    H,\,
    C\frac{H}{1-\gamma^K}
    \left(
      \sqrt{\frac{\overline L_{\rm pred}}{N_2}}
      +
      \frac{\overline L_{\rm pred}}{N_2}
    \right)
  \right\}.
  \label{eq:bf-policy-eta-bound}
\end{align}
For the residual \(b_{\rm pred}\), the targets depend on the prediction samples through \(\widehat V\), so we use Lemma~\ref{lem:bf-policy-stable-bernstein}. Replacing one prediction changes \(\widehat V\) by at most \(CH/[N_2(1-\gamma^K)]\), and the one-block target is \(\gamma^K\)-Lipschitz in its terminal value. Thus, with probability at least \(1-\delta\),
\begin{align}
  \|b_{\rm pred}\|_\infty
  \le
  C H\sqrt{\frac{\overline L_{\rm pred}}{N_2}}
  +
  C\frac{H\overline L_{\rm pred}}{N_2(1-\gamma^K)}.
  \label{eq:bf-policy-bpred-sup}
\end{align}
The deterministic bound \(\|b_{\rm pred}\|_\infty\le H\) also holds. Cancellation leaves only the propagation of this residual to future blocks. Since \(M_{P,\mu}^{\widehat\pi^{\BF}}\) is stochastic,
\begin{align}
  \left\|
    \left((I-\gamma^K M_{P,\mu}^{\widehat\pi^{\BF}})^{-1}-I\right)b_{\rm pred}
  \right\|_\infty
  \le
  \frac{\gamma^K}{1-\gamma^K}\|b_{\rm pred}\|_\infty .
  \label{eq:bf-policy-future-bpred}
\end{align}

\subsubsection{Block-Level Variance Bounds}
When the block discount is close to one, bounding every residual by its largest possible size would lose the desired horizon dependence. We therefore bound the accumulated variance under the learned policy. We first do this for the true target and then compare it with the empirical target.

\begin{lemma}[Block-level variance bound]
\label{lem:bf-policy-macro-total-variance}
Condition on the complete training dataset and fix the learned BF rule \(\widehat\pi^{\BF}\). Let \(M:=M_{P,\mu}^{\widehat\pi^{\BF}}\), let \(F_s(\boldsymbol\sigma):=\Psi_{P,W}^{\widehat\pi^{\BF}}(s,\boldsymbol\sigma)\), and set \(v(s):=\operatorname{Var}_\mu(F_s(\boldsymbol\Sigma))\). Then
\[
  \left\|(I-\gamma^K M)^{-1}v\right\|_\infty
  \le C H^2,
  \qquad
  \left\|(I-\gamma^K M)^{-1}\sqrt v\right\|_\infty
  \le
  \frac{C H}{\sqrt{1-\gamma^K}}.
\]
\end{lemma}

\begin{proof}
Conditional on the training data, the learned rule induces a Markov reward process at block boundaries. We apply the multistep variance identity to this process (see \citet{merlis2024lookahead} and Lemma~13 of \citet{zurek2024span}). Its return lies in \([0,H]\). The function \(F_s\) is the conditional mean of the one-block Bellman target given the prediction, so its variance is at most the full target variance. Both conclusions now follow from \eqref{eq:discounted-return-variance} with discount \(\beta=\gamma^K\).
\end{proof}

The preceding lemma bounds the variance of the true one-block target, but concentration uses the empirical target. To connect them, we bound their difference using the within-block value comparison.

\begin{lemma}[Block-level variance transfer]
\label{lem:bf-policy-macro-envelope-transfer}
Let \(\widehat F_s(\boldsymbol\sigma):=\Psi_{\widehat P,\widehat V}^{\widehat\pi^{\BF}}(s,\boldsymbol\sigma)\), \(F_s(\boldsymbol\sigma):=\Psi_{P,W}^{\widehat\pi^{\BF}}(s,\boldsymbol\sigma)\), and \(\rho(s):=\{\E_{\boldsymbol\Sigma\sim\mu}[(\widehat F_s(\boldsymbol\Sigma)-F_s(\boldsymbol\Sigma))^2]\}^{1/2}\). Then
\[
  \sqrt{\operatorname{Var}_\mu(\widehat F_s(\boldsymbol\Sigma))}
  \le
  \sqrt{v(s)}+\rho(s),
  \qquad
  \left\|(I-\gamma^K M)^{-1}\rho\right\|_\infty
  \le
  \frac{C H}{\sqrt{1-\gamma^K}}\|d\|_{\rm occ,2}
  +
  \frac{C\gamma^K}{1-\gamma^K}E_\pi .
\]
Combining the two variance-transfer bounds gives
\[
  \left\|(I-\gamma^K M)^{-1}
  \sqrt{\operatorname{Var}_\mu(\widehat F(\boldsymbol\Sigma))}\right\|_\infty
  \le
  \frac{C H}{\sqrt{1-\gamma^K}}
  +
  \frac{C H}{\sqrt{1-\gamma^K}}\|d\|_{\rm occ,2}
  +
  \frac{C\gamma^K}{1-\gamma^K}E_\pi .
\]
\end{lemma}

\begin{proof}
We first apply Minkowski's inequality in \(L_2(\mu)\) to compare the two standard deviations. This proves the first inequality. To bound the accumulated target difference, we apply Cauchy--Schwarz and obtain
\[
  [(I-\gamma^K M)^{-1}\rho](s_0)
  \le
  (1-\gamma^K)^{-1/2}
  \left(
    \sum_{b\ge0}\gamma^{Kb}\E_{s_0}[\rho(S_{bK})^2]
  \right)^{1/2}.
\]
We can bound the target difference inside this sum by local residuals and terminal-value error. Unrolling the fixed-policy recursions under the true transition model gives
\[
\left|\widehat F_{S_{bK}}(\boldsymbol\Sigma_b)-F_{S_{bK}}(\boldsymbol\Sigma_b)\right|
\le
\mathbb E_{P}^{\widehat\pi^{\BF}}\!\left[
\sum_{j=0}^{K-1}\gamma^j|d_{bK+j}|
\,\middle|\,S_{bK},\boldsymbol\Sigma_b
\right]+\gamma^K E_\pi.
\]
We apply conditional Jensen's inequality and use \((\sum_j\gamma^j|d_{bK+j}|)^2\le H\sum_j\gamma^j d_{bK+j}^2\) to obtain
\[
\sum_{b\ge0}\gamma^{Kb}\mathbb E_{s_0}\rho(S_{bK})^2
\le C H^2\|d\|_{\rm occ,2}^2
+\frac{C\gamma^{2K}}{1-\gamma^K}E_\pi^2.
\]
Substituting this estimate into the preceding resolvent inequality gives the stated bound. We then combine it with Lemma~\ref{lem:bf-policy-macro-total-variance} to obtain the final variance bound.
\end{proof}

Combining this variance comparison with stable Bernstein concentration controls the prediction residual under the same true-model resolvent used in the policy-evaluation identity.

\begin{lemma}[Short-block prediction residual]
\label{lem:bf-policy-short-pred-residual}
Suppose \(\gamma^K\ge1/2\). On the prediction stability-Bernstein event,
\begin{align}
  \left\|(I-\gamma^K M)^{-1}|b_{\rm pred}|\right\|_\infty
  \le&
  C H\sqrt{\frac{\overline L_{\rm pred}}{N_2(1-\gamma^K)}}
  +
  C H\|d\|_{\rm occ,2}
  \sqrt{\frac{\overline L_{\rm pred}}{N_2(1-\gamma^K)}}
  \nonumber\\
  &+
  C\frac{H\overline L_{\rm pred}}{N_2(1-\gamma^K)^2}
  +
  C\frac{\gamma^K}{1-\gamma^K}
  \sqrt{\frac{\overline L_{\rm pred}}{N_2}}\,E_\pi .
  \label{eq:bf-policy-short-pred-residual}
\end{align}
\end{lemma}

\begin{proof}
We apply Lemma~\ref{lem:bf-policy-stable-bernstein} to \(b_{\rm pred}(s)=(\widehat\mu-\mu)\widehat F_s\) at each state. The target is bounded by \(H\), and Lemma~\ref{lem:bf-policy-replace-one-stability} gives prediction-sample stability \(C\gamma^K H/[N_2(1-\gamma^K)]\). We then sum the resulting bounds using the nonnegative resolvent. For the square-root variance term, Lemma~\ref{lem:bf-policy-macro-envelope-transfer} gives the first, second, and fourth terms in \eqref{eq:bf-policy-short-pred-residual}. For the Bernstein range term and stability remainder, we use \(\|(I-\gamma^K M)^{-1}\|_\infty=(1-\gamma^K)^{-1}\) to obtain the third term.
\end{proof}

\subsubsection{Short-Block Regime: Combining the Residual Bounds}

We now combine the residual bounds for short blocks, \(\gamma^K\ge1/2\). In this case, we need the same prediction-sample order as for Bayesian value estimation. To bound the evaluation error, we substitute the transition bound \eqref{eq:bf-policy-transition-final} and Lemma~\ref{lem:bf-policy-short-pred-residual} into the evaluation identity \eqref{eq:bf-policy-eval-identity}. Under \eqref{eq:bf-policy-sample-sizes}, we use \eqref{eq:bf-policy-d-rms} to bound the mixed prediction--transition term by \(O(\varepsilon_{\rm pol}^2)\). This gives

\begin{align}
  E_\pi
  \le&
  C\sqrt{\frac{\overline L_{\rm tr}}{N_1(1-\gamma)^3}}
  +
  C\frac{\overline L_{\rm tr}}{N_1(1-\gamma)^3}
  +
  C\frac{1}{1-\gamma}
  \sqrt{\frac{\overline L_{\rm pred}}{N_2(1-\gamma^K)}}
  \nonumber\\
  &+
  C\frac{\overline L_{\rm pred}}{N_2(1-\gamma)(1-\gamma^K)^2}
  +
  C\left[
    \sqrt{\frac{\overline L_{\rm tr}}{N_1(1-\gamma)^2}}
    +
    \frac{1}{1-\gamma^K}\sqrt{\frac{\overline L_{\rm pred}}{N_2}}
  \right]E_\pi .
  \label{eq:bf-policy-short-master}
\end{align}

{
The sample-size conditions make the coefficient of \(E_\pi\) at most \(1/2\). We move this term to the left-hand side and obtain \(E_\pi\le C\varepsilon_{\rm pol}\), as required in the main proof. The stability bounds allow the targets to depend on the training data, and taking the maximum score over rows and phases allows the learned policy to choose which row to use.
}

\section{Lower Bounds for Learning}
\label{app:learning-lower-bound-proofs}
The constructions isolate the two data sources: unknown uncovered rows for the environment-sample bound, and unknown informativeness of a calibrated signal for the prediction-sample bound. We verify the value gaps, KL bounds, and parameter ranges used in the testing reductions of Appendix~\ref{subsec:bf-lower-bound-proof-ideas}.

\subsection{Testing Reduction and Main Proofs}
\label{subsec:bf-lower-bound-proof-ideas}
We prove each lower bound by constructing two nearby models that require different answers from an accurate learner. We then use the learner's answer to distinguish the models. To bound one type of sample, we make the other type have the same distribution under both models.

We use the standard two-point testing bound \citep{tsybakov2009introduction}. Let \(\mathbb P_0,\mathbb P_1\) be the distributions of the learner's observations and output under the two models. If a test distinguishes them with error at most \(1/3\) under each law, their total variation distance must be at least \(1/3\). By Pinsker's inequality, this requires \(\operatorname{KL}(\mathbb P_0\Vert\mathbb P_1)\ge2/9\). We write \(\operatorname{kl}(u\Vert v)\) for the KL divergence between Bernoulli laws with parameters \(u,v\).

\begin{proof}[Proof of Theorem~\ref{thm:bf-bola-policy-lower-bound}]
\emph{Environment samples.} We construct \(n\) gadgets with \(3n+2\) states and the same known deterministic prediction, as specified in Appendix~\ref{proof_bf_bola_policy_lower_bound}. In gadget \(i\), a decision state \(d_i\) offers a choice between a risky state \(x_i\), with an unknown transition row for each uncovered action, and a reference state of known value. We choose the reference value so that improving any one of these rows changes the better branch.

Let \(m:=|\mathcal A\setminus\mathcal A^-|\). In the baseline model \(\theta^-\), all \(nm\) unknown rows have success probability \(p-\Delta\), where \(p=1-\gamma\) and \(\Delta=162\varepsilon_{\rm pol}(1-\gamma)^2\). In the alternative model \(\theta^{(i,a)}\), only the row \((x_i,a)\) changes, to success probability \(p+\Delta\). Setting the reference value midway between the two optimal values at \(x_i\) gives an action gap of at least \(8\varepsilon_{\rm pol}\) at \(d_i\) under both models. An \(\varepsilon_{\rm pol}\)-optimal policy must therefore choose the risky branch with probability at most \(1/8\) under \(\theta^-\) and at least \(7/8\) under \(\theta^{(i,a)}\).

We distinguish these models by checking whether the returned policy's risky-branch probability at \(d_i\) exceeds \(1/2\). The policy guarantee makes this test's error at most \(1/3\). Let \(\mathbb P_\theta\) be the training-transcript law, including the returned policy, and let \(N_{i,a}\) count queries to \((x_i,a)\). Since all other rows, rewards, and predictions agree, the adaptive KL chain rule gives
\[
\frac{2}{9}
\le \operatorname{KL}(\mathbb P_{\theta^-}\Vert\mathbb P_{\theta^{(i,a)}})
=\mathbb E_{\theta^-}[N_{i,a}]\operatorname{kl}(p-\Delta\Vert p+\Delta)
\le C\mathbb E_{\theta^-}[N_{i,a}]\varepsilon_{\rm pol}^2(1-\gamma)^3.
\]
We sum over the \(m\) uncovered actions in each of the \(n\) gadgets under the same baseline law. This gives \(D_1\ge c nm/((1-\gamma)^3\varepsilon_{\rm pol}^2)\), with \(n=\Theta(|\mathcal S|)\).

\emph{Prediction samples.} The two models have the same transition probabilities and rewards, but their predictions reveal the next state with different probabilities. Only prediction samples can therefore distinguish them. We assume the physical kernel and rewards are known to the learner and set \(\mathcal A^- = \mathcal A\). For a risky transition with success indicator \(Z\sim\operatorname{Bernoulli}(q)\), we independently reveal \(Z\) with probability \(r\) and otherwise report its prior mean \(q\). The reported success probability \(\Sigma\) takes values \(1,0,q\) with probabilities \(rq,r(1-q),1-r\). Thus \(\mathbb P(Z=1\mid\Sigma)=\Sigma\), while the physical success probability remains \(q\). We use this prediction rule only for the risky row at \(x\) in the first phase. All other prediction components equal their base rows. Since only the revelation probability changes between the models, the KL divergence per prediction sample is exactly \(\operatorname{kl}(r_0\Vert r_1)\).

We need the value at \(x\) to change enough when \(r\) changes. Write this value as \(v(r)\). For \(\gamma^K\le3/4\), a single choice between a risky action and a safe action suffices. For \(\gamma^K>3/4\), we instead allow the decision maker to wait one block and return to \(x\). Returning gives another chance to use a prediction, so changes in \(r\) affect rewards over many blocks. Around the chosen parameters, \(v'(r)\) has order \((1-\gamma)^{-1}\) in the first construction and \(((1-\gamma)(1-\gamma^K))^{-1}\) in the second. Appendix~\ref{proof_bf_bola_policy_lower_bound} gives the transitions and verifies these derivatives.

We then make the learner choose between this uncertain future value and a known reference. We add a decision state \(d\) with two zero-reward paths of length \(K\), one to \(x\) and one to a reference state of value \(w=(v(r_0)+v(r_1))/2\). Their optimal values are \(\gamma^K v(r_i)\) and \(\gamma^K w\), so different branches are optimal in the two models. The value difference is discounted by one block before it affects the initial choice.

We choose the parameters so that the discounted decision gap is large enough to test the models, but the prediction laws remain close. The parameter calculations in that appendix give
\begin{equation}
\gamma^K\bigl(v(r_1)-v(r_0)\bigr)\ge32\varepsilon_{\rm pol},
\quad \min_{i=0,1}(1-r_i)\ge\frac14,
\quad
\operatorname{kl}(r_0\Vert r_1)
\le C\frac{(1-\gamma^K)(1-\gamma)^2\varepsilon_{\rm pol}^2}{\gamma^{2K}}.
\label{eq:bf-policy-lower-separation}
\end{equation}
To test the models using the returned policy, we evaluate it at the same prediction under both laws. On the common prediction \(\Sigma=q\), let \(a(\pi)\) be the probability that the policy chooses the branch to \(x\). This prediction occurs with probability at least \(1/4\) in both models, and a wrong branch loses at least \(16\varepsilon_{\rm pol}\). Thus \(\varepsilon_{\rm pol}\)-optimality requires \(a(\pi)\le1/4\) under \(r_0\), and \(a(\pi)\ge3/4\) under \(r_1\). We can therefore distinguish the training laws by comparing \(a(\pi)\) with \(1/2\), with error at most \(1/3\).

We now translate this test into a sample lower bound. Environment samples have the same law in both models, so only prediction samples contribute to the transcript KL divergence. Let \(\mathbb P_i\) denote the training-transcript law under \(r_i\). The testing bound gives
\[
\frac{2}{9}
\le\operatorname{KL}(\mathbb P_0\Vert\mathbb P_1)
\le D_2\operatorname{kl}(r_0\Vert r_1)
\le C D_2\frac{(1-\gamma^K)(1-\gamma)^2\varepsilon_{\rm pol}^2}{\gamma^{2K}}.
\]
Rearranging this inequality proves the prediction-sample bound. The construction uses \(O(K)\) states and two covered actions. The chosen parameters are feasible throughout the accuracy range in the theorem.
\end{proof}

\begin{proof}[Proof of Theorem~\ref{thm:bayesian-value-estimation-lower-bounds}]
\emph{Environment samples.} We reduce classical value estimation to Bayesian value estimation without changing the optimal value or giving the learner additional information. If \(\mathcal A=\mathcal A^-\), the bound is zero. Otherwise, we take a classical hard MDP \(M_0\) with \(m=|\mathcal A\setminus\mathcal A^-|\) unknown actions. We make every covered action lead to a known zero-reward absorbing state, with a deterministic known prediction. Since rewards are nonnegative, these extra actions cannot improve the value. An optimal stationary policy of \(M_0\) is also feasible under BF for every \(K\), so \(V^*(s)=V_0^*(s)\) on the original states. A Bayesian value estimator would therefore estimate \(V_0^*\) from the same unknown-row samples. The classical state-value lower bound gives the claimed \(D_1\) rate \citep{azar2012sample,gheshlaghi2013minimax,wang2021quantum}. Appendix~\ref{proof_predictability_dimension_lower_bound} verifies this construction, including the one-action and bounded-discount cases.

\emph{Prediction samples.} For value estimation, we estimate \(v(r)\) directly. We use the same state \(x\) and prediction rule, but remove the preceding decision state \(d\). There is then no extra \(\gamma^K\) discount before reaching \(x\). The parameter choices in Appendix~\ref{proof_prediction_oracle_lower_bound} are feasible throughout the stated accuracy range and satisfy
\begin{equation}
v(r_1)-v(r_0)>2\varepsilon_{\rm est},
\qquad
\operatorname{kl}(r_0\Vert r_1)
\le C(1-\gamma^K)(1-\gamma)^2\varepsilon_{\rm est}^2.
\label{eq:bf-value-lower-separation}
\end{equation}
The two accuracy intervals are disjoint. We can therefore test the models by comparing \(\widetilde V(x)\) with the midpoint of \(v(r_0),v(r_1)\). An accurate estimator makes this test's error at most \(1/3\). Since the environment laws are identical, the testing bound gives \(2/9\le D_2\operatorname{kl}(r_0\Vert r_1)\). Substituting \eqref{eq:bf-value-lower-separation} gives the claimed \(D_2\) rate.
\end{proof}

\subsection{Supporting Constructions for Theorem~\ref{thm:bf-bola-policy-lower-bound}}
\label{proof_bf_bola_policy_lower_bound}
{
Write \(H=(1-\gamma)^{-1}\), \(\beta=\gamma^K\), and \(h=1-\beta\).

\paragraph{Environment-sample construction.}
Let \(\mathcal A_{\rm un}:=\mathcal A\setminus\mathcal A^-\) be the nonempty set of uncovered actions. We use a recurrent state so that a small change in any one of its uncovered transition rows can change the preferred branch at a preceding decision state.

Set \(p:=1-\gamma\) and \(\Delta:=162\varepsilon_{\rm pol}(1-\gamma)^2\). Choose \(c_0\) small enough that \(\Delta\le p/4\). A zero-reward action that succeeds with probability \(u\) on each attempt, reaching a unit-reward absorbing state on success and otherwise returning to the active state, has value
\[
  v(u)=\frac{\gamma u}{(1-\gamma)(1-\gamma+\gamma u)},
  \qquad
  v'(u)=\frac{\gamma}{(1-\gamma+\gamma u)^2}.
\]
We check that these parameters give a large action gap and a small KL divergence per sample. Since \(p\pm\Delta\in[3(1-\gamma)/4,5(1-\gamma)/4]\), the mean-value theorem gives \(v(p+\Delta)-v(p-\Delta)\ge64\gamma\varepsilon_{\rm pol}\). The action gap \(\gamma[v(p+\Delta)-v(p-\Delta)]/2\) is therefore at least \(8\varepsilon_{\rm pol}\). We also have \(\operatorname{kl}(p-\Delta\Vert p+\Delta)\le C\Delta^2/p\le C\varepsilon_{\rm pol}^2(1-\gamma)^3\), as needed for the testing bound.

We repeat this construction in \(n=\lfloor (|\mathcal S|-2)/3\rfloor\) gadgets, adding known zero-reward absorbing states if needed. Gadget \(i\) contains states \(d_i,x_i,y_i\), together with two shared absorbing states \(g\) and \(b\). At \(x_i\), each uncovered action \(a\) reaches \(g\) with probability \(p+\theta_{i,a}\Delta\) and returns to \(x_i\) otherwise, where \(\theta_{i,a}\in\{-1,+1\}\). Every covered action at \(x_i\) moves to \(b\). State \(g\) has reward one forever and \(b\) has reward zero forever. At \(x_i\), all actions give zero immediate reward. Since \(g\) has the maximum possible value, an uncovered action with the largest success probability is optimal. Thus the optimal value at \(x_i\) is \(v(p-\Delta)\) if all signs in gadget \(i\) are negative and \(v(p+\Delta)\) otherwise.

We use the decision at \(d_i\) to compare the risky value with the known midpoint. One designated action moves deterministically to \(x_i\), and all other actions move deterministically to \(y_i\). There is at least one such reference action because \(|\mathcal A|\ge2\). Set \(w:=(v(p+\Delta)+v(p-\Delta))/2\) and make \(y_i\) absorbing with per-period reward \((1-\gamma)w\). All unmentioned rewards are zero. For every \(K\), predictions are the same deterministic table of covered rows at every phase. Hence Assumption~\ref{assump:prediction-law} holds and only the \(n|\mathcal A_{\rm un}|\) rows \((x_i,a)\), \(a\in\mathcal A_{\rm un}\), can distinguish the models.

Let \(q_i(\pi)\) be the probability of choosing the risky branch at \(d_i\) at the beginning of a block. The branch gap holds regardless of later decisions, so
\[
  V^*_{\theta}(d_i)-V^\pi_{\theta}(d_i)
  \ge
  \frac{\gamma}{2}\bigl(v(p+\Delta)-v(p-\Delta)\bigr)
  \begin{cases}
  1-q_i(\pi), & \max_{a\in\mathcal A_{\rm un}}\theta_{i,a}=+1,\\
  q_i(\pi), & \max_{a\in\mathcal A_{\rm un}}\theta_{i,a}=-1.
  \end{cases}
\]
This gives the policy separation between the all-negative model and each single-row alternative used in the main proof.

\paragraph{Prediction-channel construction.}
For the prediction-sample bound, we vary how often the signal reveals a transition outcome. We keep the physical kernel unchanged, so environment observations cannot distinguish the two prediction laws.

We use the calibrated prediction rule from the main proof only for the risky row at \(x\) in the first phase of each block. All other prediction rows equal their known base rows. Conditional on the prediction, we draw potential next states independently across rows and phases. We repeat the construction independently across blocks. It satisfies Assumption~\ref{assump:prediction-law} with the same physical kernel for every \(r\). In particular, independence across phases ensures base-MDP compatibility for every policy that ignores predictions. Since only one prediction component varies, the KL divergence between the two prediction laws is
\(rq\log(r/r')+r(1-q)\log(r/r')+(1-r)\log((1-r)/(1-r'))=\operatorname{kl}(r\Vert r')\).

We use two constructions depending on the block discount. When \(\beta\le3/4\), a single prediction-dependent decision gives the required value gap. The risky action at \(x\) reaches a unit-reward absorbing state with probability \(q=1/2\), and a zero-reward absorbing state otherwise. The safe action reaches an absorbing state with per-period reward \(q\). Both actions have zero immediate reward. The optimal rule chooses the risky action on revealed success and the safe action otherwise. Its value can be written as
\[
v(r)=\gamma H[q+rq(1-q)],
\qquad v'(r)=\gamma Hq(1-q).
\]
For \(\beta>3/4\), we need the value to be more sensitive to the prediction law. We allow the decision maker to wait and use a new prediction at the next block boundary. We take \(q=1/4\). The risky action has the same two absorbing outcomes, while waiting returns to \(x\) after \(K\) zero-reward steps, using \(K-1\) intermediate states. Both actions at an intermediate state follow the same deterministic path. Around \(r_\star=2h/\beta\), it is optimal to wait on erased predictions and revealed failures. Under this rule, the value and its derivative are
\[
v(r)=\frac{\gamma H r q}{h+\beta r q},
\qquad v'(r)=\frac{\gamma H q h}{(h+\beta r q)^2},
\]
and the condition for waiting on an uninformative prediction \(\beta v(r)\ge\gamma Hq\) holds whenever \(r\ge h/[\beta(1-q)]\). Revealed success favors the risky action because \(\beta v(r)<\gamma H\).

\paragraph{Policy-separation parameters.}
We now choose two nearby revelation probabilities so that their value gap remains large enough after a one-block delay. For \(\beta\le3/4\), we take \(r_\star=1/2\) and \(\Delta=128\varepsilon_{\rm pol}/(\beta H)\). For \(\beta>3/4\), we take \(r_\star=2h/\beta\) and \(\Delta=1024\varepsilon_{\rm pol}h/(\beta H)\). We set \(r_0=r_\star-\Delta\), \(r_1=r_\star+\Delta\). The accuracy condition in Theorem~\ref{thm:bf-bola-policy-lower-bound}, with sufficiently small \(c_0\), ensures \(\Delta\le r_\star/8\) and \(1-r_i\ge1/4\). In the recurrent case, these choices also ensure that waiting remains optimal on erased predictions.

We check the two requirements for testing. First, the displayed derivatives give \(\beta[v(r_1)-v(r_0)]\ge32\varepsilon_{\rm pol}\), so the discounted gap is large enough. Second, we use the Bernoulli bound \(\operatorname{kl}(r_0\Vert r_1)\le C\Delta^2/[r_\star(1-r_\star)]\). Since \(r_\star(1-r_\star)\asymp h\) in the recurrent case and \(h\ge1/4\) in the other case, this gives the KL bound in \eqref{eq:bf-policy-lower-separation}.

Finally, we turn this value gap into a policy choice one block earlier. From the predecessor state \(d\), we use two deterministic paths with \(K-1\) intermediate states each. Both actions at intermediate states are identical. The reference state is absorbing with per-period reward \((1-\gamma)w\), where \(w=(v(r_0)+v(r_1))/2\). All path rewards are zero. These paths give exactly the two branch values used in the main proof. The construction has \(O(K)\) states, two covered actions, and rewards in \([0,1]\).
}

\subsection{Prediction-Sample Parameter Checks for Theorem~\ref{thm:bayesian-value-estimation-lower-bounds}}
\label{proof_prediction_oracle_lower_bound}
{
For value estimation, we compare the values at the base state directly. We do not need an earlier policy choice, so there is no additional block discount.

We use the base states in Appendix~\ref{proof_bf_bola_policy_lower_bound} without the predecessor \(d\), retaining \(H=(1-\gamma)^{-1}\), \(\beta=\gamma^K\), and \(h=1-\beta\). For \(\beta\le3/4\), we take \(r_\star=1/2\) and \(\Delta=16\varepsilon_{\rm est}/H\). For \(\beta>3/4\), we take \(r_\star=2h/\beta\) and \(\Delta=64\varepsilon_{\rm est}h/H\). We set \(r_0=r_\star-\Delta\) and \(r_1=r_\star+\Delta\). Since \(\varepsilon_{\rm est}\le c_0H\), we can choose a small universal \(c_0\) so that \(\Delta\le r_\star/8\). The resulting probabilities are feasible, and the same actions remain optimal as in the preceding construction.

The derivatives in Appendix~\ref{proof_bf_bola_policy_lower_bound} give a value gap greater than \(2\varepsilon_{\rm est}\). The bound \(\operatorname{kl}(r_0\Vert r_1)\le C\Delta^2/[r_\star(1-r_\star)]\) gives the KL estimate in \eqref{eq:bf-value-lower-separation}: when \(\beta\le3/4\), \(r_\star=1/2\) and \(h\ge1/4\), and when \(\beta>3/4\), \(r_\star(1-r_\star)\asymp h\). These checks hold throughout the accuracy range of Theorem~\ref{thm:bayesian-value-estimation-lower-bounds}.
}

\subsection{Classical Embedding for Theorem~\ref{thm:bayesian-value-estimation-lower-bounds}}
\label{proof_predictability_dimension_lower_bound}
{
We embed a classical hard MDP in the uncovered rows and make predictions unhelpful for learning those rows. The reduction must preserve both the optimal value and the information supplied by the sampling interfaces.

Let \(\mathcal A_{\mathrm{un}}:=\mathcal A\setminus\mathcal A^-\) and \(m:=|\mathcal A_{\mathrm{un}}|\). If \(m=0\), the claimed bound is zero. Otherwise, use the classical generative-model lower bound for estimating the optimal state-value function with deterministic rewards given to the learner \citep{azar2012sample,gheshlaghi2013minimax}. Its state-value formulation for \(m\ge2\) and \(\gamma\ge0.9\) is given by \citet{wang2021quantum}. The one-action case is the Markov reward process specialization of the Bernoulli hard family, and for \(\gamma\in[1/2,0.9)\) the effective horizon is bounded, so the standard Bernoulli testing lower bound has the same order. Thus a hard family \(M_0=(\mathcal S_0,\mathcal A_{\mathrm{un}},P_0,r_0,\gamma)\) requires \(\Omega(|\mathcal S_0|m/((1-\gamma)^3\varepsilon_{\rm est}^2))\) transition samples for uniform \(\varepsilon_{\rm est}\)-accurate estimation of \(V_0^*\) with probability at least \(2/3\), in the theorem's accuracy range.

We place \(M_0\) in our prediction model by keeping \(P_0,r_0\) on uncovered actions and adding one known zero-reward absorbing state \(s_\circ\). Every covered action has reward zero and transitions deterministically to \(s_\circ\). At every prediction phase, its predicted row is the known point mass \(\sigma_j(\cdot\mid s,a)=\delta_{s_\circ}\). Thus the prediction law is a known point mass satisfying Assumption~\ref{assump:prediction-law}. We can answer every prediction query without using any environment samples.

We first check that this construction preserves the optimal value. The covered actions cannot improve on the nonnegative value in \(M_0\). At every phase, we can use \(\pi_j(s,\boldsymbol\sigma)=\pi_0^*(s)\) to implement its optimal stationary policy. Conversely, known deterministic predictions cannot improve on the classical optimum, by stationary-policy optimality \citep{puterman2014markov}. This proves the value identity used in the main proof.

We also check that the sampling interfaces give no additional information about the hard family. We answer uncovered-row queries using the generative model of \(M_0\). Covered rows, rewards, and predictions are already known. Adding one absorbing state changes only the universal constant in the sample bound, so the classical lower bound still applies.
}

\section{Reference Protocols}
\label{app:reference-protocol-baseline}
This appendix gives the Bellman equations and computational results for the fixed-horizon and receding-horizon reference protocols.

\subsection{Fixed-Horizon and Receding-Horizon Reference Results}
\label{app:fh-reference-results}

\subsubsection{Fixed-Horizon Bayesian Bellman Equation}
\label{app:fh-bellman-reference}
{
For \(\pi\in\Pi^{\FH}\), define its Bayesian value before the initial prediction is observed by
\begin{align}
  V^{\Bayes,\FH,\pi}_{K,\mathcal A^-,\mu}(s)
  &:=
  \mathbb E^{\pi,\mu}\!\left[
    \sum_{t=0}^{\infty}\gamma^t r(S_t,A_t)
    \,\middle|\,S_0=s
  \right].
  \label{def_Bayesian_bellman}
\end{align}
Let \(V^{\Bayes,\FH,*}_{K,\mathcal A^-,\mu}(s):=\sup_{\pi\in\Pi^{\FH}}V^{\Bayes,\FH,\pi}_{K,\mathcal A^-,\mu}(s)\). The open-loop block value \(J_0^{\FH}\) is defined in \eqref{eq:fh-block-value}, with transition probabilities generated by the conditional transition rows \(\bar P_j\).

\begin{theorem}[Fixed-Horizon Bayesian Bellman Equation]
\label{thm:fh-bayesian-bellman}
The optimal fixed-horizon Bayesian value \(V^{\Bayes,\FH,*}_{K,\mathcal A^-,\mu}\) is the unique solution to
\begin{align}
  V^{\Bayes,\FH,*}_{K,\mathcal A^-,\mu}(s)
  &=
  \mathbb E_{\boldsymbol\Sigma\sim\mu}
  \left[
    J_0^{\FH}
    \left(s,\boldsymbol\Sigma;
    V^{\Bayes,\FH,*}_{K,\mathcal A^-,\mu}\right)
  \right],
  \qquad s\in\mathcal S.
  \label{eq:fh-bellman}
\end{align}
\end{theorem}

\begin{proof}
For each open-loop sequence, the block return in \eqref{eq:fh-block-value} is monotone and \(\gamma^K\)-Lipschitz in the terminal value. Maximizing over sequences and averaging over predictions preserve both properties, so the FH Bellman operator is a monotone \(\gamma^K\)-contraction. At each block boundary, the next prediction is drawn independently of the completed block, so the Bayesian value at that boundary depends only on the terminal state. The policy-evaluation and greedy-policy argument in Theorem~\ref{thm:bf-bellman} then identifies the unique fixed point with the optimal FH value, using an open-loop sequence as the block action \citep{puterman2014markov,bertsekas1996neuro}.
\end{proof}
}

\subsubsection{Receding Horizon as an Augmented MDP}
\label{app:rh-reference-results}
\label{Sec:receding horizon planning}

If the physical state and rolling prediction \((s_t,\boldsymbol\sigma_t)\) form a controlled Markov process with kernel \(\widetilde P(\cdot\mid s,\boldsymbol\sigma,a)\), RH is exactly the discounted MDP on \(\mathcal S\times\mathcal Q_K\) with actions \(\mathcal A\), discount \(\gamma\), and reward \(\widetilde r((s,\boldsymbol\sigma),a)=r(s,a)\). Its tabular value has \(|\mathcal S||\mathcal Q_K|\) entries. When each phase has \(B\) possible prediction tables so that \(|\mathcal Q_K|=B^K\), the value has \(|\mathcal S|B^K\) entries. A dense Bellman sweep evaluates \(|\mathcal S||\mathcal A|B^K\) state-action pairs, summing over \(|\mathcal S|B^K\) successors per pair for a total cost of \(O(|\mathcal S|^2|\mathcal A|B^{2K})\). This is the curse of dimensionality for exact tabular RH planning \citep{luovercoming}.

\subsubsection{Exact-Tabular Complexity of the Reference Protocols}
\label{app:reference-protocol-complexity}

Table~\ref{tab:complexity_comparison} reports single-block costs for FH and BF and infinite-horizon exact planning hardness for RH \citep{pla2026hardness}. The BF cost follows from Proposition~\ref{prop:bfbola-complexity}.

For FH, consider the following decision problem. An instance consists of a finite MDP, a prediction horizon \(K\), a deterministic prediction, a discount factor \(\gamma\in(0,1)\), and a threshold \(\rho\). The question is whether there exists an open-loop sequence \(\boldsymbol a\in\mathcal A^K\) whose expected discounted reward over the block is at least \(\rho\).

\begin{theorem}[NP-Completeness of FH Open-Loop Planning]
\label{thm:fh-np-hardness}
The FH open-loop planning decision problem is NP-complete, even with binary actions and deterministic predictions.
\end{theorem}

\begin{proof}
We first check membership in NP. Given an open-loop sequence, we can compute its expected discounted reward by forward propagation over the finite state space. This takes time polynomial in the explicit input size.

To prove hardness, we reduce from 3-SAT. We make each open-loop action sequence represent a truth assignment and make its reward depend on the fraction of clauses it satisfies. Let \(\varphi\) have variables \(x_1,\ldots,x_n\) and clauses \(C_1,\ldots,C_m\). We construct an MDP with binary action set \(\{0,1\}\), discount \(\gamma=1/2\), and horizon \(K=n+1\). The initial action is ignored. From \(s_0\), the process samples a clause index \(J\) uniformly from \(\{1,\ldots,m\}\) and moves to \(q_{0,J,0}\), where the last coordinate indicates whether the sampled clause has been satisfied by the partial assignment.

At variable stage \(i=1,\ldots,n-1\), action \(a\in\{0,1\}\) assigns \(x_i=a\) and moves from \(q_{i-1,j,b}\) to \(q_{i,j,b'}\), where \(b'=b\vee\mathbf 1\{x_i=a\text{ satisfies }C_j\}\), with zero reward. At the final variable stage, action \(a\) assigns \(x_n=a\), receives reward \(b\vee\mathbf 1\{x_n=a\text{ satisfies }C_j\}\), and then moves to an absorbing state. We use an uninformative deterministic prediction, which can be fixed arbitrarily.

Every open-loop action sequence therefore corresponds to one truth assignment, and its expected discounted block reward is \(\gamma^n\) times the fraction of clauses satisfied by that assignment. Setting \(\rho=\gamma^n\), there exists an open-loop sequence with value at least \(\rho\) if and only if \(\varphi\) is satisfiable. The construction has \(O(nm)\) states and polynomial bit complexity, so the problem is NP-hard and hence NP-complete.
\end{proof}

\end{document}